\documentclass{article}

\usepackage{microtype}
\usepackage{graphicx}
\usepackage{subfigure}
\usepackage{subcaption}
\usepackage{booktabs} % for professional tables

\usepackage{hyperref}

\usepackage{tcolorbox}

\usepackage[preprint]{icml2026}

\usepackage{amsmath}
\usepackage{amssymb}
\usepackage{mathtools}
\usepackage{amsthm}

\def\R{\mathbb{R}}
\def\E{\mathbb{E}}

\def\F{{\rm F}}

\def\M{\mathcal{M}}

\usepackage[capitalize,noabbrev]{cleveref}

\theoremstyle{plain}
\newtheorem{theorem}{Theorem}[section]

\newtheorem{lemma}[theorem]{Lemma}

\theoremstyle{definition}
\newtheorem{definition}[theorem]{Definition}
\newtheorem{assumption}[theorem]{Assumption}
\theoremstyle{remark}
\newtheorem{remark}[theorem]{Remark}

\usepackage[textsize=tiny]{todonotes}

\begin{document}

\twocolumn[
  \icmltitle{HomeAdam: Adam and AdamW Algorithms Sometimes Go Home to Obtain Better Provable Generalization}



  \icmlsetsymbol{equal}{*}

  \begin{icmlauthorlist}
    \icmlauthor{Feihu Huang}{yyy,comp}
    \icmlauthor{Guanyi Zhang}{yyy}
    \icmlauthor{Songcan Chen }{yyy,comp}
  \end{icmlauthorlist}

  \icmlaffiliation{yyy}{College of Computer Science and Technology,
Nanjing University of Aeronautics and Astronautics, Nanjing, China}
  \icmlaffiliation{comp}{MIIT Key Laboratory of Pattern Analysis and Machine Intelligence, Nanjing, China}

  \icmlcorrespondingauthor{Feihu Huang}{huangfeihu2018@gmail.com}

  \icmlkeywords{Machine Learning, ICML}

  \vskip 0.3in
]



\printAffiliationsAndNotice{}  % no special notice (required even if empty)

\begin{abstract}
 Adam and AdamW are a class of default optimizers for training deep learning models in machine learning.
  These adaptive algorithms converge faster but generalize worse compared to SGD. In fact, their proved generalization error $O(\frac{1}{\sqrt{N}})$ also is larger than $O(\frac{1}{N})$
  of SGD, where $N$ denotes training sample size. Recently, although some variants of Adam
  have been proposed to improve its generalization, their improved generalizations are
  still unexplored in theory. To fill this gap, in the paper,
  we restudy generalization of Adam and AdamW via algorithmic stability, and first
   prove that Adam and AdamW without square-root (i.e., Adam(W)-srf) have a generalization error $O(\frac{\hat{\rho}^{-2T}}{N})$, where $T$ denotes iteration number and $\hat{\rho}>0$ denotes the smallest element of second-order momentum plus a small positive number. To improve generalization,
  we propose a class of efficient clever Adam (i.e., HomeAdam(W)) algorithms via sometimes returning momentum-based SGD. Moreover, we prove that our HomeAdam(W) have a smaller generalization error $O(\frac{1}{N})$ than $O(\frac{\hat{\rho}^{-2T}}{N})$ of Adam(W)-srf, since $\hat{\rho}$ is generally very small. In particular, it is also smaller than the existing $O(\frac{1}{\sqrt{N}})$ of Adam(W). Meanwhile, we prove our HomeAdam(W) have a faster convergence rate of $O(\frac{1}{T^{1/4}})$ than $O(\frac{\breve{\rho}^{-1}}{T^{1/4}})$ of the Adam(W)-srf,
  where $\breve{\rho}\leq\hat{\rho}$ also is very small.
  Extensive numerical experiments demonstrate efficiency of our HomeAdam(W) algorithms.
\end{abstract}

\section{Introduction}
Deep learning models have shown great successes in many machine learning applications such as
  computer vision~\citep{lecun2015deep,he2016deep} and natural language processing~\citep{vaswani2017attention,hu2021lora} and reinforcement learning~\citep{schulman2017proximal,tang2025deep}.
In fact, training these deep learning models mainly relies on efficient optimization algorithms~\citep{bottou2018optimization}.
For example, SGD~\citep{robbins1951stochastic,ghadimi2013stochastic} is a basic optimization algorithm
for training these large models by only
querying a mini-batch samples, but it suffers from large variances and
sensitivity of learning rate. To alleviate these defects, Adam~\citep{kingma2014adam}
has been proposed by using
momentum technique and adaptive learning rate, which has shown remarkable performances
in training many deep learning models. In particular, Adam significantly outperforms SGD in training some specific models such as transformers~\citep{zhang2024transformers,zhang2020adaptive}.

\begin{table*}
  \centering
  \caption{ \textbf{Generalization error} comparison of our HomeAdam(W) algorithms and other algorithms
  for nonconvex optimization. Here $N$ denotes the training sample size.}
  \label{tab:1}
   \resizebox{0.80\textwidth}{!}{
\begin{tabular}{c|c|c|c}
  \hline
  \textbf{Algorithm} & \textbf{Reference} & \textbf{Generalization Error} & \textbf{Adaptive Learning Rate}
  \\ \hline
  SGD  & \cite{hardt2016train} & $O(\frac{1}{N})$ &   \\   \hline
  SGDM  & \cite{ramezani2024generalization} & $O(\frac{1}{N})$ &   \\  \hline
  Adam  & \cite{zhou2024towards} & $O(\frac{1}{\sqrt{N}})$  & $\checkmark$ \\   \hline
  AdamW  & \cite{zhou2024towards} & $O(\frac{1}{\sqrt{N}})$ & $\checkmark$  \\  \hline
  HomeAdam & Ours &  \textcolor{red}{$O(\frac{1}{N})$}  & $\checkmark$ \\  \hline
  HomeAdamW & Ours &  \textcolor{red}{$O(\frac{1}{N})$} & $\checkmark$ \\  \hline
\end{tabular}
 }
\end{table*}

Owing to its efficiency
and robustness to hyper-parameters, Adam becomes one of
default optimizers in deep learning.
Thus, its convergence property has taken wide attentions in machine learning community.
Recently, many works~\citep{chen2018convergence,zhou2018convergence,reddi2019convergence,zou2019sufficient,guo2021novel,jin2024comprehensive,
defossez2020simple,zhang2022adam,wang2023closing,xie2024adan,taniguchi2024adopt,peng2025simple} studied convergence properties of the Adam and its variants.
For example, \cite{reddi2019convergence} found a non-convergence case of Adam, and
 provided convergence analysis for a variant of Adam (i.e., AMSGrad).
Meanwhile, \cite{chen2018convergence} studied convergence properties of
a class of Adam-type algorithms for non-convex optimization.
Subsequently, \cite{guo2021novel} provided a generic convergence analysis for
a family of Adam-style methods including Adam and AMSGrad, and proved that Adam has
a convergence rate of $O(\frac{1}{T^{1/4}})$ for non-convex optimization,
where $T$ denotes the total iteration number.
 \cite{wang2023closing} derived a new convergence guarantee of Adam with only assuming smooth condition
and bounded variance assumption.
More recently, \cite{jin2024comprehensive}
provided a comprehensive framework for analyzing convergence properties of the Adam.
Meanwhile, \cite{jin2024comprehensive} studied convergence properties of
Adam and its variance reduced variant under generalized smoothness assumption.

Adam has some good empirical performances, but its
generalization performances still are worse than SGD on some deep learning tasks~\citep{wilson2017marginal}.
Meanwhile, the proved generalization error $O(\frac{1}{\sqrt{N}})$ of Adam~\citep{zhou2024towards} is larger than $O(\frac{1}{N})$ of SGD~\citep{hardt2016train} and momentum-based SGD (SGDM)~\citep{ramezani2024generalization}.
\cite{zou2023understanding} also provided an explanation that
the inferior generalization performance of Adam is
fundamentally tied to the nonconvex landscape of deep learning.
Clearly, this insufficient generalization of Adam hinders its broader application.
Thus, some adaptive gradient methods~\citep{loshchilov2017decoupled,keskar2017improving,chen2018closing,zhuang2020adabelief,jin2025method}
recently have been proposed to improve generalization of Adam.
For example, AdamW~\citep{loshchilov2017decoupled} improves generalization of Adam via
weight decay, which also is one of
default optimizers.
Recently convergence properties of the AdamW have been studied in \citep{zhou2024towards,xie2024implicit,li2025frac}, and it also proved that AdamW has
a convergence rate of $O(\frac{1}{T^{1/4}})$ for non-convex optimization.
In addition, \cite{keskar2017improving} proposed
a simple strategy by switching from Adam to SGD via
a triggering condition to improving generalization performance.
\cite{zhuang2020adabelief} proposed an AdaBelief optimizer to improve generalization of Adam by
 using the 'belief' in the current gradient direction to adaptive stepsize.
 More recently, \cite{jin2025method} proposed a multiple integral Adam (i.e., MIAdam) by guiding
the optimizer towards flatter regions to enhance
generalization capability.

Although these proposed methods improve generalization of Adam shown in some empirical performances,
few works proved an improved generalization for the proposed methods in theory.
So far, based on PAC Bayesian framework,
\cite{zhou2024towards} only proved that AdamW has the same generalization error
of $O(\frac{1}{\sqrt{N}})$ as the Adam, which still is larger
than $O(\frac{1}{N})$ of SGD~\citep{hardt2016train} and SGDM~\citep{ramezani2024generalization} for nonconvex optimization.
In addition, based on the diffusion theory framework, \cite{jin2025method}
only proved that the MIAdam is more likely to escape from sharp
minima and consequently converge to flat minima than the Adam.
Clearly, the proved improved generalization error of Adam is still missing.
To fill this gap, in the paper,
  we restudy generalization of the Adam and AdamW via algorithmic stability used in~\citep{hardt2016train}.
Our main contributions are given as follows:
\begin{itemize}
\item[(1)] We first propose a class of square-root-free Adam (Adam-srf and AdamW-srf i.e., Adam(W)-srf) algorithms via removing square-root as used in~\cite{lin2024can,choudhury2024remove}. Meanwhile, we provide a useful generalization analysis framework based on the mathematical induction, and prove that Adam(W)-srf have a generalization error $O(\frac{\hat{\rho}^{-2T}}{N})$, where $T$ denotes iteration number and $\hat{\rho}>0$ denotes the smallest element of second-order momentum plus a small positive number.
\item[(2)] To improve generalization,
  we further propose a class of efficient clever Adam (HomeAdam and HomeAdamW i.e.,HomeAdam(W)) algorithms via sometimes returning SGDM. Here SGDM can be see as a non-adaptive variant of Adam, which is like Adam's \emph{home}. Moreover, we prove that our HomeAdam(W) have a smaller generalization error $O(\frac{1}{N})$ than the above error $O(\frac{\hat{\rho}^{-2T}}{N})$, since $\hat{\rho}$ is generally very small. In particular, it is also smaller than the existing error $O(\frac{1}{\sqrt{N}})$ of Adam and AdamW~\citep{zhou2024towards}.
\item[(3)] We also provide a convergence analysis for our proposed methods, and prove our HomeAdam(W) methods have a faster convergence rate of $O(\frac{1}{T^{1/4}})$ than $O(\frac{\breve{\rho}^{-1}}{T^{1/4}})$ of the Adam(W)-srf, where $\breve{\rho}\leq\hat{\rho}$ also is very small. It is the same as  the existing $O(\frac{1}{T^{1/4}})$ of Adam and AdamW for nonconvex optimization~\citep{guo2021novel,zhou2024towards}.
\item[(4)] Extensive numerical experiments demonstrate effectiveness of our HomeAdam(W) algorithms.
\end{itemize}
To the best of our knowledge, we first prove that the adaptive gradient methods have the same
 generalization error as the SGD and SGDM methods.

 \begin{figure*}[ht]
\centering
 \subfigure[Adam(W)-srf/Adam(W) ]{\includegraphics[width=0.26\textwidth]{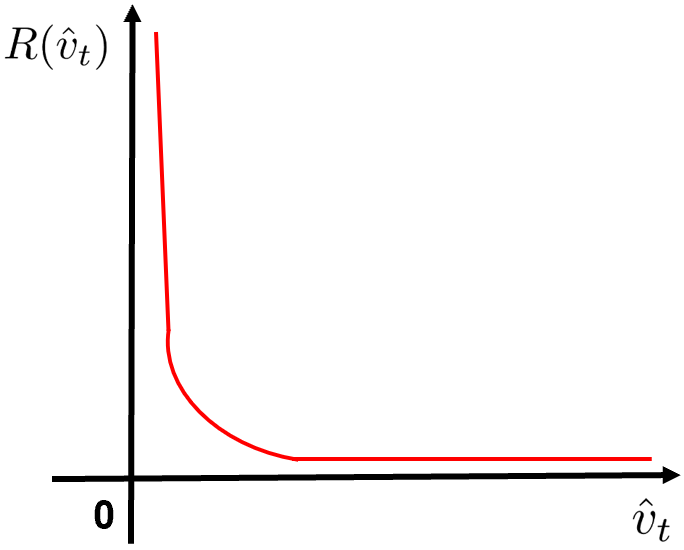}}
 \hfill
 \subfigure[SGD(M)]{\includegraphics[width=0.26\textwidth]{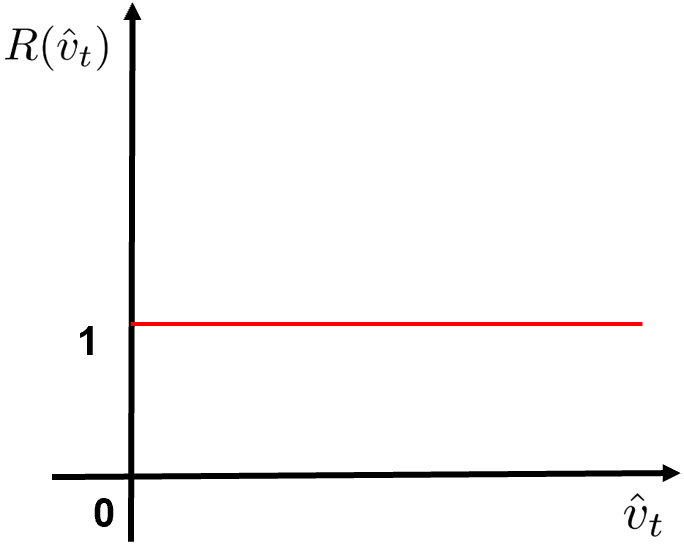}}
  \hfill
 \subfigure[HomeAdam(W)]{\includegraphics[width=0.26\textwidth]{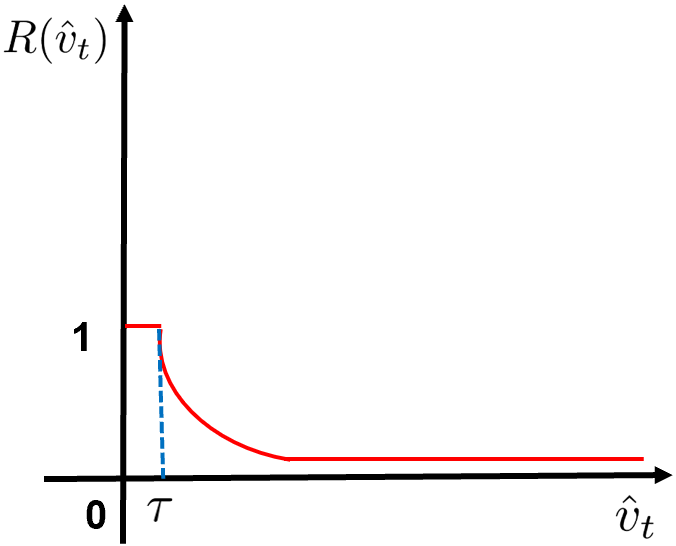}}
  \hfill
\caption{ Illustration of different
stepsize function $R(\hat{v}_t)$: \textbf{(a)} Adam(W)-srf uses $R(\hat{v}_t)=\frac{1}{\hat{v}_t}$ and
Adam(W) uses $R(\hat{v}_t)=\frac{1}{\sqrt{\hat{v}_t}}$; \textbf{(b)} SGD(M) uses $R(\hat{v}_t)=1$; \textbf{(c)} our HomeAdam(W) uses $R(\hat{v}_t)=\frac{1}{\hat{v}_t}$ when $\min_{1\leq j\leq d} (\hat{v}_t)_j \geq \tau>0$, otherwise $R(\hat{v}_t)=1$.}
\label{fig:1}
\end{figure*}
\vspace*{-8pt}
\subsection*{Notations}
$[N]$ denotes $\{1,2,\cdots,N\}$. $\R^+$ denotes set of positive real numbers. $\|\cdot\|$ denotes the $\ell_2$ norm for vectors and spectral norm for matrices. $\|\cdot\|_1$ and $\|\cdot\|_{\infty}$ denote $\ell_1$-norm and  $\ell_{\infty}$-norm, respectively. $\langle x,y\rangle$ denotes the inner product of two vectors $x\in \R^d$ and $y\in \R^d$. Let $x^\gamma\ (\gamma>0)$ denote the element-wise power operation, and $x/y$ denote the element-wise division. $I_{d}$ denotes a $d$-dimensional identity matrix. $a_t=O(b_t)$ denotes that $a_t \leq c b_t$ for some constant $c>0$.
\vspace*{-8pt}
\section{ Preliminaries }
\subsection{Problem Setup}
In machine learning, the ultimate goal of learning is to minimize
the population risk given by
\begin{align}\label{eq:p1}
 \min_{\theta \in \R^{d}} F(\theta) = \E_{z\sim \mathcal{D}}[f(\theta;z)],
\end{align}
where $f(\theta;z): \R^{d}\rightarrow \R^+$ denotes a loss function on a sample $z\sim \mathcal{D}$,
which is possibly non-convex. Here $z$ is a random variable drawn
some fixed but unknown distribution $\mathcal{D}$. Here $F(\theta) = \E_{z\sim \mathcal{D}}[f(\theta;z)]$
denotes a population loss (risk) of machine learning tasks such as training deep learning models.
Since the fixed distribution $\mathcal{D}$ is
 unknown, we only access a finite set of
 training examples $S=\{z_1,z_2,\cdots,z_N\}$ drawn i.i.d. from $\mathcal{D}$.
The above population risk $F(\theta)$ could be approximated by
 the following empirical risk
 \begin{align}
  F_S (\theta) = \frac{1}{N}\sum_{i=1}^N f(\theta;z_i).
 \end{align}
\subsection{ Generalization Gap }
The best model $\theta^*$ is defined as
 \begin{align}
  \theta^*=\mathop{\arg\min}_{\theta\in \R^{d}}F(\theta). \nonumber
 \end{align}
Based on a training dataset $S=\{z_1,z_2,\cdots,z_N\}$, we run a
 randomized algorithm $A$ to minimize the empirical risk to get a output model $A(S)$.
 In fact, we are interested in the excess
 population risk $F(A(S))-F(\theta^*)$, which measures the relative behavior of the output model as compared
 to the best model.

 Here we decompose this excess population risk into the following formation
 \begin{align}\label{eq:ep}
  F(A(S))& -F(\theta^*) = \underbrace{F(A(S)) - F_S(A(S))}_{(a)} \\
   &\quad + \underbrace{F_S(A(S)) - F_S(\theta^*)}_{(b)} + F_S(\theta^*) -F(\theta^*), \nonumber
 \end{align}
 where the term $(a)$ $F(A(S)) - F_S(A(S))$ is generalization error (generalization gap), which measures
 the gap between training loss and population loss;
 and the term $(b)$ $F_S(A(S)) - F_S(\theta^*)$ is optimization error, which quantifies how well the algorithm minimizes the empirical risk. Taking expectation on the equality~(\ref{eq:ep}) with random algorithm $A$ and
 training dataset $S$, we have
\begin{align}
 \E_{A,S} [F(A(S)) -F(\theta^*) ]&= \E_{A,S} [ F(A(S)) - F_S(A(S))] \nonumber \\
  & \ + \E_{A,S} [F_S(A(S)) - F_S(\theta^*)], \nonumber
 \end{align}
where $\E_{A,S}[F_S(\theta^*) -F(\theta^*)]=0$, i.e., $\theta^*$
is independent of $A$ and $S$.

\subsection{ Algorithmic Stability }
In the subsection, we introduce some concepts on algorithmic stability.
 For notational simplicity, let $S=\{z_1,z_2,\cdots,z_N\}$ and $\tilde{S}=\{\tilde{z}_1,\tilde{z}_2,\cdots,\tilde{z}_N\}$ be two
 independent datasets drawn from the distribution $\mathcal{D}$. Then we denote
 the dataset $S^{(i)}=\{z_1,z_2,\cdots,\tilde{z}_i,\cdots,z_N\}$ by replacing the $i$-th example $z_i$ with an independent sample $\tilde{z}_i$ for any $i\in [N]$.
 \begin{definition} (Stability)
  Let $A$ be a random algorithm and $A(S)$ denote the output of the algorithm $A$ run on dataset $S$. We say the random algorithm $A$ is $\epsilon$-uniformly stable if for any $S$ and $S^{(i)}$,
 \begin{align}
  \sup_{z} \E_A [f(A(S),z)-f(A(S^{(i)}),z)] \leq \epsilon.
 \end{align}
% We say $A$ is on-average $\epsilon$-stable if
% \begin{align}
%  \frac{1}{N}\sum_{i=1}^N\E_{A,S,\tilde{S}} [f(A(S^{(i)}),z_i)-f(A(S),z_i)] \leq \epsilon.
% \end{align}
 \end{definition}
 \begin{lemma} \label{lem:gs}
  (Generalization via Stability)~\cite{shalev2010learnability,hardt2016train}. Let algorithm $A$ be $\epsilon$-uniformly stable in function
values, then we have
 $$ |\E_{A,S}[F(A(S)) - F_S(A(S))]|\leq \epsilon.$$
 \end{lemma}

\begin{algorithm}
    \caption{\textbf{Adam(W)-srf} Algorithms }
    \label{alg:1}
    \begin{algorithmic}[1]
        \STATE \textbf{Input}: $\eta>0$, $\beta_1,\beta_2\in (0,1)$, $\varepsilon\geq 0$ and $\lambda\geq0$;
        \STATE \textbf{Initialize:} $\theta_0\in \R^{d}$,
         $m_0=0$ and $v_0=0$;
        \FOR{$t = 1, 2, \ldots, T$}
            \STATE  Draw a sample $z_{t} \sim \mathcal{D}$;
            \STATE $g_t = \nabla f(\theta_{t-1};z_{t})$;
            \STATE $m_{t}= \beta_{1}m_{t-1}+ (1-\beta_{1})g_t$;
            \STATE $v_{t}= \beta_{2}v_{t-1}+ (1-\beta_{2})g_t^2$;
            \STATE $\hat{m}_t=\frac{m_t}{1-\beta_1^t}$;
            \STATE $\hat{v}_t=\frac{v_t}{1-\beta_2^t}$;
            \STATE $\theta_{t} = \theta_{t-1} - \eta (\frac{\hat{m}_t}{{\color{red}{\hat{v}_t}} + \varepsilon}+\lambda \theta_{t-1})$.
        \ENDFOR
        \STATE \textbf{Output:} $\theta_{T}$.
    \end{algorithmic}
\end{algorithm}

\section{ Our HomeAdam(W) Algorithms}
In the section, we first propose a class of square-root-free Adam and AdamW algorithms (Adam(W)-srf),
which are shown in Algorithm~\ref{alg:1}.
When $\lambda=0$, Algorithm~\ref{alg:1} is Adam-srf algorithm, otherwise
Algorithm~\ref{alg:1} shows AdamW-srf algorithm. In fact,
Algorithm~\ref{alg:1} is the similar to AdamW algorithm, expect for
removing square-root in the second-order momentum of adaptive learning rate.

Based on Algorithm~\ref{alg:1}, we propose a class of efficient clever Adam (HomeAdam(W)) algorithms via sometimes returning momentum-based SGD, which are shown in
Algorithm~\ref{alg:2}. When $\lambda=0$, Algorithm~\ref{alg:2} shows HomeAdam algorithm, otherwise
Algorithm~\ref{alg:2} is HomeAdamW algorithm.
In Algorithm~\ref{alg:2}, we also remove the square-root in the second-order momentum of adaptive learning rate.
In particular, when $\min_{1\leq j\leq d} (\hat{v}_t)_j \geq \tau$ (i.e.,
the smallest element in the second-order momentum $\hat{v}_t$ is larger than a threshold $\tau>0$),
we use adaptive stochastic gradient to update variable $\theta_t$ as follows:
\begin{align}
 \theta_{t} = (1-\eta\lambda)\theta_{t-1} - \eta \frac{\hat{m}_t}{{\color{red}{\hat{v}_t}} + \varepsilon},
\end{align}
otherwise we use stochastic gradient to update $\theta_t$ as follows:
\begin{align}
 \theta_{t} = (1-\eta\lambda)\theta_{t-1} - \eta \hat{m}_t.
\end{align}

Here we define a stepsize function $R(\cdot)$ on the second-order momentum $\hat{v}_t$, then we could  uniformly rewrite the lines 11 and 13 of
Algorithm~\ref{alg:2} as follows:
\begin{align}
 \theta_{t} = (1-\eta\lambda)\theta_{t-1} - \eta {\color{blue}{R(\hat{v}_t)}}\hat{m}_t,
\end{align}
where $R(\hat{v}_t)=\frac{1}{\hat{v}_t+\varepsilon}$ when $\min_{1\leq j\leq d} (\hat{v}_t)_j \geq \tau>0$, otherwise $R(\hat{v}_t)=1$.
In fact, our Adam(W)-srf use $R(\hat{v}_t)=\frac{1}{\hat{v}_t+\varepsilon}$, and
the Adam(W) use $R(\hat{v}_t)=\frac{1}{\sqrt{\hat{v}_t}+\varepsilon}$, and
 the SGD(M) use $R(\hat{v}_t)=1$.

Figure~\ref{fig:1} shows different
stepsize function $R(\hat{v}_t)$, where without loss of generality, let
$\hat{v}_t \in \R$ be a scalar, and set $\varepsilon= 0$.
From Figure~\ref{fig:1}, our HomeAdam(W) not only use adaptive learning rate, but also keep learning rate from becoming too large, which protects their generalization ability and stability.
However, when $\hat{v}_t$ is very small, Adam and its some variants (e.g., Adam(W) and Adam(W)-srf) will use too large learning rate to affect  their generalization ability and stability.

\begin{algorithm}
    \caption{\textbf{HomeAdam(W)} Algorithms }
    \label{alg:2}
    \begin{algorithmic}[1]
        \STATE \textbf{Input}: $\eta>0$, $\beta_1,\beta_2\in (0,1)$, $\varepsilon\geq 0$, $\lambda\geq 0$ and $\tau>0$;
        \STATE \textbf{Initialize:} $\theta_0\in \R^{d}$,
         $m_0=0$ and $v_0=0$;
        \FOR{$t = 1, 2, \ldots, T$}
            \STATE  Draw a sample $z_{t} \sim \mathcal{D}$;
            \STATE $g_t = \nabla f(\theta_{t-1};z_{t})$;
            \STATE $m_{t}= \beta_{1}m_{t-1}+ (1-\beta_{1})g_t$;
            \STATE $v_{t}= \beta_{2}v_{t-1}+ (1-\beta_{2})g_t^2$;
            \STATE $\hat{m}_t=\frac{m_t}{1-\beta_1^t}$;
            \STATE $\hat{v}_t=\frac{v_t}{1-\beta_2^t}$;
            \IF {$\min_{1\leq j\leq d} (\hat{v}_t)_j \geq \tau$}
            \STATE $\theta_{t} = \theta_{t-1} - \eta \big(\frac{\hat{m}_t}{{\color{red}{\hat{v}_t}} + \varepsilon}+
            \lambda\theta_{t-1}\big)$;
            \ELSE
            \STATE $\theta_{t} = \theta_{t-1} - \eta (\hat{m}_t+\lambda\theta_{t-1})$.
            \ENDIF
        \ENDFOR
        \STATE \textbf{Output:} $\theta_{T}$.
    \end{algorithmic}
\end{algorithm}

Note that the SWATS~\citep{keskar2017improving} divides the optimization process into two stages via a switchover point: uses Adam to train models in the first stage, then uses SGD
to train models in the second stage. While our HomeAdam(W) algorithms
switch from Adam(W)-srf to SGDM (or from SGDM
to Adam(W)-srf) at any time throughout the optimization process based on the condition $\min_{1\leq j\leq d} (\hat{v}_t)_j < \tau$ (or otherwise).

In the Appendix~\ref{ap:ew}, we also provide an element-wise variant of our HomeAdam(W) algorithms,
which is more suitable for training deep learning models due to matching the back-propagation framework~\citep{sampson1987parallel}.

\section{Generalization Analysis}
In the section, we provide generalization analysis for our proposed methods (i.e., Adam(W)-srf and
HomeAdam(W)) under some mild
assumptions. All related proofs are provided in the following Appendices~\ref{ap:ga1} and~\ref{ap:ga2}.
We first provide some mild conditions.
\begin{assumption} \label{ass:s1}
Assume each component function $f(\theta;z)$ is $L$-Lipschitz smooth, we have for any $z_1, z_2\in \mathcal{D}$,
\begin{align}
    \|\nabla f(\theta_1;z)-\nabla f(\theta_2;z)\| \leq L\|\theta_1-\theta_2\|.
\end{align}
\end{assumption}
\begin{assumption} \label{ass:g}
Assume each component function $f(\theta;z)$ for any $z\sim \mathcal{D}$ is Lipschitz continuous,
such that
\begin{align}
    |f(\theta_1;z)-f(\theta_2;z) | \leq G\|\theta_1-\theta_2\|, \ \theta_1, \theta_2\in \R^d
\end{align}
where $G>0$.
\end{assumption}
\begin{assumption} \label{ass:v}
$\nabla f(\theta;z)$ is an unbiased stochastic estimator of full gradient $\nabla F(\theta)$ and has a bounded variance, we have
\begin{align}
 \E[\nabla f(\theta;z)]=\nabla F(\theta), \ \E\|\nabla f(\theta;z)-\nabla F(\theta)\|^2]\leq \sigma^2. \nonumber
\end{align}
\end{assumption}

Assumption~\ref{ass:s1} assumes smoothness of each component function $f(\theta;z)$, which is commonly applied in generalization analysis~\citep{zhou2024towards,hardt2016train}.
Assumption~\ref{ass:g} assumes Lipschitz continuous of each component function, which is commonly used in
generalization analysis~\citep{hardt2016train,lei2020fine,lei2023stability}.
Assumption~\ref{ass:g} implies the bounded gradient of $f(\theta;z)$, i.e., $\|\nabla f(\theta;z)\|\leq G$
for any $z\in \mathcal{D}$.
Assumption~\ref{ass:v} shows a standard bounded variance assumption used in stochastic optimization~\citep{bottou2018optimization,ghadimi2013stochastic}.
According to Assumptions~\ref{ass:s1} and~\ref{ass:v}, we have $\|\nabla F(\theta_1)-\nabla F(\theta_2)\| = \|\E[\nabla f(\theta_1;z)-\nabla f(\theta_2;z)]\| \leq \E \|\nabla f(\theta_1;z)-\nabla f(\theta_2;z)\| \leq L\|\theta_1-\theta_2\|$. Thus, we could use smoothness of each component function $f(\theta;z)$ to obtain smoothness of function $F(\theta)$.

\subsection{Generalization Errors of Adam(W)-srf Algorithms}
\begin{lemma} \label{lem:1}
Assume the sequences $\{\hat{m}_t,\hat{v}_t,m_t,v_t\}_{t=1}^T$ and $\{\hat{m}_t^{(i)},\hat{v}_t^{(i)},m_t^{(i)},v_t^{(i)}\}_{t=1}^T$ are generated from Algorithm~\ref{alg:1} based on the datasets $S$ and
$S^{(i)}$, respectively, we have
\begin{align}
\big\|\frac{\hat{m}_t}{\hat{v}_t + \varepsilon} & -  \frac{\hat{m}_t^{(i)}}{\hat{v}_t^{(i)} + \varepsilon}\big\|  \leq
 \frac{\sqrt{d}}{(\rho_t+\varepsilon)(1-\beta_1^t)}\big\|m_t - m_t^{(i)}\big\| \nonumber \\
& \  + \frac{G\sqrt{d}}{(1-\beta_1^t)(1-\beta_2^t)(\rho_t+\varepsilon)^2}\big\| v_t -v_t^{(i)}\big\|,
\end{align}
where $\rho_t = \min_{j\in [d]}(\hat{v}_t)_j$.
\end{lemma}

\begin{theorem} \label{th:1}
Assume the sequence $\{\theta_t,\hat{v}_t\}_{t=1}^T$ is generated from Algorithm~\ref{alg:1} on dataset $S=\{z_1,z_2,\cdots,z_N\}$, under the Assumptions~\ref{ass:s1},~\ref{ass:g},~\ref{ass:v}, let $\eta=O(\frac{1}{\sqrt{d}})$, $\lambda\in [0,\frac{1}{\eta})$, $\beta_1=O(1)$ with $\beta_1\in (0,1)$, $\beta_2=O(1)$ with $\beta_2\in (0,1)$, $\sigma=O(1)$, $G=O(1)$ and
$L=O(1)$, we have
\begin{align}
 |\E [ F(\theta_T) - F_S(\theta_T)]| \leq O(\frac{\hat{\rho}^{-2T}}{N}),
\end{align}
where $\hat{\rho} =\rho +\varepsilon$ with $\rho = \min_{t\geq 1} \min_{j\in [d]}(\hat{v}_t)_j$.
\end{theorem}

\begin{remark}
Based on the above Theorem~\ref{th:1},
when $\lambda=0$ in Algorithm~\ref{alg:1}, our Adam-srf algorithm has a generalization error of $O(\frac{1}{(\rho+\varepsilon)^{2T}N})$;
when $\lambda\in (0,\frac{1}{\eta})$ in Algorithm~\ref{alg:1},
our AdamW-srf algorithm also has a generalization error of $O(\frac{1}{(\rho+\varepsilon)^{2T}N})$.
From the proof of Theorem~\ref{th:1}, we have for all $t\geq1$
\begin{align}
  & \E \|\theta_{t+1} - \theta_{t+1}^{(i)}\| \leq \frac{\varphi_{t+1}}{N} = O(\frac{1}{(\rho+\varepsilon)^{2(t+1)}N}), \nonumber \\
  &  \varphi_{t+1} = \underbrace{\textcolor{blue}{(1-\eta\lambda)\varphi_t}}_{(i)} + \frac{\eta\sqrt{d}\phi_{t+1}}{(1-\beta_1^{t+1})(\rho+\varepsilon)} \nonumber \\
  & \qquad \qquad +  \frac{\eta G\sqrt{d}\psi_{t+1}}{(1-\beta_1^{t+1})(1-\beta_2^{t+1})(\rho+\varepsilon)^2}, \nonumber
\end{align}
where
$\phi_{t+1}= \beta_1\phi_t + 2(1-\beta_1)\sigma + (1-\beta_1)L\varphi_t$ and $\psi_{t+1} = \beta_2 \psi_t + 4(1-\beta_2)G\sigma + 2(1-\beta_2)GL\varphi_t$.
For notational simplicity, let $\phi_{t+1}^{a}$, $\psi_{t+1}^{a}$ and $\varphi_{t+1}^{a}$ for our Adam-srf algorithm (i.e., Algorithm~\ref{alg:1} with $\lambda=0$); let $\phi_{t+1}^{w}$, $\psi_{t+1}^{w}$ and $\varphi_{t+1}^{w}$ for our AdamW-srf algorithm (i.e., Algorithm~\ref{alg:1} with $\lambda\in (0,\frac{1}{\eta})$).
From the above term $(i)$, we can find
$$\phi_{t+1}^{w} < \phi_{t+1}^{a}, \ \psi_{t+1}^{w} < \psi_{t+1}^{a}, \ \varphi_{t+1}^{w}< \varphi_{t+1}^{a}, \ t\geq 1. $$
Since $\varphi_{t+1}^{w} < \varphi_{t+1}^{a}$ for all
$t\geq1$, clearly, our AdamW-srf has a smaller
 generalization error than our Adam-srf due to using decoupled weight decay regularization.

 From the above results, generalization errors of the Adam(W)-srf algorithms rely on the smallest element of second-order momentum $\min_{j\in [d]}(\hat{v}_t)_j$. When element of the second-order momentum $(\hat{v}_t)_j$
 is very small, adaptive learning rate becomes too large (see Figure~\ref{fig:1} (a)), which affects the algorithm's generalization ability and stability.
\end{remark}

\subsection{Generalization Errors of our HomeAdam(W) Algorithms}

\begin{theorem} \label{th:2}
Assume the sequence $\{\theta_t\}_{t=1}^T$ is generated
from Algorithm~\ref{alg:2} on dataset $S=\{z_1,z_2,\cdots,z_N\}$, under the Assumptions~\ref{ass:s1},~\ref{ass:g},~\ref{ass:v}, without loss of generality, let $\tau\geq 1$, $\lambda\in [0,\frac{1}{\eta})$, $\beta_1=O(1)$ with $\beta_1\in (0,1)$, $\beta_2=O(1)$ with $\beta_2\in (0,1)$, $\sigma=O(1)$, $G=O(1)$ and $L=O(1)$.
 If the iteration number is small (i.e., $T=O(1)$) set
 $\eta=\frac{1}{\sqrt{d}}$, otherwise set $\eta=\frac{1}{\sqrt{d}T}$, we have
\begin{align}
|\E [ F(\theta_T) - F_S(\theta_T)]| \leq O(\frac{1}{N}).
\end{align}
\end{theorem}

\begin{remark} \label{re:2}
Based on the above Theorem~\ref{th:2},
when $\lambda=0$ in Algorithm~\ref{alg:2}, our HomeAdam algorithm has a generalization error of $O(\frac{1}{N})$.
When $\lambda\in (0,\frac{1}{\eta})$ in Algorithm~\ref{alg:2}, our HomeAdamW algorithm also has a generalization error of $O(\frac{1}{N})$.
From the proof of Theorem~\ref{th:2}, we have for all $t\geq1$
\begin{align}
  & \E \|\theta_{t+1} - \theta_{t+1}^{(i)}\| \leq \frac{\varphi_{t+1}}{N} = O(\frac{1}{N}), \nonumber \\
  &  \varphi_{t+1} = \underbrace{\textcolor{blue}{(1-\eta\lambda)\varphi_t}}_{(i)} + \frac{\eta\sqrt{d}\phi_{t+1}}{1-\beta_1^{t+1}} +  \frac{\eta G\sqrt{d}\psi_{t+1}}{(1-\beta_1^{t+1})(1-\beta_2^{t+1})}, \nonumber
\end{align}
where
$\phi_{t+1}= \beta_1\phi_t + 2(1-\beta_1)\sigma + (1-\beta_1)L\varphi_t$ and $\psi_{t+1} = \beta_2 \psi_t + 4(1-\beta_2)G\sigma + 2(1-\beta_2)GL\varphi_t$.
Similarly, let $\phi_{t+1}^{a}$, $\psi_{t+1}^{a}$ and $\varphi_{t+1}^{a}$ for our HomeAdam algorithm (i.e., Algorithm~\ref{alg:2} with $\lambda=0$); let $\phi_{t+1}^{w}$, $\psi_{t+1}^{w}$ and $\varphi_{t+1}^{w}$ for our HomeAdamW algorithm (i.e., Algorithm~\ref{alg:2} with $\lambda\in (0,\frac{1}{\eta})$).
From the above term $(i)$, we can find
$$\phi_{t+1}^{w} < \phi_{t+1}^{a}, \ \psi_{t+1}^{w} < \psi_{t+1}^{a}, \ \varphi_{t+1}^{w}< \varphi_{t+1}^{a}, \ t\geq 1. $$
Since $\varphi_{t+1}^{w} < \varphi_{t+1}^{a}$ for all
$t\geq1$, our HomeAdamW has a smaller
 generalization error than our HomeAdam due to using decoupled weight decay regularization.

 From Figure~\ref{fig:1} (c), our HomeAdam(W) algorithms keep the learning rate from becoming too large, which protects the algorithm's generalization ability and stability. Thus, our HomeAdam(W) algorithms have a smaller generalization error of $O(\frac{1}{N})$ than $O(\frac{1}{(\rho+\varepsilon)^{2T}N})$ of the Adam(W)-srf algorithms. Meanwhile, it also is smaller than $O(\frac{1}{\sqrt{N}})$ of the Adam and AdamW algorithms~\citep{zhou2024towards}.
\end{remark}

\section{Convergence Analysis}
In the section, we provide convergence analysis for our Adam(W)-srf and HomeAdam(W) algorithms under some mild conditions. All related proofs are provided in the following Appendices~\ref{ap:ca1} and~\ref{ap:ca2}.
We first give some mild assumptions.

\begin{assumption} \label{ass:s2}
The objective function $F(\theta)$ is $L$-Lipschitz smooth, such that
\begin{align}
    \|\nabla F(\theta_1)-\nabla F(\theta_2)\| \leq L\|\theta_1-\theta_2\|, \ \theta_1, \theta_2\in \R^{d}.
\end{align}
\end{assumption}

\begin{assumption} \label{ass:f}
 The objective function $F(\theta)$ has a smaller bounded, i.e.,  $F^* = \inf_{\theta\in \R^d} F(\theta)>-\infty$.
\end{assumption}

Assumption~\ref{ass:s2} shows smoothness of objective function $F(\theta)$, which is milder than the above
Assumption~\ref{ass:s1}. Assumption~\ref{ass:f} guarantees feasibility of the above problem~(\ref{eq:p1}), which also is commonly used in optimization~\citep{bottou2018optimization,ghadimi2013stochastic}.

\subsection{Convergence Properties of Adam(W)-srf Algorithms}

\begin{lemma} \label{lem:2}
Assume the sequence $\{m_t\}_{t=0}^T$ is generated from Algorithm~\ref{alg:1}, let $\beta_1 = 1-c\eta\in (0,1)$, $\beta_2\in (0,1)$, we have
\begin{align}
 &\E\|\nabla F(\theta_{t}) - m_{t+1}\|^2 \leq (1-c\eta)\mathbb{E} \|\nabla F(\theta_{t-1}) - m_{t} \|^2 \nonumber \\
 & \qquad \qquad \qquad + \frac{2}{c \eta}L^2\mathbb{E}\|\theta_t-\theta_{t-1}\|^2  + c^2\eta^2\sigma^2,
\end{align}
where $c>0$.
\end{lemma}

\begin{theorem} \label{th:3}
Assume the sequence $\{\theta_t\}_{t=0}^T$ is generated
from Algorithm~\ref{alg:1}. Under the Assumptions~\ref{ass:s2},~\ref{ass:g},~\ref{ass:v},~\ref{ass:f}, and let $0\leq \lambda < \min(\frac{1}{\eta},\frac{1}{\eta T^\gamma \bar{G}\hat{G}})$, $\|\theta_0\| \leq \eta \bar{G}$, $c\geq \frac{16L}{\breve{\rho}}$, $\beta_1 = 1-c\eta\in (0,1)$, $\beta_2\in (0,1)$ and $0<\eta \leq \frac{\breve{\rho}}{4L}$, we have
\begin{align}
\frac{1}{T+1}\sum_{t=0}^T\mathbb{E}\|\nabla F(\theta_{t})\|
  & \leq \frac{4\sqrt{2\Delta}\hat{G}}{\sqrt{T\eta\breve{\rho}}} + \frac{4\sqrt{2}\hat{G}}{T^{\gamma-1}\breve{\rho}} \nonumber \\
   & \quad  +
  \frac{4c\sigma\sqrt{\eta}\hat{G}}{\sqrt{L\breve{\rho}}} + \frac{1}{T^{\gamma-1}},
\end{align}
where $\bar{G}=\frac{G}{(1-\beta_1)(\rho+\varepsilon)}$, $\hat{G} =\frac{G^2+\varepsilon}{1-\beta_2}$, $\breve{\rho}=(1-\beta_1)(\rho+(1-\beta_2)\varepsilon)$ and $\rho = \min_{t\geq 1} \min_{j\in [d]}(\hat{v}_t)_j$.
\end{theorem}

\begin{remark}
Based on the above Theorem~\ref{th:3},
Let $\eta=\frac{1}{\sqrt{T}}$ and $\gamma=\frac{3}{4}$, we can obtain
\begin{align}
  \frac{1}{T+1}\sum_{t=0}^T\mathbb{E}\|\nabla F(\theta_{t})\|
   & \leq \frac{4\sqrt{2\Delta}\hat{G}}{\sqrt{\breve{\rho}}T^{1/4}} + \frac{4\sqrt{2}\hat{G}}{\breve{\rho}T^{1/4}} \nonumber \\
   & \quad +
  \frac{4c\sigma\hat{G}}{\sqrt{L\breve{\rho}}T^{1/4}}+ \frac{1}{T^{1/4}}. \nonumber
\end{align}
Set $G=O(1)$, $L=O(1)$, $\sigma=O(1)$, $\beta_1=O(1)$ with $\beta_1\in(0,1)$, $\beta_2=O(1)$ with $\beta_2\in(0,1)$ and $c=O(1)$, we have $\hat{G}=\frac{G^2+\varepsilon}{1-\beta_2}=O(1)$ and $\Delta = F(\theta_0) + \frac{1}{L}(\sigma^2 + \beta_1^2G^2) - F^*=O(1)$. Since $\breve{\rho}=(1-\beta_1)(\rho+(1-\beta_2)\varepsilon) \leq \rho+\varepsilon$ and $\varepsilon$ is very small, $\breve{\rho}$ also is very small. Then we have
\begin{align}
  \frac{1}{T+1}\sum_{t=0}^T\mathbb{E}\|\nabla F(\theta_{t})\|  \leq O(\frac{\breve{\rho}^{-1}}{T^{1/4}}).
\end{align}

\end{remark}

\begin{figure*}[ht]
\centering
 \subfigure{\includegraphics[width=0.23\textwidth]{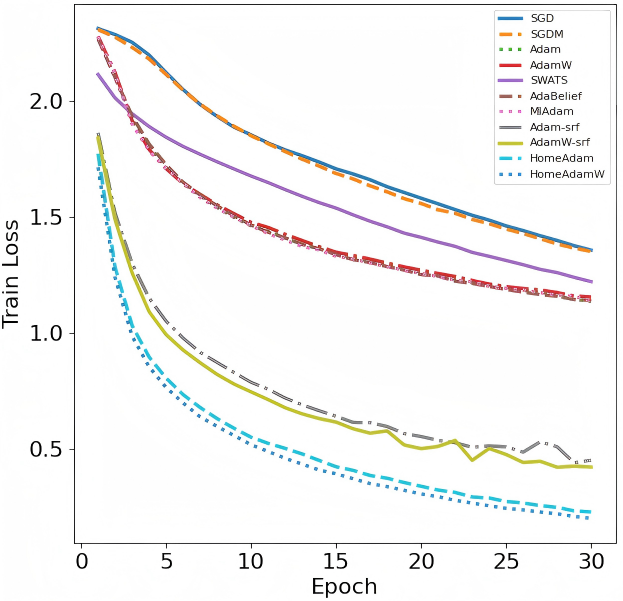}}
  \hfill
 \subfigure{\includegraphics[width=0.23\textwidth]{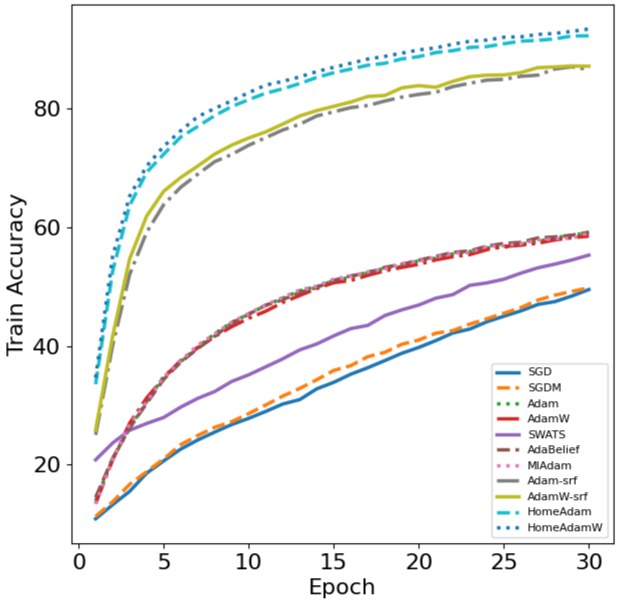}}
  \hfill
  \subfigure{\includegraphics[width=0.23\textwidth]{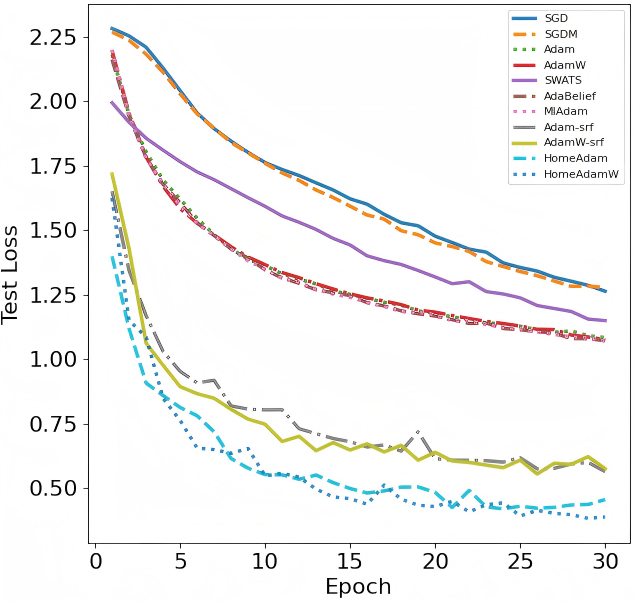}}
  \hfill
  \subfigure{\includegraphics[width=0.23\textwidth]{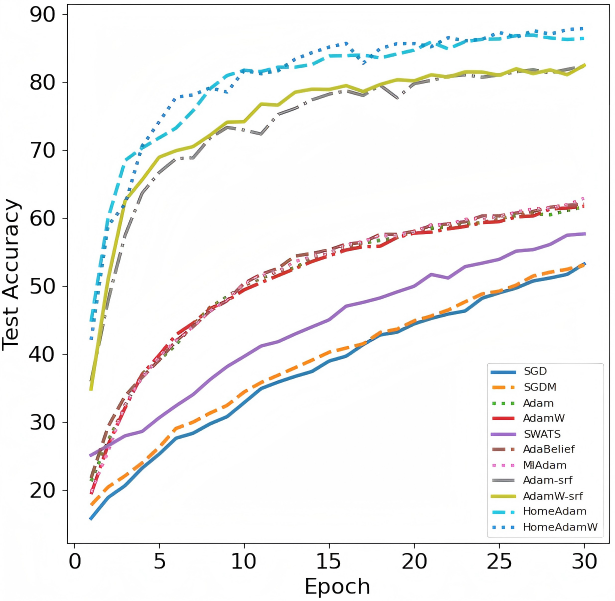}}
  \hfill
\caption{Results of image classification at \emph{Cifar-10} dataset.}
\label{fig:2}
\end{figure*}

\begin{figure*}[ht]
\centering
 \subfigure{\includegraphics[width=0.23\textwidth]{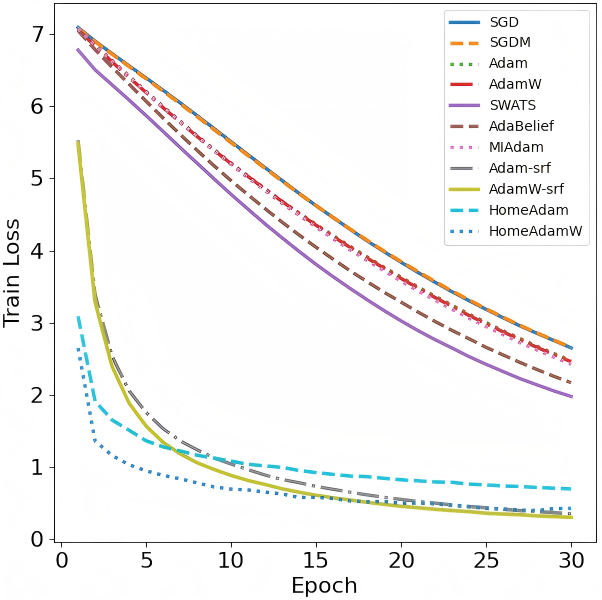}}
  \hfill
 \subfigure{\includegraphics[width=0.23\textwidth]{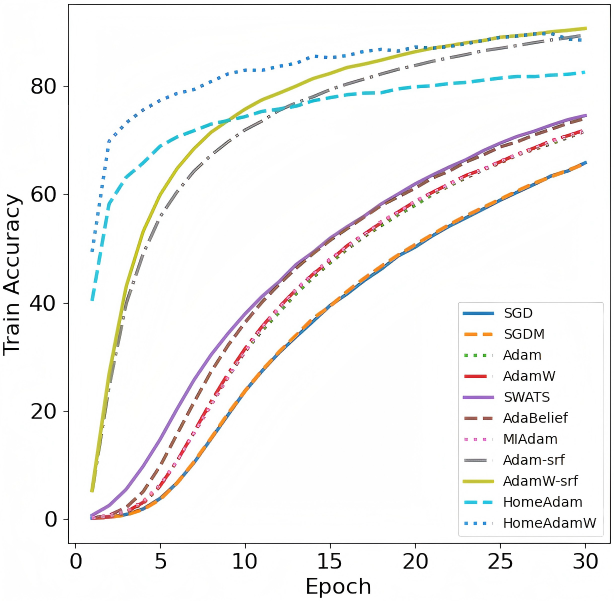}}
  \hfill
  \subfigure{\includegraphics[width=0.23\textwidth]{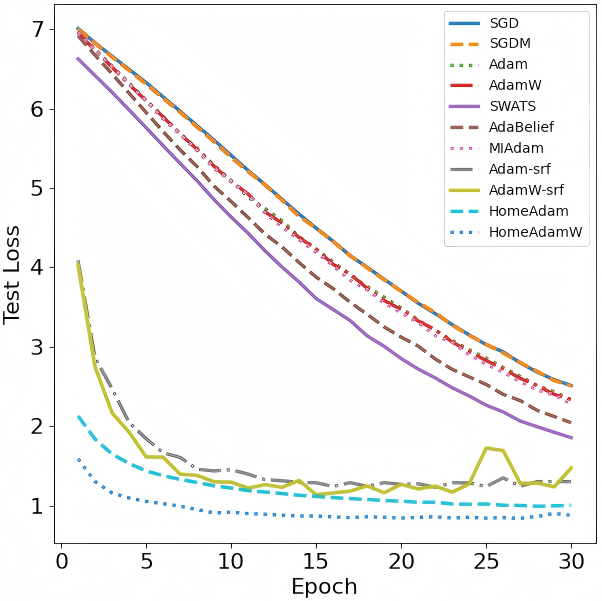}}
  \hfill
  \subfigure{\includegraphics[width=0.23\textwidth]{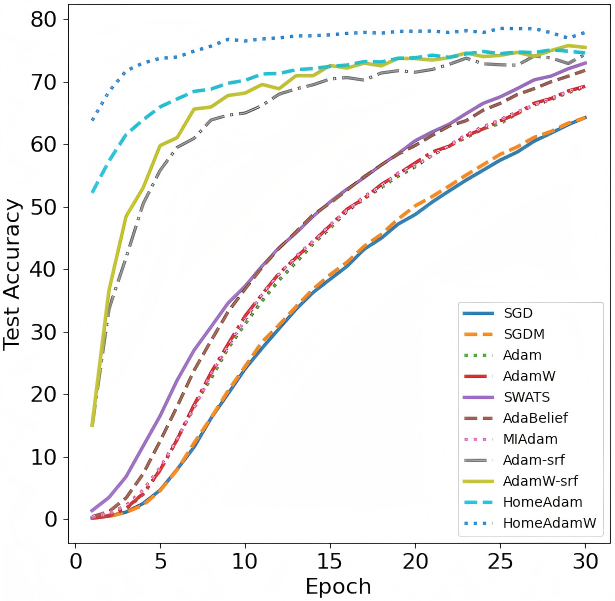}}
  \hfill
\caption{Results of image classification at \emph{tiny-ImageNet} dataset.}
\label{fig:3}
\end{figure*}

\subsection{Convergence Properties of our HomeAdam(W) Algorithms}

\begin{theorem} \label{th:4}
Assume the sequence $\{\theta_t\}_{t=0}^T$ is generated
from Algorithm~\ref{alg:2}. Under the Assumptions~\ref{ass:s2},~\ref{ass:g},~\ref{ass:v},~\ref{ass:f}, and let $0\leq \lambda < \min(\frac{1}{\eta},\frac{1}{\eta T^\gamma \tilde{G}\hat{G}})$, $\|\theta_0\| \leq \eta \tilde{G}$, $c\geq \frac{32L}{\breve{\tau}}$, $\beta_1 = 1-c\eta\in (0,1)$, $\beta_2\in (0,1)$, $0<\eta \leq \frac{\hat{\tau}}{4L}$ and $\tau>0$, we have
\begin{align}
\frac{1}{T+1}\sum_{t=0}^T\mathbb{E}\|\nabla F(\theta_{t})\|
  & \leq \frac{8\sqrt{\Delta}\breve{G}}{\sqrt{{T}\eta\breve{\tau}}}
  + \frac{8\breve{G}}{\breve{\tau}T^{\gamma-1}} \nonumber \\
  & \quad +
  \frac{4\sqrt{2\eta}c\sigma\breve{G}}{\sqrt{L\breve{\tau}}}+ \frac{\breve{G}}{\hat{G}T^{\gamma-1}},
\end{align}
where $\tilde{G}=\max(\frac{ G}{(1-\beta_1)(\tau+(1-\beta_2)\varepsilon)},\frac{G}{1-\beta_1})$, $\hat{G} =\frac{G^2+\varepsilon}{1-\beta_2}$, $\breve{G} = \max(1,\frac{G^2+\varepsilon}{1-\beta_2})$, $\hat{\tau} = (1-\beta_1)(\tau + (1-\beta_2)\varepsilon)$ and $\breve{\tau}=\min(1-\beta_1,\hat{\tau})$.
\end{theorem}

\begin{remark}
Let $\eta=\frac{1}{\sqrt{T}}$ and $\gamma=\frac{3}{4}$, we have
\begin{align}
  \frac{1}{T}\sum_{t=0}^T\mathbb{E}\|\nabla F(\theta_{t})\| & \leq \breve{G}\big(\frac{8\sqrt{\Delta}}{\sqrt{\breve{\tau}}T^{1/4}} + \frac{8}{\breve{\tau}T^{1/4}} \nonumber \\
  &\quad  +
  \frac{4\sqrt{2}c\sigma}{\sqrt{L\breve{\tau}}T^{1/4}}+ \frac{1}{\hat{G}T^{1/4}}\big).
\end{align}
Set $G=O(1)$, $L=O(1)$, $\tau=O(1)$ and $c=O(1)$,
since $\eta=\frac{1}{\sqrt{T}}$, we have $\beta_1=1-c\eta=O(1)$, $\hat{\tau} = (1-\beta_1)(\tau + (1-\beta_2)\varepsilon)=O(1)$ and $\breve{\tau}=\min(1-\beta_1,\hat{\tau})=\min(1-\beta_1,(1-\beta_1)(\tau + (1-\beta_2)\varepsilon))=O(1)$.
Then we have
\begin{align}
  \frac{1}{T+1}\sum_{t=0}^T\mathbb{E}\|\nabla F(\theta_{t})\| \leq O(\frac{1}{T^{1/4}}).
\end{align}
Here our HomeAadm(W) algorithms have a faster convergence rate of $O(\frac{1}{T^{1/4}})$ than $O(\frac{\breve{\rho}^{-1}}{T^{1/4}})$ of
the Aadm(W)-srf algorithms, since $\breve{\rho}$ is generally very small.

\end{remark}

\begin{figure*}[ht]
\centering
 \subfigure{\includegraphics[width=0.23\textwidth]{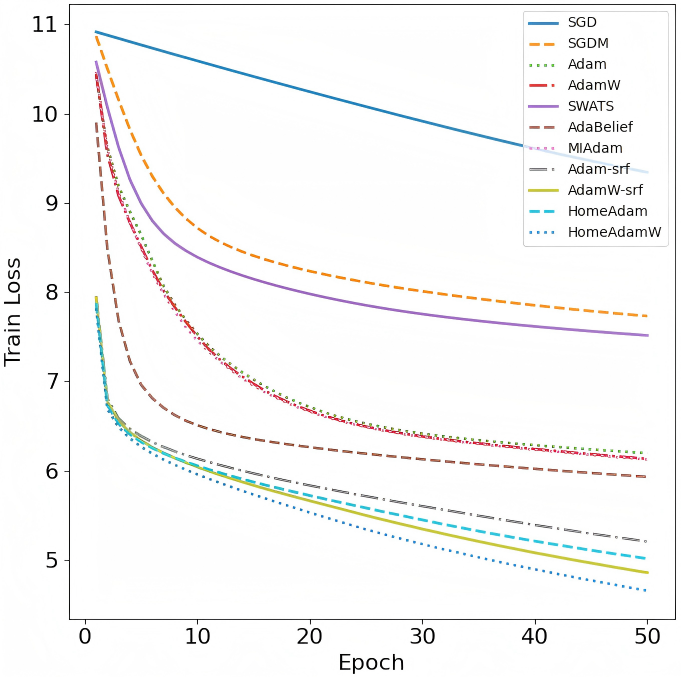}}
  \hfill
 \subfigure{\includegraphics[width=0.23\textwidth]{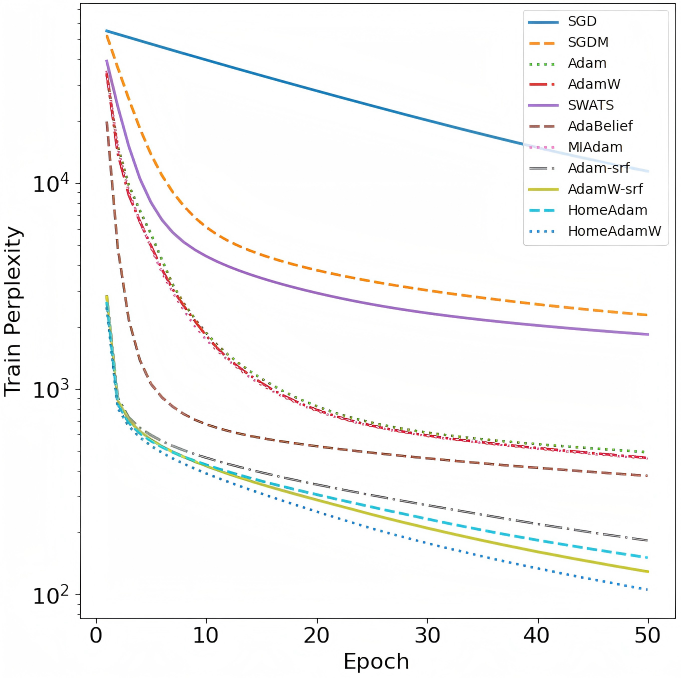}}
  \hfill
  \subfigure{\includegraphics[width=0.23\textwidth]{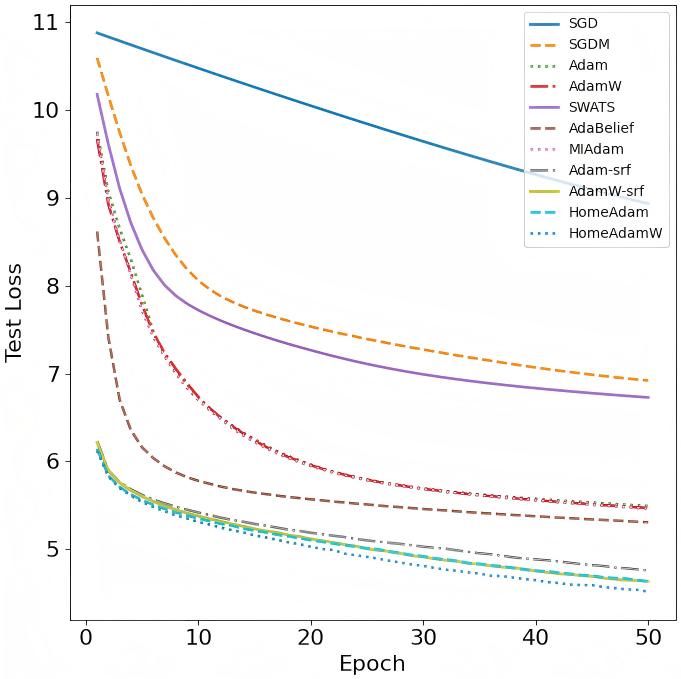}}
  \hfill
  \subfigure{\includegraphics[width=0.23\textwidth]{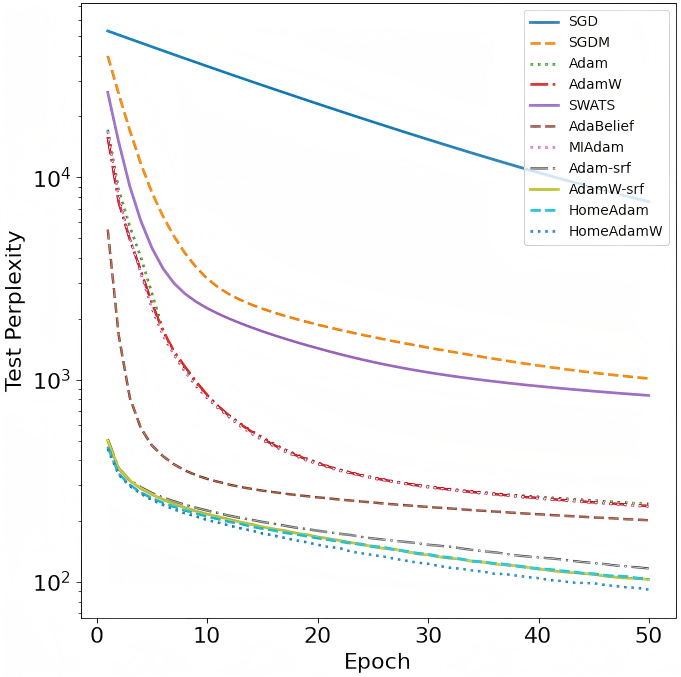}}
  \hfill
\caption{Results of language modeling task at \emph{Wikitext-2} dataset.}
\label{fig:4}
\end{figure*}

\begin{figure*}[ht]
\centering
 \subfigure{\includegraphics[width=0.23\textwidth]{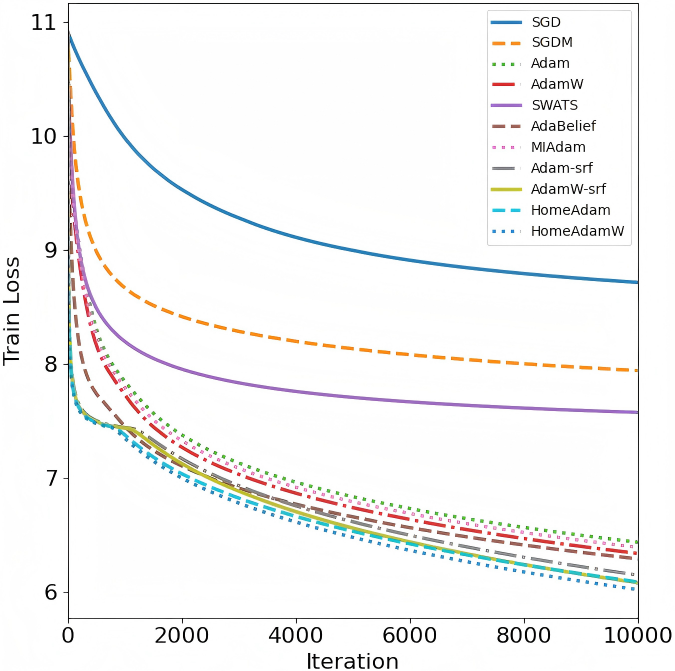}}
  \hfill
 \subfigure{\includegraphics[width=0.23\textwidth]{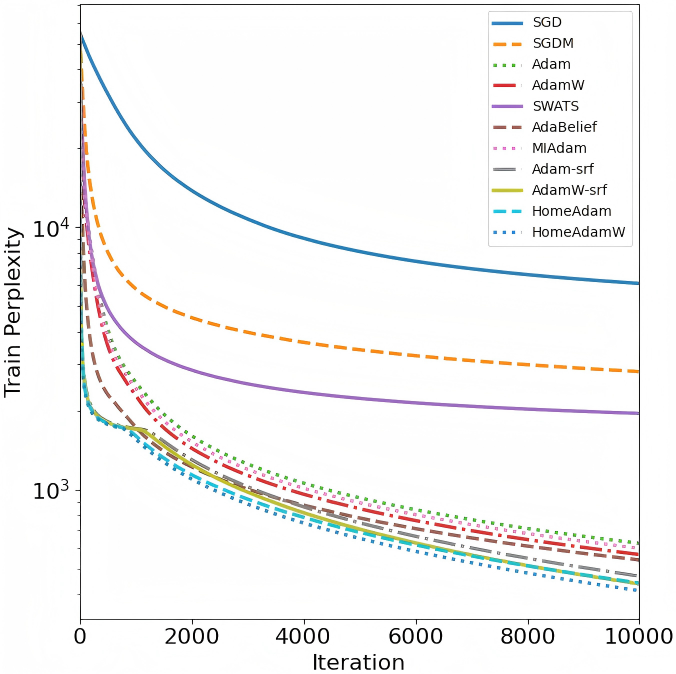}}
  \hfill
  \subfigure{\includegraphics[width=0.23\textwidth]{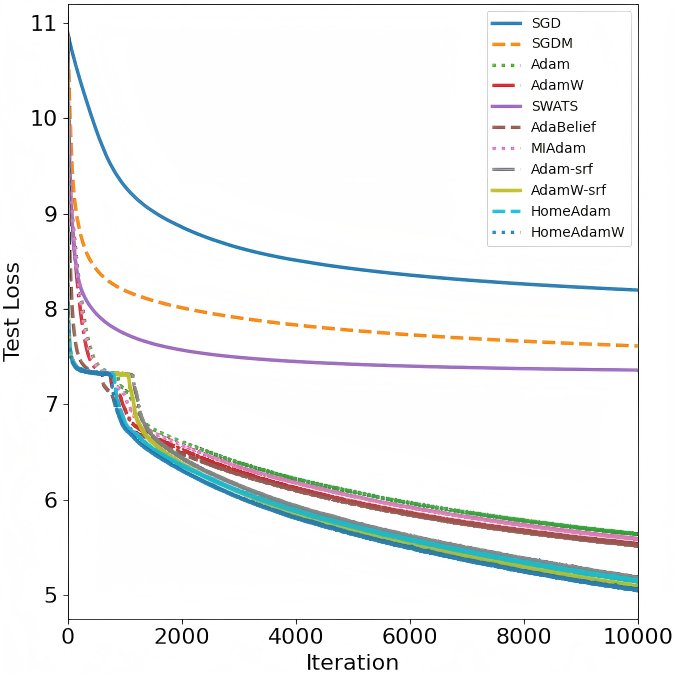}}
  \hfill
  \subfigure{\includegraphics[width=0.23\textwidth]{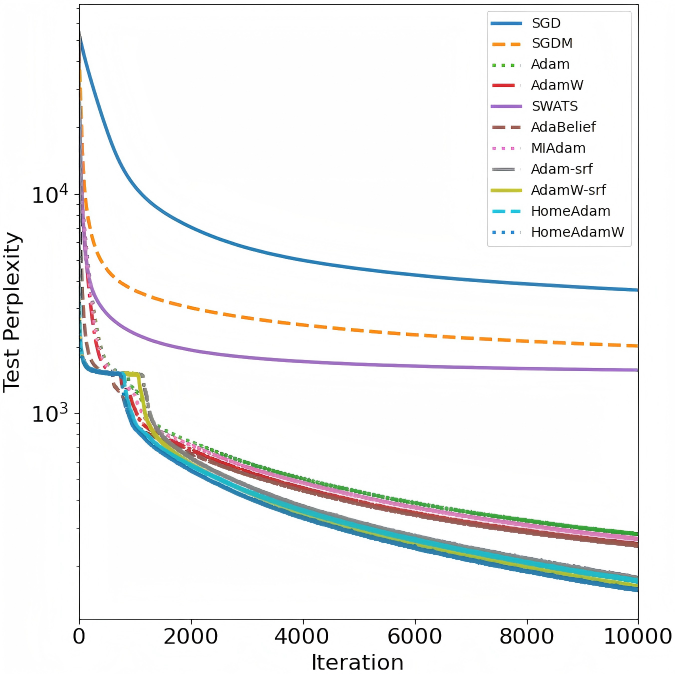}}
  \hfill
\caption{Results of language modeling task at \emph{Wikitext-103} dataset.}
\label{fig:5}
\end{figure*}
\vspace*{-8pt}
\section{Numerical Experiments}
In the section, we conduct some numerical experiments on Computer Vision (CV) and
 Natural Language Processing (NLP) tasks to
demonstrate effectiveness of our algorithms.
In the experiments, we compare our Adam(W)-srf and HomeAdam(W) methods with
some typical methods such as SGD, SGDM, Adam~\cite{kingma2014adam}, AdamW~~\cite{loshchilov2017decoupled}, SWATS~\citep{keskar2017improving}, AdaBelief~\cite{zhuang2020adabelief} and MIAdam~\citep{jin2025method}.
All experiments are run over a machine with a 24 vCPU 13th Gen Intel(R) Core(TM) i9-13900KF CPU
and 2 Nvidia RTX 4090 GPU.
\vspace*{-8pt}
\subsection{CV Task}
In this experiment, we implement image classification task on the CIFAR-10~\citep{krizhevsky2009learning} and
Tiny-ImageNet~\citep{le2015tiny} datasets. Specifically, we train the VGG16~\citep{simonyan2014very} and
ResNet34~\citep{he2016deep} models on the CIFAR-10 and Tiny-ImageNet datasets, respectively.
The hyper-parameters used in all algorithms are provided in the Appendix~\ref{ap:ex}.

Figures~\ref{fig:2} and~\ref{fig:3} show the training loss, training accuracy, test loss and test accuracy via  \emph{epoch} at image classification task, where the iteration number of each epoch equals training sample size. From these results, our Adam(W)-srf and HomeAdam(W) algorithms outperform other comparisons, which verifies effectiveness of square-root-free in Adam(W). Meanwhile, our HomeAdam(W) algorithms have higher test accuracy than our Adam(W)-srf algorithms, which verifies better generalization of our HomeAdam(W) algorithms that
also has been supported in the above generalization results. In addition, our HomeAdamW algorithm has higher test accuracy than our HomeAdam, which verifies better generalization of our HomeAdamW algorithm by using weight decay
that also has a theoretical support in Remark~\ref{re:2}.

\vspace*{-8pt}
\subsection{NLP Task}
In this experiment, we conduct language modeling task on the WikiText2~\citep{merity2016pointer} and WikiText-103~\citep{merity2016pointer} datasets. Specifically, we train the 8-layer Transformer~\cite{vaswani2017attention} and
24-layer Transformer models on the WikiText2 and WikiText-103 datasets, respectively.
The 8-layer Transformer and 24-layer Transformer models are described in the Appendix~\ref{ap:ex}.
The hyper-parameters used in all algorithms also are provided in the Appendix~\ref{ap:ex}.

Figures~\ref{fig:4} and~\ref{fig:5} provide training loss, training perplexity, test loss and test perplexity
via epoch or iteration number at language modeling task. From these results, our Adam(W)-srf and HomeAdam(W) algorithms outperform other comparisons, which demonstrates effectiveness of square-root-free in Adam(W). Meanwhile, our HomeAdam(W) algorithms have smaller test perplexity than our Adam(W)-srf algorithms,
which further verifies better generalization of our HomeAdam(W) algorithms that
also has been supported in the above generalization results.
In addition, our HomeAdamW algorithm has smaller test perplexity than
our HomeAdam, which verifies better generalization of our HomeAdamW algorithm by using weight decay
that also has a theoretical support in Remark~\ref{re:2}.
\vspace*{-8pt}
\section{Conclusion}
In the paper, we restudied generalization properties of the popular Adam and AdamW algorithms.
We first introduced a class of square-root-free Adam (i.e., Adam(W)-srf) algorithms, and proved that our Adam(W)-srf algorithms have a generalization error of $O(\frac{\hat{\rho}^{-2T}}{N})$.
To improve its generalization, we further designed a class of
efficient Adam (i.e., HomeAdam(W)) algorithms by sometimes returning momentum-based SGD.
Moreover, we proved that
our HomeAdam(W) methods have a smaller generalization error of $O(\frac{1}{N})$ than
the existing $O(\frac{1}{\sqrt{N}})$ of the Adam and AdamW algorithms.
From our generalization analysis, we also find that our HomeAdamW has a better
generalization than our HomeAdam, due to using weight decay in our HomeAdamW.
In the Appendix~\ref{ap:ew}, we also provide an element-wise variant of our HomeAdam(W) algorithms,
which is more suitable for training deep learning models due to matching the back-propagation framework.


% Authors are kindly asked to make their submissions as accessible as possible
% for everyone including people with disabilities and sensory or neurological
% differences. Tips of how to achieve this and what to pay attention to will be
% provided on the conference website \url{http://icml.cc/}.

% \section*{Software and Data}

% If a paper is accepted, we strongly encourage the publication of software and
% data with the camera-ready version of the paper whenever appropriate. This can
% be done by including a URL in the camera-ready copy. However, \textbf{do not}
% include URLs that reveal your institution or identity in your submission for
% review. Instead, provide an anonymous URL or upload the material as
% ``Supplementary Material'' into the OpenReview reviewing system. Note that
% reviewers are not required to look at this material when writing their review.

% % Acknowledgements should only appear in the accepted version.
% \section*{Acknowledgements}

% \textbf{Do not} include acknowledgements in the initial version of the paper

% If a paper is accepted, the final camera-ready version can (and usually should)
% include acknowledgements.  Such acknowledgements should be placed at the end of
% the section, in an unnumbered section that does not count towards the paper
% page limit. Typically, this will include thanks to reviewers who gave useful
% comments, to colleagues who contributed to the ideas, and to funding agencies
% and corporate sponsors that provided financial support.
\vspace*{-8pt}
\section*{Impact Statement}
This paper presents work whose goal is to advance the field of Machine
Learning. There are many potential societal consequences of our work, none
which we feel must be specifically highlighted here.

\nocite{langley00}

\bibliography{HomeAdam}
\bibliographystyle{icml2026}

%%%%%%%%%%%%%%%%%%%%%%%%%%%%%%%%%%%%%%%%%%%%%%%%%%%%%%%%%%%%%%%%%%%%%%%%%%%%%%%
%%%%%%%%%%%%%%%%%%%%%%%%%%%%%%%%%%%%%%%%%%%%%%%%%%%%%%%%%%%%%%%%%%%%%%%%%%%%%%%
% APPENDIX
%%%%%%%%%%%%%%%%%%%%%%%%%%%%%%%%%%%%%%%%%%%%%%%%%%%%%%%%%%%%%%%%%%%%%%%%%%%%%%%
%%%%%%%%%%%%%%%%%%%%%%%%%%%%%%%%%%%%%%%%%%%%%%%%%%%%%%%%%%%%%%%%%%%%%%%%%%%%%%%
\newpage
\appendix
\onecolumn

\section{Generalization Analysis of Adam(W)-srf Algorithms}
\label{ap:ga1}
In this section, we provide the detailed generalization analysis of our Adam(W)-srf algorithm.

\begin{lemma} \label{lem:A1} (Restatement of Lemma~\ref{lem:1})
Assume the sequences $\{\hat{m}_t,\hat{v}_t,m_t,v_t\}_{t=1}^T$ and $\{\hat{m}_t^{(i)},\hat{v}_t^{(i)},m_t^{(i)},v_t^{(i)}\}_{t=1}^T$ are generated from Algorithm~\ref{alg:1} based on the dataset $S$ and
$S^{(i)}$, respectively, we have
\begin{align}
\big\|\frac{\hat{m}_t}{\hat{v}_t + \varepsilon} -  \frac{\hat{m}_t^{(i)}}{\hat{v}_t^{(i)} + \varepsilon}\big\|  \leq
 \frac{\sqrt{d}}{(1-\beta_1^t)(\rho_t+\varepsilon)}\big\|m_t - m_t^{(i)}\big\| + \frac{G\sqrt{d}}{(1-\beta_1^t)(1-\beta_2^t)(\rho_t+\varepsilon)^2}\big\| v_t -v_t^{(i)}\big\|,
\end{align}
where $\rho_t = \min_{j\in [d]}(\hat{v}_t)_j$.
\end{lemma}

\begin{proof}
From Algorithm~\ref{alg:1}, since $S, S^{(i)}\sim \mathcal{D}$, we have $m_{t}= \beta_{1}m_{t-1}+ (1-\beta_{1})g_t$, $v_{t}= \beta_{2}v_{t-1}+ (1-\beta_{2})g_t^2$, $m_{t}^{(i)}= \beta_{1}m_{t-1}^{(i)}+ (1-\beta_{1})g_t^{(i)}$ and $v_{t}^{(i)}= \beta_{2}v_{t-1}^{(i)}+ (1-\beta_{2})(g_t^{(i)})^2$. Meanwhile, we have $\hat{m}_{t}=\frac{m_{t}}{1-\beta_1^t}$, $\hat{v}_{t}=\frac{v_{t}}{1-\beta_2^t}$, $\hat{m}_{t}^{(i)}=\frac{m_{t}^{(i)}}{1-\beta_1^t}$ and $\hat{v}_{t}^{(i)}=\frac{v_{t}^{(i)}}{1-\beta_2^t}$.

Let $(\hat{m}_t)_j$, $(\hat{v}_t)_j$, $(\hat{m}_t^{(i)})_j$ and $(\hat{v}_t^{(i)})_j$ denote the $j$-th element of vectors $\hat{m}_{t}$, $\hat{v}_t$,  $\hat{m}_{t}^{(i)}$ and $\hat{v}_{t}^{(i)}$, respectively.

Since $\rho_t=\min_{j\in [d]}(\hat{v}_t)_j$, where $\hat{v}_t$ be generated from Algorithm~\ref{alg:1} for any dataset $S\sim \mathcal{D}$, we have $\rho_t=\min_{j\in [d]}(\hat{v}_t^{(i)})_j$, where $\hat{v}_t^{(i)}$ be generated from Algorithm~\ref{alg:1} based on dataset $S^{(i)}\sim \mathcal{D}$. Since $m_t^{(i)}$ is exponential moving average of $g_t^{(i)}$, by using Assumption~\ref{ass:g}, we have $\|m_t^{(i)}\|\leq G$.
Then we have
\begin{align}
 & \big| \frac{(\hat{m}_t)_j}{(\hat{v}_t)_j + \varepsilon} -  \frac{(\hat{m}_t^{(i)})_j}{(\hat{v}_t^{(i)})_j + \varepsilon}\big| \nonumber \\
 & = \frac{1}{((\hat{v}_t)_j + \varepsilon)((\hat{v}_t^{(i)})_j + \varepsilon)}\big|((\hat{v}_t^{(i)})_j + \varepsilon)(\hat{m}_t)_j -  (\hat{m}_t^{(i)})_j((\hat{v}_t)_j + \varepsilon)\big| \nonumber \\
 & \leq \frac{1}{((\hat{v}_t)_j + \varepsilon)((\hat{v}_t^{(i)})_j + \varepsilon)}\Big(\big|((\hat{v}_t^{(i)})_j + \varepsilon)(\hat{m}_t)_j - (\hat{m}_t^{(i)})_j((\hat{v}_t^{(i)})_j +\varepsilon) \big| \nonumber \\
 & \qquad+\big|(\hat{m}_t^{(i)})_j((\hat{v}_t^{(i)})_j + \varepsilon) -  (\hat{m}_t^{(i)})_j((\hat{v}_t)_j + \varepsilon)\big| \Big)\nonumber \\
 & = \frac{1}{(\hat{v}_t)_j + \varepsilon}\big|(\hat{m}_t)_j - (\hat{m}_t^{(i)})_j\big| + \frac{(\hat{m}_t^{(i)})_j}{((\hat{v}_t)_j + \varepsilon)((\hat{v}_t^{(i)})_j + \varepsilon)}\big|(\hat{v}_t^{(i)})_j - (\hat{v}_t)_j\big| \nonumber \\
 & \leq \frac{1}{\rho_t+\varepsilon}\big|(\hat{m}_t)_j - (\hat{m}_t^{(i)})_j\big| + \frac{G}{(1-\beta_1^t)(\rho_t+\varepsilon)^2}\big|(\hat{v}_t^{(i)})_j - (\hat{v}_t)_j\big| \nonumber \\
 & = \frac{1}{(1-\beta_1^t)(\rho_t+\varepsilon)}\big|(m_t)_j - (m_t^{(i)})_j\big| + \frac{G}{(1-\beta_1^t)(1-\beta_2^t)(\rho_t+\varepsilon)^2}\big|(v_t^{(i)})_j - (v_t)_j\big|,
\end{align}
where the last inequality is due to $(\hat{m}_t^{(i)})_j\leq \|\hat{m}_t^{(i)}\|\leq \frac{G}{1-\beta_1^t}$.

Thus we can obtain
\begin{align}
 \big\|\frac{\hat{m}_t}{\hat{v}_t + \varepsilon} -  \frac{\hat{m}_t^{(i)}}{\hat{v}_t^{(i)} + \varepsilon}\big\|_1 \leq
 \frac{1}{(\rho_t+\varepsilon)(1-\beta_1^t)}\big\|m_t - m_t^{(i)}\big\|_1 + \frac{G}{(1-\beta_1^t)(1-\beta_2^t)(\rho_t+\varepsilon)^2}\big\| v_t -v_t^{(i)}\big\|_1.
\end{align}
Since $\|\cdot\| \leq \|\cdot\|_1 \leq \sqrt{d}\|\cdot\|$, we have
\begin{align}
\big\|\frac{\hat{m}_t}{\hat{v}_t + \varepsilon} -  \frac{\hat{m}_t^{(i)}}{\hat{v}_t^{(i)} + \varepsilon}\big\| & \leq \big\|\frac{\hat{m}_t}{\hat{v}_t + \varepsilon} -  \frac{\hat{m}_t^{(i)}}{\hat{v}_t^{(i)} + \varepsilon}\big\|_1 \nonumber \\
 &\leq
 \frac{1}{(\rho_t+\varepsilon)(1-\beta_1^t)}\big\|m_t - m_t^{(i)}\big\|_1 + \frac{G}{(1-\beta_1^t)(1-\beta_2^t)(\rho_t+\varepsilon)^2}\big\| v_t -v_t^{(i)}\big\|_1 \nonumber \\
 & \leq
 \frac{\sqrt{d}}{(\rho_t+\varepsilon)(1-\beta_1^t)}\big\|m_t - m_t^{(i)}\big\| + \frac{G\sqrt{d}}{(1-\beta_1^t)(1-\beta_2^t)(\rho_t+\varepsilon)^2}\big\| v_t -v_t^{(i)}\big\|.
\end{align}

\end{proof}

\begin{theorem} \label{th:A1} (Restatement of Theorem~\ref{th:1})
Assume the sequence $\{\theta_t,\hat{v}_t\}_{t=1}^T$ is generated from Algorithm~\ref{alg:1} on dataset $S=\{z_1,z_2,\cdots,z_N\}$. Under the Assumptions~\ref{ass:s1},~\ref{ass:g},~\ref{ass:v}, let $\eta=O(\frac{1}{\sqrt{d}})$, $\lambda\in [0,\frac{1}{\eta})$, $\beta_1=O(1)$ with $\beta_1\in (0,1)$, $\beta_2=O(1)$ with $\beta_2\in (0,1)$, $\sigma=O(1)$, $G=O(1)$ and
$L=O(1)$, we have
\begin{align}
 |\E [ F(\theta_T) - F_S(\theta_T)]| \leq O(\frac{1}{(\rho+\varepsilon)^{2T}N}),
\end{align}
where $\rho = \min_{t\geq 1} \min_{j\in [d]}(\hat{v}_t)_j$.
\end{theorem}

\begin{proof}
Implementing Algorithm~\ref{alg:1} on datasets $S$ and $S^{(i)}$ with
 the same random index sequence $\{j_t\}_{t=1}^T$, and let $\{\theta_t\}_{t=1}^T$ and $\{\theta_t^{(i)}\}_{t=1}^T$ be generated from Algorithm~\ref{alg:1} with $S$ and $S^{(i)}$.

 From Algorithm~\ref{alg:1}, $\theta_{t} = \theta_{t-1} - \eta (\frac{\hat{m}_t}{\hat{v}_t+\varepsilon}+\lambda \theta_{t-1})$ and $\theta_{t}^{(i)} = \theta_{t-1}^{(i)} - \eta (\frac{\hat{m}_t^{(i)}}{\hat{v}_t^{(i)}+\varepsilon}+\lambda\theta_{t-1}^{(i)})$, we have
\begin{align}
 \theta_{t} - \theta_{t}^{(i)}  = (1-\eta\lambda)(\theta_{t-1} -\theta_{t-1}^{(i)}) - \eta \big(\frac{\hat{m}_t}{\hat{v}_t+\varepsilon} -
 \frac{\hat{m}_t^{(i)}}{\hat{v}_t^{(i)}+\varepsilon}\big).
\end{align}
Then we have
 \begin{align} \label{eq:re}
 \|\theta_{t} - \theta_{t}^{(i)}\| & =\|(1-\eta\lambda)(\theta_{t-1} -\theta_{t-1}^{(i)}) - \eta \big(\frac{\hat{m}_t}{\hat{v}_t+\varepsilon} -
 \frac{\hat{m}_t^{(i)}}{\hat{v}_t^{(i)}+\varepsilon}\big)\| \nonumber \\
 & \leq (1-\eta\lambda)\|\theta_{t-1} -\theta_{t-1}^{(i)}\| + \eta \|\frac{\hat{m}_t}{\hat{v}_t+\varepsilon} -
 \frac{\hat{m}_t^{(i)}}{\hat{v}_t^{(i)}+\varepsilon}\| \nonumber \\
 & \leq (1-\eta\lambda)\|\theta_{t-1} -\theta_{t-1}^{(i)}\| +\frac{\eta\sqrt{d}}{(1-\beta_1^t)(\rho_t+\varepsilon)}\big\|m_t - m_t^{(i)}\big\| \nonumber \\
 &\quad + \frac{\eta G\sqrt{d}}{(1-\beta_1^t)(1-\beta_2^t)(\rho_t+\varepsilon)^2}\big\| v_t -v_t^{(i)}\big\|
 \nonumber \\
 & \leq (1-\eta\lambda)\|\theta_{t-1} -\theta_{t-1}^{(i)}\| +\frac{\eta\sqrt{d}}{(1-\beta_1^t)(\rho+\varepsilon)}\big\|m_t - m_t^{(i)}\big\| \nonumber \\
 &\quad + \frac{\eta G\sqrt{d}}{(1-\beta_1^t)(1-\beta_2^t)(\rho+\varepsilon)^2}\big\| v_t -v_t^{(i)}\big\|,
 \end{align}
where the first inequality is due to $0\leq \lambda< \frac{1}{\eta}$, and the second last inequality holds by
Lemma~\ref{lem:1}, and the last inequality is due to $\rho = \min_{t\geq 1} \min_{j\in [d]}(\hat{v}_t)_j$.

 If $j_1\neq i$ with probability $1-\frac{1}{N}$, since $m_1=(1-\beta_1)\nabla f(\theta_0;z_{j_1})$, $m_1^{(i)}=(1-\beta_1)\nabla f(\theta_0^{(i)};z_{j_1})$, $v_1=(1-\beta_2)(\nabla f(\theta_0;z_{j_1}))^2$, $v_1^{(i)}=(1-\beta_2)(\nabla f(\theta_0^{(i)};z_{j_1}))^2$ and $\theta_0=\theta_0^{(i)}$, we have $m_1=m_1^{(i)}$ and $v_1=v_1^{(i)}$.

 If $j_1= i$ with probability $\frac{1}{N}$, we have
 \begin{align}
 & \E \|m_1 - m_1^{(i)}\| \nonumber \\
 & = \frac{1}{N}\E\|(1-\beta_1)\nabla f(\theta_0;z_i)-(1-\beta_1)\nabla f(\theta_0^{(i)};\tilde{z}_i)\|
 \nonumber \\
  & =\frac{(1-\beta_1)}{N}\E\|\nabla f(\theta_0;z_i)-\nabla F(\theta_0) + \nabla F(\theta_0) - \nabla F(\theta_0^{(i)})+\nabla F(\theta_0^{(i)})-\nabla f(\theta_0^{(i)};\tilde{z}_i)\| \nonumber \\
  & \leq \frac{2(1-\beta_1)\sigma}{N} + \frac{(1-\beta_1)}{N}\E\| \nabla F(\theta_0) - \nabla F(\theta_0^{(i)})\| \nonumber \\
  & = \frac{2(1-\beta_1)\sigma}{N},
 \end{align}
where the last equality is due to $\theta_0=\theta_0^{(i)}$.
 We also have
 \begin{align}
   \E \|v_1 - v_1^{(i)}\|
  & = \frac{1}{N}\E\|(1-\beta_2)(\nabla f(\theta_0;z_i))^2-(1-\beta_2)\nabla f(\theta_0^{(i)};\tilde{z}_i)\| \nonumber \\
  & =\frac{(1-\beta_2)}{N}\E\|(\nabla f(\theta_0;z_i))^2-(\nabla f(\theta_0^{(i)};\tilde{z}_i))^2\| \nonumber \\
  & \leq\frac{(1-\beta_2)}{N}\big(\E\|(\nabla f(\theta_0;z_i))^2\|+\E\|(\nabla f(\theta_0^{(i)};\tilde{z}_i))^2\| \big) \nonumber \\
  & \leq \frac{2(1-\beta_2)G^2}{N},
 \end{align}
Let $\phi_1=2(1-\beta_1)\sigma$ and $\psi_1 = 2(1-\beta_2)G^2$, we have
\begin{align}
 \E \|m_1 - m_1^{(i)}\| \leq \frac{\phi_1}{N}, \quad   \E \|v_1 - v_1^{(i)}\| \leq \frac{\psi_1}{N}.
\end{align}
Let $\varphi_1 = \frac{2\eta\sqrt{d}\sigma}{\rho_1+\varepsilon}+\frac{2\eta\sqrt{d}G^3}{(1-\beta_1)(\rho_1+\varepsilon)^2}$, since $\theta_{0} =\theta_{0}^{(i)}$, we have
 \begin{align}
  \E \|\theta_1 - \theta_1^{(i)}\| & \leq (1-\eta\lambda)\E \|\theta_{0} -\theta_{0}^{(i)}\| + \frac{\eta\sqrt{d}}{(1-\beta_1)(\rho+\varepsilon)} \E \big\|m_1 - m_1^{(i)}\big\| \nonumber \\
  &\quad + \frac{\eta G\sqrt{d}}{(1-\beta_1)(1-\beta_2)(\rho+\varepsilon)^2} \E \big\| v_1 -v_1^{(i)}\big\| \nonumber \\
  & \leq  \frac{\eta\sqrt{d}}{(1-\beta_1)(\rho+\varepsilon)}\frac{2(1-\beta_1)\sigma}{N} +   \frac{\eta G\sqrt{d}}{(1-\beta_1)(1-\beta_2)(\rho+\varepsilon)^2}\frac{2(1-\beta_2)G^2}{N} \nonumber \\
  & = \frac{1}{N}\big(\frac{2\eta\sqrt{d}\sigma}{\rho+\varepsilon}+\frac{2\eta\sqrt{d}G^3}{(1-\beta_1)(\rho+\varepsilon)^2}\big)
  \nonumber \\
  & = \frac{\varphi_1}{N}.
  \end{align}
Let $\eta = O(\frac{1}{\sqrt{d}})$, $\beta_1=O(1)$, $\beta_2=O(1)$,
$\sigma=O(1)$, $G=O(1)$ and $L=O(1)$, we have
\begin{align}
 & \phi_1=2(1-\beta_1)\sigma = O(1), \nonumber \\
 & \psi_1 = 2(1-\beta_2)G^2=O(1),  \nonumber \\
 & \varphi_1 = \frac{2\eta\sqrt{d}\sigma}{\rho+\varepsilon}+\frac{2\eta\sqrt{d}G^3}{(1-\beta_1)(\rho+\varepsilon)^2} = O(\frac{1}{(\rho+\varepsilon)^2}),
\end{align}
where the parameters $\rho=\min_{t\geq1}\min_{j\in [d]}(\hat{v}_t)_j\geq 0$ and $\varepsilon>0$ generally is very small.
Thus, we have
 \begin{align} \label{eq:t1}
  \E \|\theta_1 - \theta_1^{(i)}\|  \leq  \frac{\varphi_1}{N} =O(\frac{1}{(\rho+\varepsilon)^2N}).
  \end{align}

 If $j_2\neq i$ with probability $1-\frac{1}{N}$, since $m_2=\beta_1m_1 + (1-\beta_1)\nabla f(\theta_1;z_{j_2})$ and $m_2^{(i)}=\beta_1 m_1^{(i)} + (1-\beta_1)\nabla f(\theta_1^{(i)};z_{j_2})$, we have
 \begin{align} \label{eq:u1}
  \E\|m_2 - m_2^{(i)}  \|
  & = (1-\frac{1}{N})\E \|\beta_1(m_1-m_1^{(i)} )  + (1-\beta_1)(\nabla f(\theta_1;z_{j_2}) -\nabla f(\theta_1^{(i)};z_{j_2}) )\| \nonumber \\
  & \leq (1-\frac{1}{N})\big(\beta_1 \E\|m_1-m_1^{(i)}\| + (1-\beta_1) \E \|\nabla f(\theta_1;z_{j_2}) -\nabla f(\theta_1^{(i)};z_{j_2})\| \big) \nonumber \\
  & \leq (1-\frac{1}{N})\beta_1\frac{\phi_1}{N} + (1-\frac{1}{N})(1-\beta_1) L \E \|\theta_1-\theta_1^{(i)}\| \nonumber \\
  & \leq (1-\frac{1}{N})\beta_1\frac{\phi_1}{N} + (1-\frac{1}{N})(1-\beta_1) L\frac{\varphi_1}{N},
 \end{align}
 where the second last inequality holds by Assumption~\ref{ass:s1}.
 Since $v_2=\beta_2v_1 + (1-\beta_2)(\nabla f(\theta_1;z_{j_2}))^2$ and $v_2^{(i)}=\beta_2 v_1^{(i)} + (1-\beta_2)(\nabla f(\theta_1^{(i)};z_{j_2}))^2$, we have
 \begin{align}
  \E\|v_2 - v_2^{(i)}  \|  \label{eq:v1}
  & = (1-\frac{1}{N})\E \big\|\beta_2(v_1-v_1^{(i)} )  + (1-\beta_2)((\nabla f(\theta_1;z_{j_2}))^2 -(\nabla f(\theta_1^{(i)};z_{j_2}) )^2\big\| \nonumber \\
  & \leq (1-\frac{1}{N})\big(\beta_2 \E\|v_1-v_1^{(i)}\| + (1-\beta_2) \E \|(\nabla f(\theta_1;z_{j_2}))^2 -\nabla f(\theta_1;z_{j_2})\nabla f(\theta_1^{(i)};z_{j_2}) \nonumber \\
  &\quad +\nabla f(\theta_1;z_{j_2})\nabla f(\theta_1^{(i)};z_{j_2}) -(\nabla f(\theta_1^{(i)};z_{j_2}))^2\| \big) \nonumber \\
  & \mathop{\leq}^{(i)} (1-\frac{1}{N})\beta_2\frac{\psi_1}{N} + (1-\frac{1}{N})(1-\beta_2) 2G \E \|\nabla f(\theta_1;z_{j_2})-\nabla f(\theta_1^{(i)};z_{j_2})\| \nonumber \\
  & \leq (1-\frac{1}{N})\beta_2\frac{\psi_1}{N} + (1-\frac{1}{N})(1-\beta_2) 2G L\E \|\theta_1-\theta_1^{(i)}\| \nonumber \\
  & \leq (1-\frac{1}{N})\beta_2\frac{\psi_1}{N} + (1-\frac{1}{N})(1-\beta_2) 2G L\frac{\varphi_1}{N},
 \end{align}
 where the above inequality~$(i)$ holds by Assumption~\ref{ass:g}, i.e., $\|\nabla f(\theta;z)\|\leq G$ for all
 $\theta\in \R^d, \ z\sim \mathcal{D}$.

If $j_2= i$ with probability $\frac{1}{N}$, we have
\begin{align}  \label{eq:u2}
   \E\|m_2 - m_2^{(i)} \|
  & = \frac{1}{N}\E \big\| \beta_1(m_1-m_1^{(i)} )  + (1-\beta_1)(\nabla f(\theta_1;z_{i}) -\nabla f(\theta_1^{(i)};\tilde{z}_{i}) )\big\| \nonumber \\
  & \leq \frac{1}{N}\beta_1 \E\|m_1-m_1^{(i)}\| + \frac{1}{N}(1-\beta_1) \E \|\nabla f(\theta_1;z_{i}) -\nabla f(\theta_1^{(i)};\tilde{z}_{i})\| \nonumber \\
  & \leq \frac{1}{N}\beta_1\frac{\phi_1}{N} + \frac{1-\beta_1}{N} \E \big\|\nabla f(\theta_1;z_{i}) - \nabla F(\theta_1) + \nabla F(\theta_1) - \nabla F(\theta_1^{(i)}) \nonumber \\
  & \quad + \nabla F(\theta_1^{(i)})-\nabla f(\theta_1^{(i)};\tilde{z}_{i})\big\| \nonumber \\
  & \leq  \frac{1}{N}\beta_1\frac{\phi_1}{N} + \frac{1-\beta_1}{N} \big( \E \|\nabla f(\theta_1;z_{i}) - \nabla F(\theta_1)\| + \E\|\nabla F(\theta_1) - \nabla F(\theta_1^{(i)})\| \nonumber \\
  &\quad + \E\|\nabla F(\theta_1^{(i)})-\nabla f(\theta_1^{(i)};\tilde{z}_{i})\| \big) \nonumber \\
  & \mathop{\leq}^{(i)}  \frac{1}{N}\frac{\beta_1\phi_1}{N}  + \frac{2(1-\beta_1)\sigma}{N} + \frac{(1-\beta_1)L}{N} \E \|\theta_1-\theta_1^{(i)}\| \nonumber \\
  & \leq \frac{1}{N}\frac{\beta_1\phi_1}{N}  + \frac{2(1-\beta_1)\sigma}{N} + \frac{(1-\beta_1)L}{N} \frac{\varphi_1}{N},
 \end{align}
 where the above inequality $(i)$ holds by Assumption~\ref{ass:v}.
We also have
 \begin{align}  \label{eq:v2}
   \E\|v_2 - v_2^{(i)}  \|
  & = \frac{1}{N}\E \big\|\beta_2(v_1-v_1^{(i)} )  + (1-\beta_2)((\nabla f(\theta_1;z_{i}))^2 -(\nabla f(\theta_1^{(i)};\tilde{z}_{i}) )^2\big\| \nonumber \\
  & \leq \frac{1}{N}\big(\beta_2 \E\|v_1-v_1^{(i)}\| + (1-\beta_2) \E \|(\nabla f(\theta_1;z_{i}))^2 -\nabla f(\theta_1;z_{i})\nabla f(\theta_1^{(i)};\tilde{z}_{i}) \nonumber \\
  &\qquad +\nabla f(\theta_1;z_{i})\nabla f(\theta_1^{(i)};\tilde{z}_{i}) -(\nabla f(\theta_1^{(i)};\tilde{z}_{i}))^2\| \big) \nonumber \\
  & \leq \frac{1}{N}\beta_2\frac{\psi_1}{N} + \frac{1}{N}(1-\beta_2) 2G \E \|\nabla f(\theta_1;z_{i})-\nabla f(\theta_1^{(i)};\tilde{z}_{i})\| \nonumber \\
  & \leq  \frac{1}{N}\beta_2\frac{\psi_1}{N} + \frac{1}{N}(1-\beta_2) 2G \E \|\nabla f(\theta_1;z_{i})-\nabla F(\theta_1) + \nabla F(\theta_1) - \nabla F(\theta_1^{(i)}) \nonumber \\
  & \qquad + \nabla F(\theta_1^{(i)})-\nabla f(\theta_1^{(i)};\tilde{z}_{i})\| \nonumber \\
  & \leq \frac{1}{N}\beta_2\frac{\psi_1}{N} + \frac{1}{N}(1-\beta_2) 2G \big(2\sigma+ L \E\|\theta_1-\theta_1^{(i)}\| \big) \nonumber \\
  & \leq  \frac{1}{N}\beta_2\frac{\psi_1}{N} + \frac{1}{N}(1-\beta_2) 4G\sigma + \frac{1}{N}(1-\beta_2) 2GL \frac{\varphi_1}{N},
 \end{align}
where the last inequality holds by (\ref{eq:t1}).

 Let $\phi_2= \beta_1\phi_1 + 2(1-\beta_1)\sigma + (1-\beta_1)L\varphi_1$, by
 using the above inequalities~(\ref{eq:u1}) and~(\ref{eq:u2}), we have
 \begin{align}
  \E\|m_2 - m_2^{(i)}  \|
  & \leq (1-\frac{1}{N})\frac{\beta_1\phi_1}{N} + (1-\frac{1}{N})(1-\beta_1) L\frac{\varphi_1}{N} \nonumber \\
  &\quad + \frac{1}{N}\frac{\beta_1\phi_1}{N}  + \frac{2(1-\beta_1)\sigma}{N} + \frac{(1-\beta_1)L}{N}\frac{\varphi_1}{N} \nonumber \\
  & = \frac{\beta_1\phi_1}{N} + \frac{2(1-\beta_1)\sigma}{N} + \frac{(1-\beta_1)L\varphi_1}{N} \nonumber \\
  & = \frac{\phi_2}{N}.
 \end{align}

Let $\psi_2 = \beta_2\psi_1 + 4(1-\beta_2)G\sigma + 2(1-\beta_2)GL\varphi_1$, by
 using the above inequalities~(\ref{eq:v1}) and~(\ref{eq:v2}), we have
 \begin{align}
 \E\|v_2 - v_2^{(i)}  \|
  & \leq (1-\frac{1}{N})\beta_2\frac{\psi_1}{N} + (1-\frac{1}{N})(1-\beta_2) 2G L\frac{\varphi_1}{N} \nonumber \\
  & \quad +\frac{1}{N}\beta_2\frac{\psi_1}{N} + \frac{1}{N}(1-\beta_2) 4G\sigma + \frac{1}{N}(1-\beta_2) 2GL \frac{\varphi_1}{N} \nonumber \\
  & = \frac{\beta_2 \psi_1}{N} +\frac{ 4(1-\beta_2)G\sigma}{N} + \frac{2(1-\beta_2)GL\varphi_1}{N} \nonumber \\
  & = \frac{\psi_2}{N}.
 \end{align}
 According to the above inequality~(\ref{eq:re}), then we can obtain
 \begin{align}
  \E \|\theta_2 - \theta_2^{(i)}\| & \leq (1-\eta\lambda)\E \|\theta_{1} -\theta_{1}^{(i)}\| + \frac{\eta\sqrt{d}}{(1-\beta_1^2)(\rho+\varepsilon)} \E \big\|m_2 - m_2^{(i)}\big\| \nonumber \\
  & \quad + \frac{\eta G\sqrt{d}}{(1-\beta_1^2)(1-\beta_2^2)(\rho+\varepsilon)^2} \E \big\| v_2 -v_2^{(i)}\big\| \nonumber \\
  & \leq  (1-\eta\lambda)\frac{\varphi_1}{N} + \frac{\eta\sqrt{d}}{(1-\beta_1^2)(\rho+\varepsilon)} \frac{\phi_2}{N} + \frac{\eta G\sqrt{d}}{(1-\beta_1^2)(1-\beta_2^2)(\rho+\varepsilon)^2}\frac{\psi_2}{N} \nonumber \\
  & = \frac{\varphi_2}{N},
  \end{align}
where $\varphi_2 = (1-\eta\lambda)\varphi_1 +  \frac{\eta\sqrt{d}\phi_2}{(\rho+\varepsilon)(1-\beta_1^2)} + \frac{\eta G\sqrt{d}\psi_2}{(1-\beta_1^2)(1-\beta_2^2)(\rho+\varepsilon)^2}$.

Let $\eta = O(\frac{1}{\sqrt{d}})$, $\lambda\in [0,\frac{1}{\eta})$, $G=O(1)$, $L=O(1)$, $\beta_1=O(1)$ and $\beta_2=O(1)$ with $\beta_1, \beta_2 \in (0,1)$, since $\varphi_1=O(\frac{1}{(\rho+\varepsilon)^2})$,
$\phi_1=O(1)$ and $\psi_1=O(1)$,we have
\begin{align}
 & \phi_2= \beta_1\phi_1 + 2(1-\beta_1)\sigma + (1-\beta_1)L\varphi_1 = O(\frac{1}{(\rho+\varepsilon)^2}), \nonumber \\
 & \psi_2 = \beta_2\psi_1 + 4(1-\beta_2)G\sigma + 2(1-\beta_2)GL\varphi_1 = O(\frac{1}{(\rho+\varepsilon)^2}), \nonumber \\
 & \varphi_2 = (1-\eta\lambda)\varphi_1 +  \frac{\eta\sqrt{d}\phi_2}{(1-\beta_1^2)(\rho+\varepsilon)} + \frac{\eta G\sqrt{d}\psi_2}{(1-\beta_1^2)(1-\beta_2^2)(\rho+\varepsilon)^2}  =
 O(\frac{1}{(\rho+\varepsilon)^4}).
\end{align}
Thus, we have
\begin{align}
  \E \|\theta_2 - \theta_2^{(i)}\| \leq \frac{\varphi_2}{N} = O(\frac{1}{(\rho+\varepsilon)^4N}).
  \end{align}

Based on mathematical induction, we assume $\E \|\theta_t - \theta_t^{(i)}\| \leq \frac{\varphi_t}{N}$ with
$\varphi_t=O(\frac{1}{(\rho+\varepsilon)^{2t}})$, and $\E\|m_t-m_t^{(i)}\|\leq \frac{\phi_t}{N}$ with $\phi_t=O(\frac{1}{(\rho+\varepsilon)^{2(t-1)}})$, and
$\E\|v_t-v_t^{(i)}\|\leq \frac{\psi_t}{N}$ with $\psi_t=O(\frac{1}{(\rho+\varepsilon)^{2(t-1)}})$.

If $j_{t+1}\neq i$ with probability $1-\frac{1}{N}$, since $m_{t+1}=\beta_1m_t + (1-\beta_1)\nabla f(\theta_t;z_{j_{t+1}})$ and $m_{t+1}^{(i)}=\beta_1 m_t^{(i)} + (1-\beta_1)\nabla f(\theta_{t}^{(i)};z_{j_{t+1}})$, we have
 \begin{align} \label{eq:u3}
  \E\|m_{t+1} - m_{t+1}^{(i)}  \|
  & = (1-\frac{1}{N})\E \|\beta_1(m_{t}-m_{t}^{(i)} )  + (1-\beta_1)(\nabla f(\theta_t;z_{j_{t+1}}) -\nabla f(\theta_t^{(i)};z_{j_{t+1}}) )\| \nonumber \\
  & \leq (1-\frac{1}{N})\big(\beta_1 \E\|m_t-m_t^{(i)}\| + (1-\beta_1) \E \|\nabla f(\theta_t;z_{j_{t+1}}) -\nabla f(\theta_t^{(i)};z_{j_{t+1}})\| \big) \nonumber \\
  & \leq (1-\frac{1}{N})\beta_1\frac{\phi_t}{N} + (1-\frac{1}{N})(1-\beta_1) L \E \|\theta_t-\theta_t^{(i)}\| \nonumber \\
  & \leq (1-\frac{1}{N})\beta_1\frac{\phi_t}{N} + (1-\frac{1}{N})(1-\beta_1) L\frac{\varphi_t}{N},
 \end{align}
 where the second last inequality holds by Assumption~\ref{ass:s1}.
 Since $v_{t+1}=\beta_2v_t + (1-\beta_2)(\nabla f(\theta_t;z_{j_{t+1}}))^2$ and $v_{t+1}^{(i)}=\beta_2 v_t^{(i)} + (1-\beta_2)(\nabla f(\theta_t^{(i)};z_{j_{t+1}}))^2$, we have
 \begin{align} \label{eq:v3}
   \E\|v_{t+1} - v_{t+1}^{(i)}  \|
  & = (1-\frac{1}{N})\E \big\|\beta_2(v_t-v_t^{(i)} )  + (1-\beta_2)((\nabla f(\theta_t;z_{j_{t+1}}))^2 -(\nabla f(\theta_t^{(i)};z_{j_{t+1}}) )^2\big\| \nonumber \\
  & \leq (1-\frac{1}{N})\big(\beta_2 \E\|v_t-v_t^{(i)}\| + (1-\beta_2) \E \|(\nabla f(\theta_t;z_{j_{t+1}}))^2 -\nabla f(\theta_t;z_{j_{t+1}})\nabla f(\theta_t^{(i)};z_{j_{t+1}}) \nonumber \\
  &\quad +\nabla f(\theta_t;z_{j_{t+1}})\nabla f(\theta_t^{(i)};z_{j_{t+1}}) -(\nabla f(\theta_t^{(i)};z_{j_{t+1}}))^2\| \big) \nonumber \\
  & \leq (1-\frac{1}{N})\beta_2\frac{\psi_t}{N} + (1-\frac{1}{N})(1-\beta_2) 2G \E \|\nabla f(\theta_t;z_{j_{t+1}})-\nabla f(\theta_t^{(i)};z_{j_{t+1}})\| \nonumber \\
  & \leq (1-\frac{1}{N})\beta_2\frac{\psi_t}{N} + (1-\frac{1}{N})(1-\beta_2) 2G L\E \|\theta_t-\theta_t^{(i)}\| \nonumber \\
  & \leq (1-\frac{1}{N})\beta_2\frac{\psi_t}{N} + (1-\frac{1}{N})(1-\beta_2) 2G L\frac{\varphi_t}{N},
 \end{align}
where the second inequality holds by Assumption~\ref{ass:g}.

If $j_{t+1}= i$ with probability $\frac{1}{N}$, we have
\begin{align} \label{eq:u4}
   \E\|m_{t+1} - m_{t+1}^{(i)} \|_\F
  & = \frac{1}{N}\E \big\| \beta_1(m_t-m_t^{(i)} )  + (1-\beta_1)(\nabla f(\theta_t;z_{i}) -\nabla f(\theta_t^{(i)};\tilde{z}_{i}) )\big\| \nonumber \\
  & \leq \frac{1}{N}\beta_1 \E\|m_t-m_t^{(i)}\| + \frac{1}{N}(1-\beta_1) \E \|\nabla f(\theta_t;z_{i}) -\nabla f(\theta_t^{(i)};\tilde{z}_{i})\| \nonumber \\
  & \leq \frac{1}{N}\beta_1\frac{\phi_t}{N} + \frac{1-\beta_1}{N} \E \big\|\nabla f(\theta_t;z_{i}) - \nabla F(\theta_t) + \nabla F(\theta_t) - \nabla F(\theta_t^{(i)}) \nonumber \\
  & \quad + \nabla F(\theta_t^{(i)})-\nabla f(\theta_t^{(i)};\tilde{z}_{i})\big\| \nonumber \\
  & \leq  \frac{1}{N}\beta_1\frac{\phi_t}{N} + \frac{1-\beta_1}{N} \big( \E \|\nabla f(\theta_t;z_{i}) - \nabla F(\theta_t)\| + \E\|\nabla F(\theta_t) - \nabla F(\theta_t^{(i)})\| \nonumber \\
  &\quad + \E\|\nabla F(\theta_t^{(i)})-\nabla f(\theta_t^{(i)};\tilde{z}_{i})\| \big) \nonumber \\
  & \leq  \frac{1}{N}\beta_1\frac{\phi_t}{N}  + \frac{2(1-\beta_1)\sigma}{N} + \frac{(1-\beta_1)L}{N} \E \|\theta_t-\theta_t^{(i)}\| \nonumber \\
  & \leq \frac{1}{N}\beta_1\frac{\phi_t}{N}  + \frac{2(1-\beta_1)\sigma}{N} + \frac{(1-\beta_1)L}{N} \frac{\varphi_t}{N}.
 \end{align}
We also have
 \begin{align} \label{eq:v4}
   \E\|v_{t+1} - v_{t+1}^{(i)}  \|
  & = \frac{1}{N}\E \big\|\beta_2(v_t-v_t^{(i)} )  + (1-\beta_2)((\nabla f(\theta_t;z_{i}))^2 -(\nabla f(\theta_t^{(i)};\tilde{z}_{i}) )^2\big\| \nonumber \\
  & \leq \frac{1}{N}\big(\beta_2 \E\|v_t-v_t^{(i)}\| + (1-\beta_2) \E \|(\nabla f(\theta_t;z_{i}))^2 -\nabla f(\theta_t;z_{i})\nabla f(\theta_t^{(i)};\tilde{z}_{i}) \nonumber \\
  &\qquad +\nabla f(\theta_t;z_{i})\nabla f(\theta_t^{(i)};\tilde{z}_{i}) -(\nabla f(\theta_t^{(i)};\tilde{z}_{i}))^2\| \big) \nonumber \\
  & \leq \frac{1}{N}\beta_2\frac{\psi_t}{N} + \frac{1}{N}(1-\beta_2) 2G \E \|\nabla f(\theta_t;z_{i})-\nabla f(\theta_t^{(i)};\tilde{z}_{i})\| \nonumber \\
  & \leq  \frac{1}{N}\beta_2\frac{\psi_t}{N} + \frac{1}{N}(1-\beta_2) 2G \E \|\nabla f(\theta_t;z_{i})-\nabla F(\theta_t) + \nabla F(\theta_t) - \nabla F(\theta_t^{(i)}) \nonumber \\
  & \qquad + \nabla F(\theta_t^{(i)})-\nabla f(\theta_t^{(i)};\tilde{z}_{i})\| \nonumber \\
  & \leq \frac{1}{N}\beta_2\frac{\psi_t}{N} + \frac{1}{N}(1-\beta_2) 2G \big(2\sigma+ L \E\|\theta_t-\theta_t^{(i)}\| \big) \nonumber \\
  & \leq  \frac{1}{N}\beta_2\frac{\psi_t}{N} + \frac{1}{N}(1-\beta_2) 4G\sigma + \frac{1}{N}(1-\beta_2) 2GL \frac{\varphi_t}{N}.
 \end{align}

 Let $\phi_{t+1}= \beta_1\phi_t + 2(1-\beta_1)\sigma + (1-\beta_1)L\varphi_t$, by
 using the above inequalities~(\ref{eq:u3}) and~(\ref{eq:u4}), we have
 \begin{align} \label{eq:m}
  \E\|m_{t+1} - m_{t+1}^{(i)}  \|
  & \leq (1-\frac{1}{N})\frac{\beta_1\phi_t}{N} + (1-\frac{1}{N})(1-\beta_1) L\frac{\varphi_t}{N} \nonumber \\
  &\quad + \frac{1}{N}\frac{\beta_1\phi_t}{N}  + \frac{2(1-\beta_1)\sigma}{N} + \frac{(1-\beta_1)L}{N}\frac{\varphi_t}{N} \nonumber \\
  & = \frac{\beta_1\phi_t}{N}  + \frac{2(1-\beta_1)\sigma}{N} + \frac{(1-\beta_1)L\varphi_t}{N} \nonumber \\
  & = \frac{\phi_{t+1}}{N}.
 \end{align}
Let $\psi_{t+1} = \beta_2 \psi_t + 4(1-\beta_2)G\sigma + 2(1-\beta_2)GL\varphi_t$, by
 using the above inequalities~(\ref{eq:v3}) and~(\ref{eq:v4}), we have
 \begin{align} \label{eq:v}
 \E\|v_{t+1} - v_{t+1}^{(i)}  \|
  & \leq (1-\frac{1}{N})\beta_2\frac{\psi_t}{N} + (1-\frac{1}{N})(1-\beta_2) 2G L\frac{\varphi_t}{N} \nonumber \\
  &\quad +\frac{1}{N}\beta_2\frac{\psi_t}{N} + \frac{1}{N}(1-\beta_2) 4G\sigma + \frac{1}{N}(1-\beta_2) 2GL \frac{\varphi_t}{N} \nonumber \\
  & = \frac{\beta_2 \psi_t}{N} +\frac{ 4(1-\beta_2)G\sigma}{N} + \frac{2(1-\beta_2)GL\varphi_t}{N} \nonumber \\
  & = \frac{\psi_{t+1}}{N}.
 \end{align}
 According to the above inequality~(\ref{eq:re}), then we can obtain
 \begin{align}
  \E \|\theta_{t+1} - \theta_{t+1}^{(i)}\| & \leq (1-\eta\lambda)\E \|\theta_{t} -\theta_{t}^{(i)}\| + \frac{\eta\sqrt{d}}{(1-\beta_1^{t+1})(\rho+\varepsilon)} \E \big\|m_{t+1} - m_{t+1}^{(i)}\big\| \nonumber \\
  & \quad + \frac{\eta G\sqrt{d}}{(1-\beta_1^{t+1})(1-\beta_2^{t+1})(\rho+\varepsilon)^2} \E \big\| v_{t+1} -v_{t+1}^{(i)}\big\| \nonumber \\
  & \leq  (1-\eta\lambda)\frac{\varphi_t}{N} + \frac{\eta\sqrt{d}}{(1-\beta_1^{t+1})(\rho+\varepsilon)} \frac{\phi_{t+1}}{N} + \frac{\eta G\sqrt{d}}{(1-\beta_1^{t+1})(1-\beta_2^{t+1})(\rho+\varepsilon)^2}\frac{\psi_{t+1}}{N} \nonumber \\
  & = \frac{\varphi_{t+1}}{N},
  \end{align}
where $\varphi_{t+1} = (1-\eta\lambda)\varphi_t + \frac{\eta\sqrt{d}\phi_{t+1}}{(1-\beta_1^{t+1})(\rho+\varepsilon)}+  \frac{\eta G\sqrt{d}\psi_{t+1}}{(1-\beta_1^{t+1})(1-\beta_2^{t+1})(\rho+\varepsilon)^2}$.

Let $\eta = O(\frac{1}{\sqrt{d}})$, $\lambda\in [0,\frac{1}{\eta})$, $G=O(1)$, $L=O(1)$, $\beta_1=O(1)$ and $\beta_2=O(1)$ with $\beta_1, \beta_2 \in (0,1)$, since $\varphi_t=O(\frac{1}{(\rho+\varepsilon)^{2(t-1)}})$, $\phi_t=O(\frac{1}{(\rho+\varepsilon)^{2(t-1)}})$ and $\psi_t=O(\frac{1}{(\rho+\varepsilon)^{2t}})$, we have
\begin{align}
 & \phi_{t+1}= \beta_1\phi_t + 2(1-\beta_1)\sigma + (1-\beta_1)L\varphi_t = O(\frac{1}{(\rho+\varepsilon)^{2t}}), \nonumber \\
 & \psi_{t+1} = \beta_2 \psi_t + 4(1-\beta_2)G\sigma + 2(1-\beta_2)GL\varphi_t = O(\frac{1}{(\rho+\varepsilon)^{2t}}), \nonumber \\
 & \varphi_{t+1} = (1-\eta\lambda)\varphi_t + \frac{\eta\sqrt{d}\phi_{t+1}}{(1-\beta_1^{t+1})(\rho+\varepsilon)}+  \frac{\eta G\sqrt{d}\psi_{t+1}}{(1-\beta_1^{t+1})(1-\beta_2^{t+1})(\rho+\varepsilon)^2} = O(\frac{1}{(\rho+\varepsilon)^{2(t+1)}}).
\end{align}
Thus, we have
\begin{align}
  \E \|\theta_{t+1} - \theta_{t+1}^{(i)}\| \leq \frac{\varphi_{t+1}}{N} = O(\frac{1}{(\rho+\varepsilon)^{2(t+1)}N}).
\end{align}

By using mathematical induction, we have
\begin{align}
  \E \|\theta_{T} - \theta_{T}^{(i)}\| = O\Big(\frac{1}{(\rho+\varepsilon)^{2T}N}\Big).
\end{align}

By using Assumption~\ref{ass:g}, i.e., the condition of $G$-Lipschitz $f(\theta;z)$
for any $z\in \mathcal{D}$, then we have
\begin{align} \label{eq:f1}
 \E |f(\theta_T;z)-f(\theta_T^{(i)};z)| \leq G \E \|\theta_{T} - \theta_{T}^{(i)}\| =O\Big(\frac{1}{(\rho+\varepsilon)^{2T}N}\Big).
\end{align}
By taking expectations over $S$, $S^{(i)}$ and the algorithm's randomness on
the above inequality~(\ref{eq:f1}), and according to the above lemma~\ref{lem:gs}, we
can obtain
\begin{align}
 |\E [ F(\theta_T) - F_S(\theta_T)]| \leq O\Big(\frac{1}{(\rho+\varepsilon)^{2T}N}\Big).
\end{align}

\end{proof}

\section{Generalization Analysis of our HomeAdam(W) Algorithms}
\label{ap:ga2}
\begin{theorem} (Restatement of Theorem~\ref{th:2})
Assume the sequence $\{\theta_t\}_{t=1}^T$ is generated
from Algorithm~\ref{alg:2} on dataset $S=\{z_1,z_2,\cdots,z_N\}$. Under the Assumptions~\ref{ass:s1},~\ref{ass:g},~\ref{ass:v}, without loss of generality, let $\tau\geq 1$, $\lambda\in [0,\frac{1}{\eta})$, $\beta_1=O(1)$ with $\beta_1\in (0,1)$, $\beta_2=O(1)$ with $\beta_2\in (0,1)$, $\sigma=O(1)$, $G=O(1)$ and $L=O(1)$.
 If the iteration number is small (i.e., $T=O(1)$) set
 $\eta=\frac{1}{\sqrt{d}}$, otherwise set $\eta=\frac{1}{\sqrt{d}T}$, we have
\begin{align}
|\E [ F(\theta_T) - F_S(\theta_T)]| \leq O(\frac{1}{N}).
\end{align}
\end{theorem}

\begin{proof}
Implementing Algorithm~\ref{alg:2} on datasets $S$ and $S^{(i)}$ with
 the same random index sequence $\{j_t\}_{t=1}^T$, and let $\{\theta_t\}_{t=1}^T$ and $\{\theta_t^{(i)}\}_{t=1}^T$ be generated from Algorithm~\ref{alg:2} with $S$ and $S^{(i)}$.

 Without loss of generality, let $\tau \geq 1$ in Algorithm~\ref{alg:2}. When $\min_{1\leq j\leq d} (v_t)_j \geq \tau\geq 1$, we have $\theta_{t} = \theta_{t-1} - \eta (\frac{\hat{m}_t}{\hat{v}_t+\varepsilon}+\lambda\theta_{t-1})$ and $\theta_{t}^{(i)} = \theta_{t-1}^{(i)} - \eta (\frac{\hat{m}_t^{(i)}}{\hat{v}_t^{(i)}+\varepsilon}+\lambda\theta_{t-1}^{(i)} )$. Then we have
\begin{align}
 \theta_{t} - \theta_{t}^{(i)}  = (1-\eta\lambda)(\theta_{t-1} -\theta_{t-1}^{(i)}) - \eta \big(\frac{\hat{m}_t}{\hat{v}_t+\varepsilon} -
 \frac{\hat{m}_t^{(i)}}{\hat{v}_t^{(i)}+\varepsilon}\big).
\end{align}
Let $\rho_t=\min_{1\leq j\leq d} (v_t)_j$, we have
 \begin{align} \label{eq:re1}
 \|\theta_{t} - \theta_{t}^{(i)}\| & =\|(1-\eta\lambda)(\theta_{t-1} -\theta_{t-1}^{(i)}) - \eta \big(\frac{\hat{m}_t}{\hat{v}_t+\varepsilon} -
 \frac{\hat{m}_t^{(i)}}{\hat{v}_t^{(i)}+\varepsilon}\big)\| \nonumber \\
 & \leq (1-\eta\lambda)\|\theta_{t-1} -\theta_{t-1}^{(i)}\| + \eta \|\frac{\hat{m}_t}{\hat{v}_t+\varepsilon} -
 \frac{\hat{m}_t^{(i)}}{\hat{v}_t^{(i)}+\varepsilon}\| \nonumber \\
 & \mathop{\leq}^{(i)} (1-\eta\lambda)\|\theta_{t-1} -\theta_{t-1}^{(i)}\| +\frac{\eta\sqrt{d}}{(1-\beta_1^t)(\rho_t+\varepsilon)}\big\|m_t - m_t^{(i)}\big\| \nonumber \\
 &\quad + \frac{\eta G\sqrt{d}}{(1-\beta_1^t)(1-\beta_2^t)(\rho_t +\varepsilon)^2}\big\| v_t -v_t^{(i)}\big\| \nonumber \\
  & \leq (1-\eta\lambda)\|\theta_{t-1} -\theta_{t-1}^{(i)}\| +\frac{\eta\sqrt{d}}{(1-\beta_1^t)(\tau +\varepsilon)}\big\|m_t - m_t^{(i)}\big\| \nonumber \\
 &\quad + \frac{\eta G\sqrt{d}}{(1-\beta_1^t)(1-\beta_2^t)(\tau+\varepsilon)^2}\big\| v_t -v_t^{(i)}\big\|
 \nonumber \\
 & \leq (1-\eta\lambda)\|\theta_{t-1} -\theta_{t-1}^{(i)}\| +\frac{\eta\sqrt{d}}{(1-\beta_1^t)}\big\|m_t - m_t^{(i)}\big\|  + \frac{\eta G\sqrt{d}}{(1-\beta_1^t)(1-\beta_2^t)}\big\| v_t -v_t^{(i)}\big\|,
 \end{align}
 where the above inequality $(i)$ is due to Lemma~\ref{lem:1}, and
 the last inequality holds by $\tau+\varepsilon\geq 1$.

 When $\min_{1\leq j\leq d} (v_t)_j < \tau$, we have $\theta_{t} = \theta_{t-1} - \eta(\hat{m}_t+\lambda\theta_{t-1})$ and $\theta_{t}^{(i)} = \theta_{t-1}^{(i)} - \eta (\hat{m}_t^{(i)}+\lambda\theta_{t-1}^{(i)})$. Then we have
 \begin{align} \label{eq:re2}
 \|\theta_{t} - \theta_{t}^{(i)}\| & =\|(1-\eta\lambda)(\theta_{t-1} -\theta_{t-1}^{(i)}) - \eta \big(\hat{m}_t-\hat{m}_t^{(i)}\big)\| \nonumber \\
 & \leq (1-\eta\lambda)\|\theta_{t-1} -\theta_{t-1}^{(i)}\| + \eta \|\hat{m}_t-\hat{m}_t^{(i)}\| \nonumber \\
 & \leq (1-\eta\lambda)\|\theta_{t-1} -\theta_{t-1}^{(i)}\| +\frac{\eta\sqrt{d}}{(1-\beta_1^t)}\big\|m_t - m_t^{(i)}\big\|  + \frac{\eta G\sqrt{d}}{(1-\beta_1^t)(1-\beta_2^t)}\big\| v_t -v_t^{(i)}\big\|.
 \end{align}
Thus, we have
\begin{align} \label{eq:re3}
 \|\theta_{t} - \theta_{t}^{(i)}\| \leq (1-\eta\lambda)\|\theta_{t-1} -\theta_{t-1}^{(i)}\| +\frac{\eta\sqrt{d}}{(1-\beta_1^t)}\big\|m_t - m_t^{(i)}\big\|  + \frac{\eta G\sqrt{d}}{(1-\beta_1^t)(1-\beta_2^t)}\big\| v_t -v_t^{(i)}\big\|.
 \end{align}

According to the above inequality~(\ref{eq:re3}), following the above proof of Theorem~\ref{th:1}, let $\eta = O(\frac{1}{\sqrt{d}})$, $\beta_1=O(1)$, $\beta_2=O(1)$,
$\sigma=O(1)$, $G=O(1)$ and $L=O(1)$, we have
\begin{align}
 & \phi_1=2(1-\beta_1)\sigma = O(1), \quad \psi_1 = 2(1-\beta_2)G^2=O(1),  \quad \varphi_1 = 2\eta\sqrt{d}\sigma+\frac{2\eta\sqrt{d}G^3}{(1-\beta_1)} = O(1), \nonumber \\
 &  \E \|m_1 - m_1^{(i)}\| \leq \frac{\phi_1}{N}, \quad   \E \|v_1 - v_1^{(i)}\| \leq \frac{\psi_1}{N}, \quad
  \E \|\theta_1 - \theta_1^{(i)}\|  \leq  \frac{\varphi_1}{N} =O(\frac{1}{N}).
  \end{align}

\textbf{If} the iteration number is small (i.e., $T=O(1)$),
let $\eta=\frac{1}{\sqrt{d}}$, $\lambda\in [0,\frac{1}{\eta})$, $\beta_1=O(1)$ with $\beta_1\in (0,1)$, $\beta_2=O(1)$ with $\beta_2\in (0,1)$, $\sigma=O(1)$, $G=O(1)$ and $L=O(1)$.
Following the above proof of Theorem~\ref{th:1}, assume $\E \|m_t - m_t^{(i)}\| \leq \frac{\phi_t}{N}$ with
$\phi_t=O(1)$, and $\E\|v_t-v_t^{(i)}\|\leq \frac{\psi_t}{N}$ with $\psi_t=O(1)$, we can obtain $\phi_{t+1} =O(1)$ and $\psi_{t+1}=O(1)$.
By using the above inequality~(\ref{eq:re3}), then we have
\begin{align}
 \E \|\theta_{t+1} - \theta_{t+1}^{(i)}\| \leq \frac{\varphi_{t+1}}{N} =O(\frac{1}{N}),
\end{align}
where $\varphi_{t+1}=(1-\eta\lambda)\varphi_t + \frac{\eta\sqrt{d}\phi_{t+1}}{(1-\beta_1^{t+1})}+  \frac{\eta G\sqrt{d}\psi_{t+1}}{(1-\beta_1^{t+1})(1-\beta_2^{t+1})}=O(1)$.
By using the mathematical induction, we have
\begin{align}
 \E \|\theta_{T} - \theta_{T}^{(i)}\| &  \leq  \frac{\varphi_T}{N} = O(\frac{1}{N}).
\end{align}

\textbf{If} the iteration number $T$ is large, we consider the iteration number $t\geq1$
in the generalization analysis. By using the mathematical induction, due to recursion of the above inequality~(\ref{eq:re3}), following the above proof of Theorem~\ref{th:1}, we assume $\E \|m_t - m_t^{(i)}\| \leq \frac{\phi_t}{N}$ with
$\phi_t=O(t)$, and $\E\|v_t-v_t^{(i)}\|\leq \frac{\psi_t}{N}$ with $\psi_t=O(t)$.

Let $\eta=O(\frac{1}{\sqrt{d}})$, $\lambda\in [0,\frac{1}{\eta})$, $\beta_1=O(1)$ with $\beta_1\in (0,1)$, $\beta_2=O(1)$ with $\beta_2\in (0,1)$, $\sigma=O(1)$, $G=O(1)$ and
$L=O(1)$, we have
\begin{align}
 & \phi_{t+1}= \beta_1\phi_t + 2(1-\beta_1)\sigma + (1-\beta_1)L\varphi_t=O(t+1), \nonumber \\
 & \psi_{t+1}= \beta_2 \psi_t + 4(1-\beta_2)G\sigma + 2(1-\beta_2)GL\varphi_t=O(t+1). \nonumber
\end{align}
By using the above inequality~(\ref{eq:re3}), then we can obtain
\begin{align}
 \E \|\theta_{t+1} - \theta_{t+1}^{(i)}\| \leq \frac{\varphi_{t+1}}{N} =O(\frac{t+1}{N}).
\end{align}
By using the mathematical induction, we have
\begin{align}
 \E \|\theta_{T} - \theta_{T}^{(i)}\| &  \leq  \frac{\varphi_T}{N} = O(\frac{T}{N}).
\end{align}

Further let $\eta=O( \textcolor{blue}{\frac{1}{T\sqrt{d}}})$, $\lambda\in [0,\frac{1}{\eta})$, $\beta_1=O(1)$ with $\beta_1\in (0,1)$, $\beta_2=O(1)$ with $\beta_2\in (0,1)$, $\sigma=O(1)$, $G=O(1)$ and
$L=O(1)$, we have $\phi_T=O(1)$, $\psi_T=O(1)$, $\phi_T=O(1)$,
\begin{align}
 \E \|\theta_{T} - \theta_{T}^{(i)}\| &  \leq  \frac{\varphi_T}{N} = O(\frac{1}{N}).
\end{align}

By using Assumption~\ref{ass:g}, i.e., the condition of $G$-Lipschitz $f(\theta;z)$ for any $z\in \mathcal{D}$, then we have
\begin{align} \label{eq:f2}
 \E |f(\theta_T;z)-f(\theta_T^{(i)};z)| \leq G \E \|\theta_{T} - \theta_{T}^{(i)}\| =O(\frac{1}{N}).
\end{align}
By taking expectations over $S$, $S^{(i)}$ and the algorithm's randomness on
the above inequality~(\ref{eq:f2}), and according to the above lemma~\ref{lem:gs}, we
can obtain
\begin{align}
 |\E [ F(\theta_T) - F_S(\theta_T)]| \leq O(\frac{1}{N}).
\end{align}

\end{proof}

\section{Convergence Analysis of Adam(W)-srf Algorithms}
\label{ap:ca1}
In this section, we provide a detailed convergence analysis of our Adam(W)-srf algorithms.

\begin{lemma} \label{lem:A2} (Restatement of Lemma~\ref{lem:2})
Assume the sequence $\{m_t\}_{t=0}^T$ is generated from Algorithm~\ref{alg:1}, let $\beta_1 = 1-c\eta\in (0,1)$, $\beta_2\in (0,1)$, we have
\begin{align}
 \E\|\nabla F(\theta_{t}) - m_{t+1}\|^2   \leq (1-c\eta)\mathbb{E} \|\nabla F(\theta_{t-1}) - m_{t} \|^2 + \frac{2}{c \eta}L^2\mathbb{E}\|\theta_t-\theta_{t-1}\|^2  + c^2\eta^2\sigma^2,
\end{align}
where $c>0$.
\end{lemma}

\begin{proof}
 At the line 6 in Algorithm \ref{alg:1}, it has $m_{t}= \beta_{1}m_{t-1}+ (1-\beta_{1})g_t$ for $t\geq1$.
 Since $m_{t+1}= \beta_{1}m_{t}+ (1-\beta_{1})g_{t+1}=\beta_{1}m_{t}+ (1-\beta_{1})\nabla f(\theta_{t};z_{t+1})$, then we have
 \begin{align}
  \E\|\nabla F(\theta_{t}) - m_{t+1}\|^2
 & = \mathbb{E}\|\nabla F(\theta_{t-1}) - m_{t} + \nabla F(\theta_t)-\nabla F(\theta_{t-1}) -(m_{t+1}-m_t)\|^2 \nonumber \\
  & =  \mathbb{E}\|\nabla F(\theta_{t-1}) - m_{t} + \nabla F(\theta_t)-\nabla F(\theta_{t-1})  + (1-\beta_1)m_{t} - (1-\beta_1)\nabla f(\theta_{t};z_{t+1})\|^2 \nonumber \\
  & = \mathbb{E}\|\beta_1(\nabla F(\theta_{t-1}) - m_{t}) + \beta_1\big( \nabla F(\theta_t)-\nabla F(\theta_{t-1})\big) + (1-\beta_1)(\nabla F(\theta_t)- \nabla f(\theta_{t};z_{t+1})) \|^2 \nonumber \\
  & \mathop{=}^{(i)} \beta_1^2\mathbb{E} \|\nabla F(\theta_{t-1}) - m_{t} +\nabla F(\theta_t)-\nabla F(\theta_{t-1})\|^2  + (1-\beta_1)^2\mathbb{E} \|\nabla F(\theta_t)- \nabla f(\theta_{t};z_{t+1}) \|^2 \nonumber \\
  & \mathop{\leq}^{(ii)} \beta_1^2(2-\beta_1)\mathbb{E} \|\nabla F(\theta_{t-1}) - m_{t} \|^2 + \beta_1^2(1+\frac{1}{1-\beta_1})\mathbb{E}\|\nabla F(\theta_t)-\nabla F(\theta_{t-1})\|^2  \nonumber \\
& \quad + (1-\beta_1)^2\mathbb{E} \|\nabla F(\theta_t)- \nabla f(\theta_{t};z_{t+1}) \|^2 \nonumber \\
 & \mathop{\leq}^{(iii)} \beta_1\mathbb{E} \|\nabla F(\theta_{t-1}) - m_{t} \|^2 + \frac{2}{1-\beta_1}\mathbb{E}\|\nabla F(\theta_t)-\nabla F(\theta_{t-1})\|^2  \nonumber \\
 & \quad + (1-\beta_1)^2\mathbb{E} \|\nabla F(\theta_t)- \nabla f(\theta_{t};z_{t+1}) \|^2 \nonumber \\
  & \leq  \beta_1\mathbb{E} \|\nabla F(\theta_{t-1}) - m_{t} \|^2 + \frac{2}{1-\beta_1}L^2\mathbb{E}\|\theta_t-\theta_{t-1}\|^2  + (1-\beta_1)^2\sigma^2,
 \end{align}
 where the equality $(i)$ holds by $\mathbb{E} [\nabla f(\theta_{t};z_{t+1})]=\nabla F(\theta_t)$, and
 the inequality $(ii)$ holds by Young's inequality, and the inequality $(iii)$ is due to $0< \beta_1 < 1$ such that  $\beta_1^2(2-\beta_1) =(1 -(1-\beta_1))^2(1+1-\beta_1)=1-(1-\beta_1)-(1-\beta_1)^2+
 (1-\beta_1)^3\leq 1-(1-\beta_1)$ and $\beta_1^2(1+\frac{1}{1-\beta_1}) \leq \frac{2}{1-\beta_1}$, and the last inequality holds by Assumptions~\ref{ass:s2},~\ref{ass:v}.
 Then we have
 \begin{align}
 & \E\|\nabla F(\theta_{t}) - m_{t+1}\|^2  - \mathbb{E} \|\nabla F(\theta_{t-1}) - m_{t} \|^2 \nonumber \\
  & \leq -(1 -\beta_1)\mathbb{E} \|\nabla F(\theta_{t-1}) - m_{t} \|^2 + \frac{2}{1-\beta_1}L^2\mathbb{E}\|\theta_t-\theta_{t-1}\|^2  + (1-\beta_1)^2\sigma^2.
 \end{align}
 Let $\beta_1 = 1-c\eta$, we can obtain
  \begin{align}
 &\E\|\nabla F(\theta_{t}) - m_{t+1}\|^2  - \mathbb{E} \|\nabla F(\theta_{t-1}) - m_{t} \|^2 \nonumber \\
  & \leq -c\eta\mathbb{E} \|\nabla F(\theta_{t-1}) - m_{t} \|^2 + \frac{2}{c \eta}L^2\mathbb{E}\|\theta_t-\theta_{t-1}\|^2  + c^2\eta^2\sigma^2.
 \end{align}

\end{proof}

\begin{theorem}  (Restatement of Theorem~\ref{th:3})
Assume the sequence $\{\theta_t\}_{t=0}^T$ is generated
from Algorithm~\ref{alg:1}. Under the Assumptions~\ref{ass:s2},~\ref{ass:g},~\ref{ass:v},~\ref{ass:f}, and let $0\leq \lambda < \min(\frac{1}{\eta},\frac{1}{\eta T^\gamma \bar{G}\hat{G}})$, $\|\theta_0\| \leq \eta \bar{G}$, $c\geq \frac{16L}{\breve{\rho}}$, $\beta_1 = 1-c\eta\in (0,1)$, $\beta_2\in (0,1)$ and $0<\eta \leq \frac{\breve{\rho}}{4L}$, we have
\begin{align}
\frac{1}{T+1}\sum_{t=0}^T\mathbb{E}\|\nabla F(\theta_{t})\|
   \leq \hat{G}\big(\frac{4\sqrt{2\Delta}}{\sqrt{T\eta\breve{\rho}}} + \frac{4\sqrt{2}}{T^{\gamma-1}\breve{\rho}} +
  \frac{4c\sigma\sqrt{\eta}}{\sqrt{L\breve{\rho}}}\big)+ \frac{1}{T^{\gamma-1}},
\end{align}
where $\bar{G}=\frac{G}{(1-\beta_1)(\rho+\varepsilon)}$, $\hat{G} =\frac{G^2+\varepsilon}{1-\beta_2}$, $\breve{\rho}=(1-\beta_1)(\rho+(1-\beta_2)\varepsilon)$ and $\rho = \min_{t\geq 1} \min_{j\in [d]}(\hat{v}_t)_j$.
\end{theorem}

\begin{proof}
From our Algorithm~\ref{alg:1},
since $m_t$ is exponential moving average of $g_t$, by using Assumption~\ref{ass:g}, we have $\|m_t\|\leq G$.
Since $\theta_{t} = \theta_{t-1} - \eta (\frac{\hat{m}_t}{\hat{v}_t+\varepsilon} + \lambda \theta_{t-1})$ from the line 10 of Algorithm~\ref{alg:1}, let $\rho = \min_{t\geq 1} \min_{j\in [d]}(\hat{v}_t)_j$, we have
 \begin{align}
  \|\theta_t\| & = \| \theta_{t-1} - \eta (\frac{\hat{m}_t}{\hat{v}_t+\varepsilon} + \lambda \theta_{t-1})\| \nonumber \\
  & = \| (1-\eta\lambda)\theta_{t-1} - \eta \frac{\hat{m}_t}{\hat{v}_t+\varepsilon} \| \nonumber \\
  & \leq (1-\eta\lambda) \|\theta_{t-1}\| + \eta \| \frac{\hat{m}_t}{\hat{v}_t+\varepsilon}\| \nonumber \\
  & \leq (1-\eta\lambda) \|\theta_{t-1}\| + \frac{\eta G}{(1-\beta_1^t)(\rho+\varepsilon)} \nonumber \\
  & \leq (1-\eta\lambda)^{t}\|\theta_{0}\| + \frac{t\eta G}{(1-\beta_1)(\rho+\varepsilon)} \nonumber \\
  & \leq (t+1)\frac{\eta G}{(1-\beta_1)(\rho+\varepsilon)},
 \end{align}
 where the first inequality is due to $0\leq \lambda <\frac{1}{\eta}$, and the last inequality holds by $\|\theta_0\| \leq \frac{\eta G}{(1-\beta_1)(\rho+\varepsilon)}$. Let $\bar{G}=\frac{G}{(1-\beta_1)(\rho+\varepsilon)}$, we have
 $\|\theta_t\|\leq (t+1)\eta \bar{G}$ for all $t\geq1$.

We could rewrite the line 10 of Algorithm~\ref{alg:1}, for all $j=1,2,\cdots,d$
\begin{align}
(\theta_{t})_j & = (\theta_{t-1})_j - \eta\big(\frac{ (\hat{m}_t)_j}{(\hat{v}_t)_j+\varepsilon} -\lambda (\theta_{t-1})_j\big) \nonumber \\
& = (1-\lambda\eta)(\theta_{t-1})_j - \eta\frac{\frac{1}{1-\beta_1^t}(m_t)_j}{\frac{(v_t)_j}{1-\beta_2^t}+\varepsilon},
\end{align}
where $(\cdot)_j$ denotes the $j$-th element of vector.
By using Assumption~\ref{ass:g}, since $v_t$ is exponential moving average of $g^2_t$, we have $(v_t)_j\leq G^2$ for all $j=1,2,\cdots,d$. Then we have $\frac{1-\beta_2}{G^2+\varepsilon}\leq \frac{\frac{1}{1-\beta_1^t}}{\frac{(v_t)_j}{1-\beta_2^t}+\varepsilon} \leq \frac{1}{(1-\beta_1)(\rho+(1-\beta_2)\varepsilon)}$.
Let $H_t = \mbox{diag}\big(\frac{\frac{1}{1-\beta_1^t}}{\frac{(v_t)_j}{1-\beta_2^t}+\varepsilon}\big)$ be a diagonal matrix, we have $ \frac{1-\beta_2}{G^2+\varepsilon}I_d \preceq H_t \preceq \frac{1}{(1-\beta_1)(\rho+(1-\beta_2)\varepsilon)}I_d$. Then we can also rewrite the line 10 of Algorithm~\ref{alg:1} as follow:
\begin{align} \label{eq:E1}
\theta_{t} & = \theta_{t-1} - \eta (H_t m_t + \lambda\theta_{t-1} )\nonumber\\
& =(1-\eta\lambda)\theta_{t-1} - \eta H_t m_t \nonumber \\
&=\arg\min_{\theta\in \R^d} \Big\{ \langle m_t,\theta\rangle + \frac{1}{2\eta}\big(\theta-(1-\lambda\eta)\theta_{t-1}\big)^TH_t^{-1}
\big(\theta-(1-\lambda\eta)\theta_{t-1}\big)\Big\} .
\end{align}
By using the optimality condition of the subproblem~(\ref{eq:E1}),
we have
\begin{align} \label{eq:E2}
\langle m_t +\frac{1}{\eta}H_t^{-1}\big(\theta_{t}-(1-\lambda\eta)\theta_{t-1}\big),\theta-\theta_t\rangle \geq 0, \quad \forall \theta\in \mathbb{R}^d.
\end{align}
By putting $\theta=\theta_{t-1}$ into the above inequality~(\ref{eq:E2}), we have
\begin{align} \label{eq:E3}
\langle m_t +\frac{1}{\eta}H_t^{-1}\big(\theta_{t}-(1-\lambda\eta)\theta_{t-1}\big),\theta_{t-1}-\theta_{t}\rangle
\geq 0.
\end{align}
Thus we can obtain
\begin{align} \label{eq:E4}
\langle m_t,\theta_{t-1}-\theta_{t}\rangle &\geq \frac{1}{\eta}\langle H_t^{-1}(\theta_{t}-\theta_{t-1}),\theta_{t}-\theta_{t-1}\rangle + \lambda\langle H_t^{-1}\theta_{t-1},\theta_t-\theta_{t-1} \rangle \nonumber \\
& \geq \frac{\hat{\rho}}{\eta}\|\theta_{t}-\theta_{t-1}\|^2+ \lambda\langle H_t^{-1}\theta_{t-1},\theta_t-\theta_{t-1} \rangle,
\end{align}
where the last inequality holds by $\breve{\rho} I_d \preceq H_t^{-1} \preceq \hat{G}I_d$ with
$\breve{\rho}=(1-\beta_1)(\rho+(1-\beta_2)\varepsilon)$ and $\hat{G}=\frac{G^2+\varepsilon}{1-\beta_2}$.

 Since $\breve{\rho} I_d \preceq H_t^{-1} \preceq \hat{G}I_d$ and $\|\theta_{t-1}\|\leq t\eta \bar{G}$, we have $\|H_t^{-1}\theta_{t-1}\|^2 \leq \hat{G}^2\|\theta_{t-1}\|^2 \leq t^2\eta^2\bar{G}^2\hat{G}^2$ for all $t\geq1$.

According to Assumption~\ref{ass:s2}, i.e., $F(\theta)$ is $L$-smooth, we have
\begin{align} \label{eq:E5}
  \mathbb{E} [F(\theta_{t})]
 & \leq  \mathbb{E} [F(\theta_{t-1}) + \nabla F(\theta_{t-1})^T(\theta_{t}-\theta_{t-1}) + \frac{L}{2}\|\theta_{t}-\theta_{t-1}\|^2] \nonumber \\
 & = \mathbb{E} [F(\theta_{t-1}) + (\nabla F(\theta_{t-1})-m_{t})^T(\theta_{t}-\theta_{t-1}) + m_{t}^T(\theta_{t}-\theta_{t-1})+ \frac{L}{2}\|\theta_{t}-\theta_{t-1}\|^2] \nonumber \\
 & \mathop{\leq}^{(i)} \mathbb{E} [F(\theta_{t-1}) + \frac{\eta}{2\breve{\rho}}\|\nabla F(\theta_{t-1})-m_{t}\|^2 + \frac{\breve{\rho}}{2\eta}\|\theta_{t}-\theta_{t-1}\|^2 + m_t^T(\theta_{t}-\theta_{t-1})+ \frac{L}{2}\|\theta_{t}-\theta_{t-1}\|^2] \nonumber \\
 & \mathop{\leq}^{(ii)}  \mathbb{E} [F(\theta_{t-1}) + \frac{\eta}{2\breve{\rho}}\|\nabla F(\theta_{t-1})-m_{t}\|^2 + \frac{\breve{\rho}}{2\eta}\|\theta_{t}-\theta_{t-1}\|^2 - \frac{\breve{\rho}}{\eta_{t}}\|\theta_{t}-\theta_{t-1}\|^2 \nonumber \\
 & \quad -  \lambda\langle H_t^{-1}\theta_{t-1},\theta_t-\theta_{t-1} \rangle + \frac{L}{2}\|\theta_{t}-\theta_{t-1}\|^2] \nonumber \\
 & \leq \mathbb{E} [F(\theta_{t-1}) + \frac{\eta}{2\breve{\rho}}\|\nabla F(\theta_{t-1})-m_{t}\|^2 + \frac{\breve{\rho}}{2\eta}\|\theta_{t}-\theta_{t-1}\|^2 - \frac{\breve{\rho}}{\eta}\|\theta_{t}-\theta_{t-1}\|^2 \nonumber \\
 & \quad + \frac{\lambda^2\eta}{\breve{\rho}}\| H_t^{-1}\theta_{t-1}\|^2+ \frac{\breve{\rho}}{4\eta}\|\theta_t-\theta_{t-1}\|^2 + \frac{L}{2}\|\theta_{t}-\theta_{t-1}\|^2] \nonumber \\
 & \leq \mathbb{E} [F(\theta_{t-1}) + \frac{\eta}{\breve{\rho}T^{2\gamma-2}}+ \frac{\eta}{2\breve{\rho}}\|\nabla F(\theta_{t-1})-m_{t}\|^2  - \frac{\breve{\rho}}{8\eta}\|\theta_{t}-\theta_{t-1}\|^2],
\end{align}
where the above inequality $(i)$ holds by Young's inequality, the above inequality $(ii)$ follows by
 the above inequality~(\ref{eq:E4}), and the last inequality holds by $0< \eta \leq \frac{\breve{\rho}}{4L}$
 and $0\leq \lambda \leq \frac{1}{\eta T^\gamma \bar{G}\hat{G}}$.

We define a useful Lyapunov function $\Phi_t = F(\theta_t) + \frac{1}{2L}\|\nabla F(\theta_{t}) - m_{t+1}\|^2$.
Then we have
 \begin{align}
  \E[\Phi_{t} -\Phi_{t-1} ]  & = \E [F(\theta_{t})] - \E [F(\theta_{t-1}) ] + \frac{1}{2L} (\E \|\nabla F(\theta_{t}) - m_{t+1}\|^2 - \E \|\nabla F(\theta_{t-1}) - m_{t}\|^2 ) \nonumber \\
 & \leq \frac{\eta}{\breve{\rho}T^{2\gamma-2}}+ \frac{\eta}{2\breve{\rho}}\E \|\nabla F(\theta_{t-1})-m_{t}\|^2  - \frac{\breve{\rho}}{8\eta} \E \|\theta_{t}-\theta_{t-1}\|^2 \nonumber \\
 & \quad + \frac{1}{2L} \big( -c\eta\mathbb{E} \|\nabla F(\theta_{t-1}) - m_{t} \|^2 + \frac{2}{c \eta}L^2\mathbb{E}\|\theta_t-\theta_{t-1}\|^2  + c^2\eta^2\sigma^2\big) \nonumber \\
 & \mathop{\leq}^{(i)}  \frac{\eta}{\breve{\rho}T^{2\gamma-2}}- \frac{c\eta}{4L}\E \|\nabla F(\theta_{t-1})-m_{t}\|^2  - \frac{\breve{\rho}}{16\eta} \E \|\theta_{t}-\theta_{t-1}\|^2 + \frac{c^2\eta^2\sigma^2}{2L} \nonumber \\
 & \leq \frac{\eta}{\breve{\rho}T^{2\gamma-2}}- \frac{4\eta}{\breve{\rho}}\E \|\nabla F(\theta_{t-1})-m_{t}\|^2  - \frac{\breve{\rho}}{16\eta} \E \|\theta_{t}-\theta_{t-1}\|^2 + \frac{c^2\eta^2\sigma^2}{2L},
 \end{align}
where the first inequality holds by Lemma~\ref{lem:A2}, and the above inequality $(i)$ holds by $c\geq \frac{16L}{\breve{\rho}}$ such as
$\frac{\breve{\rho}}{16\eta}\geq \frac{L}{c\eta}$ and $ \frac{c\eta}{4L}\geq \frac{c\eta}{32L} \geq \frac{\eta}{2\breve{\rho}}$, and
the last inequality also is due to $c\geq \frac{16L}{\breve{\rho}}$.
Then we have
\begin{align} \label{eq:E6}
    \frac{4\eta}{\breve{\rho}}\E \|\nabla F(\theta_{t-1})-m_{t}\|^2 + \frac{\breve{\rho}}{16\eta} \E \|\theta_{t}-\theta_{t-1}\|^2 \leq \E[\Phi_{t-1} - \Phi_{t}] + \frac{\eta}{\breve{\rho}T^{2\gamma-2}} + \frac{c^2\eta^2\sigma^2}{2L}.
 \end{align}
By multiplying both sides of the above inequality~(\ref{eq:E6}) by $\frac{16}{\eta\breve{\rho}}$, we can obtain
\begin{align}
     \frac{1}{\breve{\rho}^2}\E \|\nabla F(\theta_{t-1})-m_{t}\|^2 + \frac{1}{\eta^2} \E \|\theta_{t}-\theta_{t-1}\|^2
    & \leq \frac{4}{\breve{\rho}^2}\E \|\nabla F(\theta_{t-1})-m_{t}\|^2 + \frac{1}{\eta^2} \E \|\theta_{t}-\theta_{t-1}\|^2 \nonumber \\
    & \leq \E[\frac{16(\Phi_{t-1} - \Phi_{t})}{\eta\breve{\rho}}] + \frac{16}{T^{2\gamma-2}\breve{\rho}^2} + \frac{8c^2\eta\sigma^2}{L\breve{\rho}}.
 \end{align}

Since $m_1 = \beta_{1}m_{0}+ (1-\beta_{1})g_1=(1-\beta_{1})\nabla f(\theta_{0};z_{1})$, we have
\begin{align}
 \|\nabla F(\theta_0) - m_1\|^2 & = \|\nabla F(\theta_0) - (1-\beta_{1})\nabla f(\theta_{0};z_{1})\|^2 \nonumber \\
 & =\|\nabla F(\theta_0) - \nabla f(\theta_{0};z_{1}) + \beta_1\nabla f(\theta_{0};z_{1})\|^2  \leq 2\sigma^2 + 2\beta_1^2G^2.
\end{align}
Given $\Phi_0 = F(\theta_0) + \frac{1}{2L}\|\nabla F(\theta_{0}) - m_1\|^2$, we have
\begin{align} \label{eq:E7}
    & \frac{1}{T} \sum_{t=1}^T [\frac{1}{\breve{\rho}^2}\E \|\nabla F(\theta_{t-1})-m_{t}\|^2 + \frac{1}{\eta^2} \E \|\theta_{t}-\theta_{t-1}\|^2] \nonumber \\
    & \leq \frac{1}{T} \sum_{t=1}^T\E[\frac{16(\Phi_{t-1} - \Phi_{t})}{\eta\breve{\rho}}] + \frac{16}{T^{2\gamma-2}\breve{\rho}^2} + \frac{8c^2\eta\sigma^2}{L\breve{\rho}} \nonumber \\
    &\leq \frac{16(F(\theta_0) + \frac{1}{2L}(2\sigma^2 + 2\beta_1^2G^2) - F^*)}{T\eta\breve{\rho}} + \frac{16}{T^{2\gamma-2}\breve{\rho}^2} + \frac{8c^2\eta\sigma^2}{L\breve{\rho}}.
 \end{align}
Let $\Delta = F(\theta_0) + \frac{1}{L}(\sigma^2 + \beta_1^2G^2) - F^*$,
we can rewrite the above inequality~(\ref{eq:E6}) as follows:
\begin{align}
    & \frac{1}{T} \sum_{t=1}^T [\frac{1}{\breve{\rho}^2}\E \|\nabla F(\theta_{t-1})-m_{t}\|^2 + \frac{1}{\eta^2} \E \|\theta_{t}-\theta_{t-1}\|^2] \nonumber \\
    &\leq \frac{16\Delta}{T\eta\breve{\rho}} + \frac{16}{T^{2\gamma-2}\breve{\rho}^2} + \frac{8c^2\eta\sigma^2}{L\breve{\rho}}.
 \end{align}
According to the Jensen’s inequality, then we can obtain
\begin{align}  \label{eq:E8}
 & \frac{1}{T}\sum_{t=1}^T\mathbb{E}[\frac{1}{\breve{\rho}} \|\nabla F(\theta_{t-1})-m_{t}\| + \frac{1}{\eta}  \|\theta_{t}-\theta_{t-1}\| ] \nonumber \\
 & \leq  \Big(\frac{2}{T}\sum_{t=1}^T\mathbb{E}[\frac{1}{\breve{\rho}^2} \|\nabla F(\theta_{t-1})-m_{t}\|^2 + \frac{1}{\eta^2} \|\theta_{t}-\theta_{t-1}\|^2]\Big)^{1/2} \nonumber \\
 & \leq  \sqrt{\frac{32\Delta}{T\eta\breve{\rho}} + \frac{32}{T^{2\gamma-2}\breve{\rho}^2} + \frac{16c^2\eta\sigma^2}{L\breve{\rho}}}
  \leq \frac{4\sqrt{2\Delta}}{\sqrt{T\eta\breve{\rho}}} + \frac{4\sqrt{2}}{T^{\gamma-1}\breve{\rho}} +
  \frac{4c\sigma\sqrt{\eta}}{\sqrt{L\breve{\rho}}}.
\end{align}

By using $\theta_{t} = \theta_{t-1} - \eta (H_t m_t + \lambda\theta_{t-1} )$, we have
\begin{align} \label{eq:E9}
 \frac{1}{\breve{\rho}}\|\nabla F(\theta_{t-1})-m_{t}\|  + \frac{1}{\eta}\|\theta_{t}-\theta_{t-1}\|  & =  \frac{1}{\breve{\rho}}\|\nabla F(\theta_{t-1})-m_{t}\|  + \frac{1}{\eta}\|\eta (H_t m_t + \lambda\theta_{t-1} )\| \nonumber \\
 & \geq \frac{1}{\breve{\rho}}\|\nabla F(\theta_{t-1})-m_{t}\|  + \|H_tm_t\| - \|\lambda\theta_{t-1}\|  \nonumber \\
 & =\frac{1}{\breve{\rho}}\|H_t^{-1}H_t(\nabla F(\theta_{t-1})-m_{t})\|  + \|H_t m_t\| - \|\lambda\theta_{t-1}\| \nonumber \\
 & \geq \|H_t(\nabla F(\theta_{t-1})-m_{t})\|  + \|H_t m_t\| - \|\lambda\theta_{t-1}\| \nonumber \\
 & \geq \|H_t\nabla F(\theta_{t-1})\| - \|\lambda\theta_{t-1}\| \nonumber \\
 & \geq \frac{\|\nabla F(\theta_{t-1})\|}{\hat{G}}- \|\lambda\theta_{t-1}\|,
\end{align}
where the above inequality holds by $\breve{\rho} I_d \preceq H_t^{-1} \preceq \hat{G}I_d$
and $\frac{1}{\hat{G}}I_d \preceq H_t\preceq \frac{1}{\breve{\rho} }I_d$.

By putting the above inequalities~(\ref{eq:E9}) into~(\ref{eq:E8}), we can obtain
\begin{align}
  \frac{1}{T}\sum_{t=1}^T\mathbb{E}\|\nabla F(\theta_{t-1})\|
 & \leq  \frac{\hat{G}}{T}\sum_{t=1}^T\mathbb{E}[\frac{1}{\breve{\rho}}\|\nabla F(\theta_{t-1})-m_{t}\|  + \frac{1}{\eta}\|\theta_{t}-\theta_{t-1}\|+\|\lambda\theta_{t-1}\| ] \nonumber \\
 & \leq \hat{G}\big(\frac{4\sqrt{2\Delta}}{\sqrt{T\eta\breve{\rho}}} + \frac{4\sqrt{2}}{T^{\gamma-1}\breve{\rho}} +
  \frac{4c\sigma\sqrt{\eta}}{\sqrt{L\breve{\rho}}}\big) + \frac{\hat{G}}{T}\sum_{t=1}^T\E\|\lambda\theta_{t-1}\| \nonumber \\
 & \mathop{\leq}^{(i)} \hat{G}\big(\frac{4\sqrt{2\Delta}}{\sqrt{T\eta\breve{\rho}}} + \frac{4\sqrt{2}}{T^{\gamma-1}\breve{\rho}} +
  \frac{4c\sigma\sqrt{\eta}}{\sqrt{L\breve{\rho}}}\big) + \frac{\hat{G}}{T}\sum_{t=1}^T\lambda t\eta \bar{G} \nonumber \\
 & \mathop{\leq}^{(ii)} \hat{G}\big(\frac{4\sqrt{2\Delta}}{\sqrt{T\eta\breve{\rho}}} + \frac{4\sqrt{2}}{T^{\gamma-1}\breve{\rho}} +
  \frac{4c\sigma\sqrt{\eta}}{\sqrt{L\breve{\rho}}}\big) + \frac{\hat{G}}{T}\sum_{t=1}^T \frac{1}{\eta T^\gamma \bar{G}\hat{G}}t\eta \bar{G}  \nonumber \\
 & \leq \hat{G}\big(\frac{4\sqrt{2\Delta}}{\sqrt{T\eta\breve{\rho}}} + \frac{4\sqrt{2}}{T^{\gamma-1}\breve{\rho}} +
  \frac{4c\sigma\sqrt{\eta}}{\sqrt{L\breve{\rho}}}\big)+ \frac{1}{T^{\gamma-1}},
\end{align}
where the above inequality~$(i)$ is due to $\|\theta_{t-1}\| \leq t\eta \bar{G}$, and
the above inequality~$(ii)$ holds by $0\leq \lambda \leq \frac{1}{\eta T^\gamma \bar{G}\hat{G}}$.
Then we have
\begin{align}
  \frac{1}{T+1}\sum_{t=0}^T\mathbb{E}\|\nabla F(\theta_{t})\|
  & \leq \hat{G}\big(\frac{4\sqrt{2\Delta}}{\sqrt{(T+1)\eta\breve{\rho}}} + \frac{4\sqrt{2}}{(T+1)^{\gamma-1}\breve{\rho}} +
  \frac{4c\sigma\sqrt{\eta}}{\sqrt{L\breve{\rho}}}\big)+ \frac{1}{(T+1)^{\gamma-1}} \nonumber \\
  & \leq \hat{G}\big(\frac{4\sqrt{2\Delta}}{\sqrt{T\eta\breve{\rho}}} + \frac{4\sqrt{2}}{T^{\gamma-1}\breve{\rho}} +
  \frac{4c\sigma\sqrt{\eta}}{\sqrt{L\breve{\rho}}}\big)+ \frac{1}{T^{\gamma-1}}.
\end{align}

Let $\eta=\frac{1}{\sqrt{T}}$ and $\gamma=\frac{3}{4}$, we can obtain
\begin{align}
  \frac{1}{T+1}\sum_{t=0}^T\mathbb{E}\|\nabla F(\theta_{t})\|
   \leq  \hat{G}\big(\frac{4\sqrt{2\Delta}}{\sqrt{\breve{\rho}}T^{1/4}} + \frac{4\sqrt{2}}{\breve{\rho}T^{1/4}} +
  \frac{4c\sigma}{\sqrt{L\breve{\rho}}T^{1/4}}\big)+ \frac{1}{T^{1/4}}.
\end{align}

Further let $G=O(1)$, $L=O(1)$, $\sigma=O(1)$, $\beta_1=O(1)$ with $\beta_1\in(0,1)$, $\beta_2=O(1)$ with $\beta_2\in(0,1)$ and $c=O(1)$, we have $\hat{G}=\frac{G^2+\varepsilon}{1-\beta_2}=O(1)$ and $\Delta = F(\theta_0) + \frac{1}{L}(\sigma^2 + \beta_1^2G^2) - F^*=O(1)$. Since $\breve{\rho}=(1-\beta_1)(\rho+(1-\beta_2)\varepsilon) \leq \rho+\varepsilon$ and $\varepsilon$ is very small, $\breve{\rho}$ is very small. Then we have
\begin{align}
  \frac{1}{T+1}\sum_{t=0}^T\mathbb{E}\|\nabla F(\theta_{t})\|  \leq O(\frac{1}{\breve{\rho}T^{1/4}}).
\end{align}

\end{proof}

\section{Convergence Analysis of HomeAdam(W) Algorithms}
\label{ap:ca2}
In this section, we provide a detailed convergence analysis of our HomeAdam(W) algorithms.

\begin{lemma} \label{lem:A3}
Assume the sequence $\{m_t\}_{t=0}^T$ is generated from Algorithm~\ref{alg:2}, let $\beta_1 = 1-c\eta\in (0,1)$, $\beta_2\in (0,1)$, we have
\begin{align}
 \E\|\nabla F(\theta_{t}) - m_{t+1}\|^2   \leq (1-c\eta)\mathbb{E} \|\nabla F(\theta_{t-1}) - m_{t} \|^2 + \frac{2}{c \eta}L^2\mathbb{E}\|\theta_t-\theta_{t-1}\|^2  + c^2\eta^2\sigma^2,
\end{align}
where $c>0$.
\end{lemma}

\begin{proof}
This proof can follows the above proof of Lemma~\ref{lem:A2}.

\end{proof}

\begin{theorem} (Restatement of Theorem~\ref{th:4})
Assume the sequence $\{\theta_t\}_{t=0}^T$ is generated
from Algorithm~\ref{alg:2}. Under the Assumptions~\ref{ass:s2},~\ref{ass:g},~\ref{ass:v},~\ref{ass:f}, and let $0\leq \lambda < \min(\frac{1}{\eta},\frac{1}{\eta T^\gamma \tilde{G}\hat{G}})$, $\|\theta_0\| \leq \eta \tilde{G}$, $c\geq \frac{32L}{\breve{\tau}}$, $\beta_1 = 1-c\eta\in (0,1)$, $\beta_2\in (0,1)$, $0<\eta \leq \frac{\hat{\tau}}{4L}$ and $\tau>0$, we have
\begin{align}
\frac{1}{T+1}\sum_{t=0}^T\mathbb{E}\|\nabla F(\theta_{t})\|
   \leq \frac{8\sqrt{\Delta}\breve{G}}{\sqrt{{T}\eta\breve{\tau}}}
  + \frac{8\breve{G}}{\breve{\tau}T^{\gamma-1}} +
  \frac{4\sqrt{2\eta}c\sigma\breve{G}}{\sqrt{L\breve{\tau}}}+ \frac{\breve{G}}{\hat{G}T^{\gamma-1}},
\end{align}
where $\tilde{G}=\max(\frac{ G}{(1-\beta_1)(\tau+(1-\beta_2)\varepsilon)},\frac{G}{1-\beta_1})$, $\hat{G} =\frac{G^2+\varepsilon}{1-\beta_2}$, $\breve{G} = \max(1,\frac{G^2+\varepsilon}{1-\beta_2})$, $\hat{\tau} = (1-\beta_1)(\tau + (1-\beta_2)\varepsilon)$ and $\breve{\tau}=\min(1-\beta_1,\hat{\tau})$.
\end{theorem}

\begin{proof}
From our Algorithm~\ref{alg:2},
since $m_t$ is exponential moving average of $g_t$, by using Assumption~\ref{ass:g}, we have $\|m_t\|\leq G$.
When $\min_{1\leq j\leq d} (v_t)_j \geq \tau>0$, we have $\theta_{t} = \theta_{t-1} - \eta (\frac{\hat{m}_t}{\hat{v}_t+\varepsilon} + \lambda \theta_{t-1})$. Let $0\leq \lambda <\frac{1}{\eta}$, we have
 \begin{align}
  \|\theta_t\| & = \| \theta_{t-1} - \eta (\frac{\hat{m}_t}{\hat{v}_t+\varepsilon} + \lambda \theta_{t-1})\| \nonumber \\
  & = \| (1-\eta\lambda)\theta_{t-1} - \eta \frac{\hat{m}_t}{\hat{v}_t+\varepsilon} \| \nonumber \\
  & \leq (1-\eta\lambda) \|\theta_{t-1}\| + \eta \| \frac{\frac{m_t}{1-\beta_1^t}}{\frac{v_t}{1-\beta_2^t}+\varepsilon}\| \nonumber \\
  & \leq (1-\eta\lambda) \|\theta_{t-1}\| + \frac{\eta G}{(1-\beta_1)(\tau+(1-\beta_2)\varepsilon)} \nonumber \\
  & \leq (1-\eta\lambda)^{t}\|\theta_{0}\| + \frac{t\eta G}{(1-\beta_1)(\tau+(1-\beta_2)\varepsilon)} \nonumber \\
  & \leq \frac{(t+1)\eta G}{(1-\beta_1)(\tau+(1-\beta_2)\varepsilon)},
 \end{align}
 where the last inequality holds by $\|\theta_0\| \leq \frac{\eta G}{(1-\beta_1)(\tau+(1-\beta_2)\varepsilon)}$.

When $\min_{1\leq j\leq d} (v_t)_j < \tau$, we have $\theta_{t} = \theta_{t-1} - \eta (\hat{m}_t+\lambda\theta_{t-1})$. Let $0\leq \lambda <\frac{1}{\eta}$, we have
 \begin{align}
  \|\theta_t\| & = \| \theta_{t-1} - \eta (\hat{m}_t+\lambda\theta_{t-1})\| \nonumber \\
  & = \| (1-\eta\lambda)\theta_{t-1} - \eta \hat{m}_t \| \nonumber \\
  & \leq (1-\eta\lambda) \|\theta_{t-1}\| + \eta \| \frac{m_t}{1-\beta_1^t}\| \nonumber \\
  & \leq (1-\eta\lambda) \|\theta_{t-1}\| + \frac{\eta G}{1-\beta_1} \nonumber \\
  & \leq (1-\eta\lambda)^{t}\|\theta_{0}\| + \frac{t\eta G}{1-\beta_1} \nonumber \\
  & \leq \frac{(t+1)\eta G}{1-\beta_1},
 \end{align}
 where the last inequality holds by $\|\theta_0\| \leq \frac{\eta G}{1-\beta_1}$.

Let $\tilde{G}=\max(\frac{ G}{(1-\beta_1)(\tau+(1-\beta_2)\varepsilon)},\frac{G}{1-\beta_1})$, we have
 $\|\theta_t\|\leq (t+1) \eta\tilde{G}$.

When $\min_{1\leq j\leq d} (v_t)_j \geq \tau>0$, we could rewrite the line 11 of Algorithm~\ref{alg:2}, for all $j=1,2,\cdots,d$
\begin{align}
(\theta_{t})_j & = (\theta_{t-1})_j - \eta\big(\frac{ (\hat{m}_t)_j}{(\hat{v}_t)_j+\varepsilon} -\lambda (\theta_{t-1})_j\big) \nonumber \\
& = (1-\lambda\eta)(\theta_{t-1})_j - \eta\frac{ \frac{1}{1-\beta_1^t}(m_t)_j}{\frac{1}{1-\beta_2^t}(v_t)_j+\varepsilon},
\end{align}
where $(\cdot)_j$ denotes the $j$-th element of vector.
 Then we have $\frac{1-\beta_2}{G^2+\varepsilon} \leq \frac{\frac{1}{1-\beta_1^t}}{\frac{1}{1-\beta_2^t}(v_t)_j+\varepsilon} \leq \frac{1}{(1-\beta_1)(\tau + (1-\beta_2)\varepsilon)}$.
Let $H_t = \mbox{diag}(\frac{ \frac{1}{1-\beta_1^t}}{\frac{1}{1-\beta_2^t}v_t+\varepsilon})$ be a diagonal matrix, we have $\frac{1-\beta_2}{G^2+\varepsilon} I_d \preceq H_t \preceq \frac{1}{(1-\beta_1)(\tau + (1-\beta_2)\varepsilon)}I_d$. Then we can also rewrite the line 11 of Algorithm~\ref{alg:2} as follows:
\begin{align} \label{eq:F1}
\theta_{t} & = \theta_{t-1} - \eta (H_t m_t + \lambda\theta_{t-1} )\nonumber\\
& =(1-\eta\lambda)\theta_{t-1} - \eta H_t m_t \nonumber \\
&=\arg\min_{\theta\in \R^d} \Big\{ \langle m_t,\theta\rangle + \frac{1}{2\eta}\big(\theta-(1-\lambda\eta)\theta_{t-1}\big)^TH_t^{-1}
\big(\theta-(1-\lambda\eta)\theta_{t-1}\big)\Big\} .
\end{align}
By using the optimality condition of the subproblem~(\ref{eq:F1}),
we have
\begin{align} \label{eq:F2}
\langle m_t +\frac{1}{\eta}H_t^{-1}\big(\theta_{t}-(1-\lambda\eta)\theta_{t-1}\big),\theta-\theta_t\rangle \geq 0, \quad \forall \theta\in \mathbb{R}^d.
\end{align}
By putting $\theta=\theta_{t-1}$ into the above inequality~(\ref{eq:F2}), we have
\begin{align} \label{eq:F3}
\langle m_t +\frac{1}{\eta}H_t^{-1}\big(\theta_{t}-(1-\lambda\eta)\theta_{t-1}\big),\theta_{t-1}-\theta_{t}\rangle \geq 0.
\end{align}
Let $\hat{\tau} = (1-\beta_1)(\tau + (1-\beta_2)\varepsilon)$, we can obtain
\begin{align} \label{eq:F4}
\langle m_t,\theta_{t-1}-\theta_{t}\rangle &\geq \frac{1}{\eta}\langle H_t^{-1}(\theta_{t}-\theta_{t-1}),\theta_{t}-\theta_{t-1}\rangle + \lambda\langle H_t^{-1}\theta_{t-1},\theta_t-\theta_{t-1} \rangle \nonumber \\
& \geq \frac{\hat{\tau}}{\eta}\|\theta_{t}-\theta_{t-1}\|^2+ \lambda\langle H_t^{-1}\theta_{t-1},\theta_t-\theta_{t-1} \rangle.
\end{align}
Since $\frac{1-\beta_2}{G^2+\varepsilon} I_d \preceq H_t \preceq \frac{1}{(1-\beta_1)(\tau + (1-\beta_2)\varepsilon)}I_d$, let $\hat{\tau} =(1-\beta_1)(\tau + (1-\beta_2)\varepsilon) $ and
 $\hat{G} =\frac{G^2+\varepsilon}{1-\beta_2}$, we have $\hat{\tau}I_d \preceq H_t^{-1} \preceq \hat{G} I_d$ and
 $\|H_t^{-1}\theta_{t-1}\|^2 \leq \hat{G}^2\|\theta_{t-1}\|^2 \leq t^2\eta^2\tilde{G}^2\hat{G}^2$.

According to Assumption~\ref{ass:s2}, i.e., $F(\theta)$ is $L$-smooth, we have
\begin{align} \label{eq:F5}
 \mathbb{E} [F(\theta_{t})]
 & \leq  \mathbb{E} [F(\theta_{t-1}) + \nabla F(\theta_{t-1})^T(\theta_{t}-\theta_{t-1}) + \frac{L}{2}\|\theta_{t}-\theta_{t-1}\|^2] \nonumber \\
 & = \mathbb{E} [F(\theta_{t-1}) + (\nabla F(\theta_{t-1})-m_{t})^T(\theta_{t}-\theta_{t-1}) + m_{t}^T(\theta_{t}-\theta_{t-1})+ \frac{L}{2}\|\theta_{t}-\theta_{t-1}\|^2] \nonumber \\
 & \leq \mathbb{E} [F(\theta_{t-1}) + \frac{\eta}{2\hat{\tau}}\|\nabla F(\theta_{t-1})-m_{t}\|^2 + \frac{\hat{\tau}}{2\eta}\|\theta_{t}-\theta_{t-1}\|^2 + m_t^T(\theta_{t}-\theta_{t-1})+ \frac{L}{2}\|\theta_{t}-\theta_{t-1}\|^2] \nonumber \\
 & \mathop{\leq}^{(i)}  \mathbb{E} [F(\theta_{t-1}) + \frac{\eta}{2\hat{\tau}}\|\nabla F(\theta_{t-1})-m_{t}\|^2 + \frac{\hat{\tau}}{2\eta}\|\theta_{t}-\theta_{t-1}\|^2 - \frac{\hat{\tau}}{\eta_{t}}\|\theta_{t}-\theta_{t-1}\|^2 \nonumber \\
 & \quad -  \lambda\langle H_t^{-1}\theta_{t-1},\theta_t-\theta_{t-1} \rangle + \frac{L}{2}\|\theta_{t}-\theta_{t-1}\|^2] \nonumber \\
 & \mathop{\leq}^{(ii)} \mathbb{E} [F(\theta_{t-1}) + \frac{\eta}{2\hat{\tau}}\|\nabla F(\theta_{t-1})-m_{t}\|^2 + \frac{\hat{\tau}}{2\eta}\|\theta_{t}-\theta_{t-1}\|^2 - \frac{\hat{\tau}}{\eta}\|\theta_{t}-\theta_{t-1}\|^2 \nonumber \\
 & \quad + \frac{\lambda^2\eta}{\hat{\tau}}\|H_t^{-1}\theta_{t-1}\|^2+ \frac{\hat{\tau}}{4\eta}\|\theta_t-\theta_{t-1}\|^2 + \frac{L}{2}\|\theta_{t}-\theta_{t-1}\|^2] \nonumber \\
 & \leq \mathbb{E} [F(\theta_{t-1}) + \frac{\eta}{\hat{\tau}T^{2\gamma-2}}+ \frac{\eta}{2\hat{\tau}}\|\nabla F(\theta_{t-1})-m_{t}\|^2  - \frac{\hat{\tau}}{8\eta}\|\theta_{t}-\theta_{t-1}\|^2],
\end{align}
where the inequality $(i)$ holds by the above inequality~(\ref{eq:F4}), and the last inequality holds by $\eta \leq \frac{\hat{\tau}}{4L}$ and $\lambda \leq \frac{1}{\eta T^\gamma \tilde{G}\hat{G}}$.

Then we have
\begin{align} \label{eq:F6}
 \E \|\theta_{t}-\theta_{t-1}\|^2 \leq \frac{8\eta(F(\theta_{t-1})-F(\theta_{t}))}{\hat{\tau}}+ \frac{4\eta^2}{\hat{\tau}^2}\|\nabla F(\theta_{t-1})-m_{t}\|^2 + \frac{8\eta^2}{\hat{\tau}^2T^{2\gamma-2}}.
\end{align}

When $\min_{1\leq j\leq d} (v_t)_j < \tau$, we have
 \begin{align}
  \theta_{t} = \theta_{t-1} - \eta (\hat{m}_t+\lambda\theta_{t-1}) =
  \theta_{t-1} - \eta (\frac{m_t}{1-\beta_1^t}+\lambda\theta_{t-1}).
 \end{align}
Let $\tilde{H}_t = \frac{1}{1-\beta_1^t}I_d$ be a diagonal matrix, and we have $ I_d \preceq \tilde{H}_t \preceq \frac{1}{1-\beta_1}I_d$.
 Following the above proof, we have
\begin{align} \label{eq:F7}
\langle m_t,\theta_{t-1}-\theta_{t}\rangle &\geq \frac{1}{\eta}\langle \tilde{H}_t^{-1}(\theta_{t}-\theta_{t-1}),\theta_{t}-\theta_{t-1}\rangle + \lambda\langle \tilde{H}_t^{-1}\theta_{t-1},\theta_t-\theta_{t-1} \rangle \nonumber \\
& \geq \frac{1-\beta_1}{\eta}\|\theta_{t}-\theta_{t-1}\|^2+ \lambda\langle \tilde{H}_t^{-1}\theta_{t-1},\theta_t-\theta_{t-1} \rangle.
\end{align}
Then we have
\begin{align} \label{eq:F8}
  \mathbb{E} [F(\theta_{t})]
 \leq \mathbb{E} [F(\theta_{t-1}) + \frac{\eta}{T^{2\gamma-2}}+ \frac{\eta}{2(1-\beta_1)}\|\nabla F(\theta_{t-1})-m_{t}\|^2  - \frac{1-\beta_1}{8\eta}\|\theta_{t}-\theta_{t-1}\|^2].
\end{align}
Thus we can obtain
\begin{align} \label{eq:F9}
  \E \|\theta_{t}-\theta_{t-1}\|^2 \leq \frac{8\eta(F(\theta_{t-1})-F(\theta_{t}))}{1-\beta_1} + \frac{8\eta^2}{(1-\beta_1)T^{2\gamma-2}}+ \frac{4\eta^2}{(1-\beta_1)^2}\E\|\nabla F(\theta_{t-1})-m_{t}\|^2.
\end{align}

Let $\breve{\tau}=\min(1-\beta_1,\hat{\tau})$ with $\hat{\tau} = (1-\beta_1)(\tau + (1-\beta_2)\varepsilon)$ and $\breve{G} = \max(1,\frac{G^2+\varepsilon}{1-\beta_2})$, and further let $ \frac{1}{\breve{G}}I_d=\min(1,\frac{1-\beta_2}{G^2+\varepsilon} ) I_d\preceq \hat{H}_t \preceq \max(\frac{1}{1-\beta_1},\frac{1}{(1-\beta_1)(\tau + (1-\beta_2)\varepsilon)})I_d = \frac{1}{\breve{\tau}}I_d$.

According to the above inequalities~(\ref{eq:F6}) and~(\ref{eq:F9}),
 we have
\begin{align}
 \E \|\theta_{t}-\theta_{t-1}\|^2 & \leq \frac{8\eta(F(\theta_{t-1})-F(\theta_{t}))}{\hat{\tau}}+ \frac{8\eta(F(\theta_{t-1})-F(\theta_{t}))}{1-\beta_1}+ \frac{4\eta^2}{\hat{\tau}^2}\|\nabla F(\theta_{t-1})-m_{t}\|^2 \nonumber \\
 & \quad + \frac{4\eta^2}{(1-\beta_1)^2}\E\|\nabla F(\theta_{t-1})-m_{t}\|^2+ \frac{8\eta^2}{\hat{\tau}^2T^{2\gamma-2}}+ \frac{8\eta^2}{(1-\beta_1)T^{2\gamma-2}} \nonumber \\
 & \leq  \frac{16\eta(F(\theta_{t-1})-F(\theta_{t}))}{\breve{\tau}} + \frac{8\eta^2}{\breve{\tau}^2}\|\nabla F(\theta_{t-1})-m_{t}\|^2+ \frac{16\eta^2}{\breve{\tau}^2T^{2\gamma-2}},
\end{align}
where the last inequality is due to $\breve{\tau}=\min(1-\beta_1,\hat{\tau})$ with $\hat{\tau} = (1-\beta_1)(\tau + (1-\beta_2)\varepsilon)$.

Then we can obtain
\begin{align} \label{eq:F10}
 F(\theta_{t-1}) \leq F(\theta_{t})) - \frac{\breve{\tau}}{16\eta}\E \|\theta_{t}-\theta_{t-1}\|^2  + \frac{\eta}{2\breve{\tau}}\|\nabla F(\theta_{t-1})-m_{t}\|^2+ \frac{\eta}{\breve{\tau}T^{2\gamma-2}}.
\end{align}

We define a useful Lyapunov function $\Omega_t = F(\theta_t) + \frac{1}{2L}\|\nabla F(\theta_{t}) - m_{t+1}\|^2$.
 According to the above inequality~(\ref{eq:F10}), we have
 \begin{align}
  \E[\Omega_{t} -\Omega_{t-1} ]
 & = \E [F(\theta_{t})-F(\theta_{t-1}) ] + \frac{1}{2L} (\E \|\nabla F(\theta_{t}) - m_{t+1}\|^2 - \E \|\nabla F(\theta_{t-1}) - m_{t}\|^2 ) \nonumber \\
 & \leq \frac{\eta}{\breve{\tau}T^{2\gamma-2}}+ \frac{\eta}{2\breve{\tau}}\|\nabla F(\theta_{t-1})-m_{t}\|^2 - \frac{\breve{\tau}}{16\eta}\E \|\theta_{t}-\theta_{t-1}\|^2 \nonumber \\
 & \quad + \frac{1}{2L} \big( -c\eta\mathbb{E} \|\nabla F(\theta_{t-1}) - m_{t} \|^2 + \frac{2}{c \eta}L^2\mathbb{E}\|\theta_t-\theta_{t-1}\|^2  + c^2\eta^2\sigma^2\big) \nonumber \\
 & \mathop{\leq}^{(i)}  \frac{\eta}{\breve{\tau}T^{2\gamma-2}}- \frac{c\eta}{4L}\E \|\nabla F(\theta_{t-1})-m_{t}\|^2  - \frac{\breve{\tau}}{32\eta} \E \|\theta_{t}-\theta_{t-1}\|^2 + \frac{c^2\eta^2\sigma^2}{2L} \nonumber \\
 & \leq \frac{\eta}{\breve{\tau}T^{2\gamma-2}}- \frac{8\eta}{\breve{\tau}}\E \|\nabla F(\theta_{t-1})-m_{t}\|^2  - \frac{\breve{\tau}}{32\eta} \E \|\theta_{t}-\theta_{t-1}\|^2 + \frac{c^2\eta^2\sigma^2}{2L},
 \end{align}
where the first inequality holds by Lemma~\ref{lem:A3}, and the above inequality $(i)$ holds by $c\geq \frac{32L}{\breve{\tau}}$ such as
$\frac{\breve{\tau}}{32\eta}\geq \frac{L}{c\eta}$ and $ \frac{c\eta}{4L}\geq \frac{c\eta}{64L} \geq \frac{\eta}{2\breve{\tau}}$, and
the last inequality also is due to $c\geq \frac{32L}{\breve{\tau}}$.
Then we have
\begin{align} \label{eq:F11}
    \frac{8\eta}{\breve{\tau}}\E \|\nabla F(\theta_{t-1})-m_{t}\|^2 + \frac{\breve{\tau}}{32\eta} \E \|\theta_{t}-\theta_{t-1}\|^2 \leq \E[\Omega_{t-1} - \Omega_{t}] + \frac{\eta}{\breve{\tau}T^{2\gamma-2}} + \frac{c^2\eta^2\sigma^2}{2L}.
 \end{align}

By multiplying both sides of the inequality~(\ref{eq:F11}) by $\frac{32}{\eta\breve{\tau}}$, we can obtain
\begin{align} \label{eq:F12}
     \frac{1}{\breve{\tau}^2}\E \|\nabla F(\theta_{t-1})-m_{t}\|^2 + \frac{1}{\eta^2} \E \|\theta_{t}-\theta_{t-1}\|^2
    & \leq \frac{4}{\breve{\tau}^2}\E \|\nabla F(\theta_{t-1})-m_{t}\|^2 + \frac{1}{\eta^2} \E \|\theta_{t}-\theta_{t-1}\|^2 \nonumber \\
    & \leq \E[\frac{32(\Omega_{t-1} - \Omega_{t})}{\eta\breve{\tau}}] + \frac{32}{\breve{\tau}^2 T^{2\gamma-2}} + \frac{16c^2\eta\sigma^2}{L\breve{\tau}}.
 \end{align}

Since $m_1 = \beta_{1}m_{0}+ (1-\beta_{1})g_1=(1-\beta_{1})\nabla f(\theta_{0};z_{1})$, we have
\begin{align}
 \|\nabla F(\theta_0) - m_1\|^2 & = \|\nabla F(\theta_0) - (1-\beta_{1})\nabla f(\theta_{0};z_{1})\|^2 \nonumber \\
 & =\|\nabla F(\theta_0) - \nabla f(\theta_{0};z_{1}) + \beta_1\nabla f(\theta_{0};z_{1})\|^2  \leq 2\sigma^2 + 2\beta_1^2G^2.
\end{align}
Since $\Omega_0 = F(\theta_0) + \frac{1}{2L}\|\nabla F(\theta_{0}) - m_1\|^2$, we have
\begin{align}
    & \frac{1}{T} \sum_{t=1}^T [\frac{1}{\breve{\tau}^2}\E \|\nabla F(\theta_{t-1})-m_{t}\|^2 + \frac{1}{\eta^2} \E \|\theta_{t}-\theta_{t-1}\|^2] \nonumber \\
    & \leq \frac{1}{T} \sum_{t=1}^T\E[\frac{32(\Omega_{t-1} - \Omega_{t})}{\eta\breve{\tau}}] + \frac{32}{\breve{\tau}^2T^{2\gamma-2}} + \frac{16c^2\eta\sigma^2}{L\breve{\tau}} \nonumber \\
    &\leq \frac{32(F(\theta_0) + \frac{1}{2L}(2\sigma^2 + 2\beta_1^2G^2) - F^*)}{T\eta\breve{\tau}} + \frac{32}{\breve{\tau}^2T^{2\gamma-2}} + \frac{16c^2\eta\sigma^2}{L\breve{\tau}}.
 \end{align}
Let $\Delta = F(\theta_0) + \frac{1}{L}(\sigma^2 + \beta_1^2G^2) - F^*$,
we have
\begin{align}
    & \frac{1}{T} \sum_{t=1}^T [\frac{1}{\breve{\tau}^2}\E \|\nabla F(\theta_{t-1})-m_{t}\|^2 + \frac{1}{\eta^2} \E \|\theta_{t}-\theta_{t-1}\|^2] \nonumber \\
    &\leq \frac{32\Delta}{T\eta\breve{\tau}} + \frac{32}{\breve{\tau}^2T^{2\gamma-2}} + \frac{16c^2\sigma^2\eta}{L\breve{\tau}}.
 \end{align}

According to the Jensen’s inequality, we can obtain
\begin{align}  \label{eq:F13}
 & \frac{1}{T}\sum_{t=1}^T\mathbb{E}[\frac{1}{\breve{\tau}} \|\nabla F(\theta_{t-1})-m_{t}\| + \frac{1}{\eta}  \|\theta_{t}-\theta_{t-1}\| ] \nonumber \\
 & \leq  \Big(\frac{2}{T}\sum_{t=1}^T\mathbb{E}[\frac{1}{\breve{\tau}^2} \|\nabla F(\theta_{t-1})-m_{t}\|^2 + \frac{1}{\eta^2} \|\theta_{t}-\theta_{t-1}\|^2]\Big)^{1/2} \nonumber \\
 & \leq  \sqrt{\frac{64\Delta}{T\eta\breve{\tau}} + \frac{64}{\breve{\tau}^2T^{2\gamma-2}} + \frac{32c^2\eta\sigma^2}{L\breve{\tau}}}
  \leq \frac{8\sqrt{\Delta}}{\sqrt{T\eta\breve{\tau}}} + \frac{8}{\breve{\tau}T^{\gamma-1}} +
  \frac{4\sqrt{2\eta}c\sigma}{\sqrt{L\breve{\tau}}}.
\end{align}
By using $\theta_{t} = \theta_{t-1} - \eta (\hat{H}_t m_t + \lambda\theta_{t-1} )$, we have
\begin{align} \label{eq:F14}
 \frac{1}{\breve{\tau}}\|\nabla F(\theta_{t-1})-m_{t}\|  + \frac{1}{\eta}\|\theta_{t}-\theta_{t-1}\|
 & =  \frac{1}{\breve{\tau}}\|\nabla F(\theta_{t-1})-m_{t}\|  + \frac{1}{\eta}\|\eta (\hat{H}_t m_t + \lambda\theta_{t-1} )\| \nonumber \\
 & \geq \frac{1}{\breve{\tau}}\|\nabla F(\theta_{t-1})-m_{t}\|  + \|\hat{H}_t m_t\| - \|\lambda\theta_{t-1}\|  \nonumber \\
 & =\frac{1}{\breve{\tau}}\|\hat{H}_t^{-1}\hat{H}_t(\nabla F(\theta_{t-1})-m_{t})\|  + \|\hat{H}_t m_t\| - \|\lambda\theta_{t-1}\| \nonumber \\
 & \geq \|\hat{H}_t(\nabla F(\theta_{t-1})-m_{t})\|  + \|\hat{H}_t m_t\| - \|\lambda\theta_{t-1}\| \nonumber \\
 & \geq \|\hat{H}_t\nabla F(\theta_{t-1})\| - \|\lambda\theta_{t-1}\| \nonumber \\
 & \geq \frac{\|\nabla F(\theta_{t-1})\|}{\breve{G}}- \|\lambda\theta_{t-1}\|,
\end{align}
where the above inequality holds by $\breve{\tau}I_d \preceq \hat{H}_t^{-1} \preceq \breve{G} I_d$
and $\frac{1}{\breve{G}}I_d \preceq \hat{H}_t \preceq \frac{1}{\breve{\tau}}I_d$.

By putting the above inequalities~(\ref{eq:F14}) into~(\ref{eq:F13}), we can obtain
\begin{align}
 \frac{1}{T}\sum_{t=1}^T\mathbb{E}\|\nabla F(\theta_{t-1})\|
 & \leq  \frac{\breve{G}}{T}\sum_{t=1}^T\mathbb{E}[\frac{1}{\breve{\tau}}\|\nabla F(\theta_{t-1})-m_{t}\|  + \frac{1}{\eta}\|\theta_{t}-\theta_{t-1}\|+\|\lambda\theta_{t-1}\| ] \nonumber \\
 & \leq \breve{G}\big(\frac{8\sqrt{\Delta}}{\sqrt{T\eta\breve{\tau}}} + \frac{8}{\breve{\tau}T^{\gamma-1}} +
  \frac{4\sqrt{2\eta}c\sigma}{\sqrt{L\breve{\tau}}}\big) + \frac{\breve{G}}{T}\sum_{t=1}^T\E\|\lambda\theta_{t-1}\| \nonumber \\
 & \mathop{\leq}^{(i)} \breve{G}\big(\frac{8\sqrt{\Delta}}{\sqrt{T\eta\breve{\tau}}} + \frac{8}{\breve{\tau}T^{\gamma-1}} +
  \frac{4\sqrt{2\eta}c\sigma}{\sqrt{L\breve{\tau}}}\big) + \frac{\breve{G}}{T}\sum_{t=1}^T\lambda t\eta\tilde{G} \nonumber \\
 & \mathop{\leq}^{(ii)} \breve{G}\big(\frac{8\sqrt{\Delta}}{\sqrt{T\eta\breve{\tau}}} + \frac{8}{\breve{\tau}T^{\gamma-1}} +
  \frac{4\sqrt{2\eta}c\sigma}{\sqrt{L\breve{\tau}}}\big) + \frac{\breve{G}}{T}\sum_{t=1}^T \frac{1}{\eta T^\gamma \tilde{G}\hat{G}} t \eta\tilde{G} \nonumber \\
 & \leq \breve{G}\big(\frac{8\sqrt{\Delta}}{\sqrt{T\eta\breve{\tau}}} + \frac{8}{\breve{\tau}T^{\gamma-1}} +
  \frac{4\sqrt{2\eta}c\sigma}{\sqrt{L\breve{\tau}}}\big)+ \frac{\breve{G}}{\hat{G}T^{\gamma-1}} ,
\end{align}
where the above inequality $(i)$ holds by $\|\theta_t\|\leq (t+1) \eta\tilde{G}$, and the above inequality $(ii)$
is due to $\lambda \leq \frac{1}{\eta T^\gamma \tilde{G}\hat{G}}$.
Thus, we have
\begin{align}
 \frac{1}{T+1}\sum_{t=0}^T\mathbb{E}\|\nabla F(\theta_{t})\|
  & \leq \breve{G}\big(\frac{8\sqrt{\Delta}}{\sqrt{{T+1}\eta\breve{\tau}}} + \frac{8}{\breve{\tau}(T+1)^{\gamma-1}} +
  \frac{4\sqrt{2\eta}c\sigma}{\sqrt{L\breve{\tau}}}\big)+ \frac{\breve{G}}{\hat{G}(T+1)^{\gamma-1}} \nonumber \\
  & \leq \frac{8\sqrt{\Delta}\breve{G}}{\sqrt{{T}\eta\breve{\tau}}}
  + \frac{8\breve{G}}{\breve{\tau}T^{\gamma-1}} +
  \frac{4\sqrt{2\eta}c\sigma\breve{G}}{\sqrt{L\breve{\tau}}}+ \frac{\breve{G}}{\hat{G}T^{\gamma-1}}.
\end{align}

Let $\eta=\frac{1}{\sqrt{T}}$ and $\gamma=\frac{3}{4}$, we have
\begin{align}
  \frac{1}{T}\sum_{t=0}^T\mathbb{E}\|\nabla F(\theta_{t})\| \leq \breve{G}\big(\frac{8\sqrt{\Delta}}{\sqrt{\breve{\tau}}T^{1/4}} + \frac{8}{\breve{\tau}T^{1/4}} +
  \frac{4\sqrt{2}c\sigma}{\sqrt{L\breve{\tau}}T^{1/4}}+ \frac{1}{\hat{G}T^{1/4}}\big).
\end{align}
Since $\eta=\frac{1}{\sqrt{T}}$ and set $c=O(1)$, we have $\beta_1=1-c\eta=O(1)$. Set $\tau=O(1)$, we have $\breve{\tau}=\min(1-\beta_1,\hat{\tau})=\min(1-\beta_1,(1-\beta_1)(\tau + (1-\beta_2)\varepsilon))=O(1)$.
Further let $G=O(1)$ and $L=O(1)$, we can obtain
\begin{align}
  \frac{1}{T+1}\sum_{t=0}^T\mathbb{E}\|\nabla F(\theta_{t})\| \leq O(\frac{1}{T^{1/4}}).
\end{align}

\end{proof}

\section{Detailed Experimental Setting}
\label{ap:ex}

\subsection{CV Task}
In the experiment, we set the mini-batch size be 64 for all algorithms.

When training VGG16 at CIFAR-10 dataset, we set the learning rate $10^{-4}$ for SGD and SGDM, and set
momentum parameter $\beta=0.9$ for SGDM. Adam, AdamW, AdaBelief and MiAdam
use the basic learning rate $10^{-6}$, the tuning parameter $\varepsilon=10^{-8}$, the first-order momentum parameter $\beta_1=0.9$, and
the second-order momentum parameter $\beta_2=0.99$. Meanwhile, AdamW uses the weight decay parameter
$\lambda=10^{-5}$, and MiAdam uses the multiple integration rate $\kappa=0.9$.
SWATS uses the basic learning rate $10^{-5}$, the tuning parameter $\varepsilon=10^{-8}$, $\beta_1=0.9$, and $\beta_2=0.99$.
Our Adam(W)-srf and HomeAdam(W) use the basic learning rate $10^{-6}$, $\varepsilon=10^{-7}$, $\beta_1=0.9$, and
$\beta_2=0.99$. Meanwhile, our AdamW-srf and HomeAdamW use the weight decay parameter $\lambda=10^{-5}$.
%We use the thresholds $\tau=10^{-12}$ and $\tau=10^{-13}$ at our HomeAdam and
%HomeAdamW, respectively (the smallest element of second-order momentum $\rho=\min_{1\leq t\leq T}\rho_t=10^{-42}$ in these algorithms).

When training ResNet34 at Tiny-ImageNet dataset, we set the learning rate $4\times10^{-4}$ for SGD and SGDM,
and set
$\beta=0.9$ for SGDM. Adam, AdamW, AdaBelief and MiAdam
use the basic learning rate $10^{-6}$, $\varepsilon=10^{-8}$, $\beta_1=0.9$, and
 $\beta_2=0.99$. Meanwhile, AdamW uses the weight decay parameter
$\lambda=10^{-4}$, and MiAdam uses the multiple integration rate $\kappa=0.999$.
SWATS uses the basic learning rate $10^{-5}$, $\varepsilon=10^{-8}$, $\beta_1=0.9$ and $\beta_2=0.99$.
Our Adam(W)-srf and HomeAdam(W) use the basic learning rate $10^{-6}$, $\varepsilon=10^{-7}$, $\beta_1=0.9$, and
$\beta_2=0.99$. Meanwhile, our AdamW-srf and HomeAdamW use the weight decay parameter $\lambda=10^{-4}$.
%In our HomeAdam(W), we use the threshold $\tau=10^{-8}$ (the smallest element of second-order momentum $\rho=\min_{1\leq t\leq T}\rho_t=10^{-41}$).

\subsection{NLP Task}

The first language model is modeled as a 8-layer Transformer~\citep{vaswani2017attention} encoder with 768-dimensional embeddings and 8 attention heads per layer, which employs a feed-forward network dimension of 1024 and uses sinusoidal positional encodings. Meanwhile, it uses a dropout rate of 0.1 throughout the network. The final output layer projects the representations back to vocabulary size for token prediction.

The second language model is modeled as a 24-layer Transformer encoder with 768-dimensional embeddings and 8 attention heads per layer, which employs a feed-forward network dimension of 2048 and uses sinusoidal positional encodings. Meanwhile, it uses a dropout rate of 0.15 throughout the network. The final output layer projects the representations back to vocabulary size for token prediction.

When training 8-layer Transformer model at WikiText2 dataset, we set the minibatch size be 32 for all algorithms.
We set the learning rate $2\times10^{-5}$ for SGD and SGDM,
and set momentum parameter $\beta=0.9$ for SGDM. Adam, AdamW and AdaBelief
use the basic learning rate $10^{-6}$, the tuning parameter $\varepsilon=10^{-8}$,
the first-order momentum parameter $\beta_1=0.9$, and
the second-order momentum parameter $\beta_2=0.999$. Meanwhile, AdamW uses the weight decay parameter
$\lambda=10^{-4}$, and MiAdam uses the basic learning rate $10^{-6}$, $\varepsilon=10^{-8}$, $\beta_1=0.9$, $\beta_2=0.99$ and the multiple integration rate $\kappa=0.999$.
SWATS uses the basic learning rate $10^{-5}$, the tuning parameter $\varepsilon=10^{-8}$, $\beta_1=0.9$,
and $\beta_2=0.99$.
Our Adam(W)-srf and HomeAdam(W) use the basic learning rate $10^{-6}$, $\varepsilon=10^{-5}$, $\beta_1=0.9$, and
$\beta_2=0.999$. Meanwhile, our AdamW-srf and HomeAdamW use the weight decay parameter $\lambda=10^{-4}$.
%Our HomeAdam and HomeAdamW  use the threshold $\tau=10^{-16}$ (the smallest element of second-order momentum $\rho=\min_{1\leq t\leq T}\rho_t=10^{-29}$ in these algorithms).

When training 24-layer Transformer model at WikiText103 dataset,
we set the minibatch size be 10 for all algorithms.
We set the learning rate $2\times10^{-5}$ for SGD and SGDM, and set
momentum parameter $\beta=0.9$ for SGDM. Adam, AdamW and AdaBelief
use the basic learning rate $10^{-6}$, the tuning parameter $\varepsilon=10^{-8}$, the first-order momentum parameter $\beta_1=0.9$, and
the second-order momentum parameter $\beta_2=0.999$. Meanwhile, AdamW uses the weight decay parameter
$\lambda=10^{-4}$, and MiAdam uses the basic learning rate $10^{-6}$, $\varepsilon=10^{-8}$, $\beta_1=0.9$, $\beta_2=0.999$ and the multiple integration rate $\kappa=0.999$.
SWATS uses the basic learning rate $10^{-5}$, the tuning parameter $\varepsilon=10^{-8}$, $\beta_1=0.9$, and $\beta_2=0.99$.
Our Adam(W)-srf and HomeAdam(W) use the basic learning rate $10^{-6}$, $\varepsilon=10^{-5}$, $\beta_1=0.9$, and
$\beta_2=0.999$. Meanwhile, our AdamW-srf and HomeAdamW use the weight decay parameter $\lambda=10^{-4}$.
%Our HomeAdam and HomeAdamW  use the threshold $\tau=10^{-16}$ (the smallest element of second-order momentum $\rho=\min_{1\leq t\leq T}\rho_t=10^{-41}$ in these algorithms).

\section{ Element-Wise Variant of HomeAdam and HomeAdamW Algorithms}
\label{ap:ew}
In this section, we provide an element-wise variant of our HomeAdam and HomeAdamW (HomeAdam-ew and HomeAdamW-ew) algorithms,
which is more suitable for training deep learning models due to matching the back-propagation framework.
Algorithm~\ref{alg:3} shows the algorithmic framework of our HomeAdam-ew and HomeAdamW-ew algorithms.

\subsection{ Generalization Analysis of our HomeAdam(W)-ew Algorithms }
In this subsection, we prove that our HomeAdam-ew and HomeAdamW-ew optimizers also has a smaller generalization error of $O(\frac{1}{N})$.

In the theoretical analysis, we first define a useful gradient mapping $\M(\cdot,\cdot)$ as follows:
\begin{align}
 \M((\hat{m}_t)_j,(\hat{v}_t)_j) = \begin{cases}
 \frac{(\hat{m}_t)_j}{(\hat{v}_t)_j + \varepsilon}, & \ \mbox{if} \ (\hat{v}_t)_j \geq \tau \\
 (\hat{m}_t)_j, & \mbox{otherwise}
 \end{cases}
\end{align}
Then we can rewrite the lines 12 and 14 of Algorithm~\ref{alg:3} as follows:
\begin{align}
 (\theta_{t})_j = (\theta_{t-1})_j - \eta (\M((\hat{m}_t)_j,(\hat{v}_t)_j) +\lambda(\theta_{t-1})_j).
\end{align}

\begin{algorithm}
    \caption{ \textbf{Element-Wise} HomeAdam(W) Algorithms }
    \label{alg:3}
    \begin{algorithmic}[1]
        \STATE \textbf{Input}: $\eta>0$, $\beta_1,\beta_2\in (0,1)$, $\varepsilon\geq 0$, $\lambda\geq 0$ and $\tau>0$;
        \STATE \textbf{Initialize:} $\theta_0\in \R^{d}$,
         $m_0=0$ and $v_0=0$;
        \FOR{$t = 1, 2, \ldots, T$}
            \STATE  Draw a sample $z_{t} \sim \mathcal{D}$;
            \FOR {$j = 1, 2, \ldots, d \ $ (Parallel Execution in Each Element or Layer)}
            \STATE $(g_t)_j = (\nabla f(\theta_{t-1};z_{t}))_j$;
            \STATE $(m_{t})_j= \beta_{1}(m_{t-1})_j+ (1-\beta_{1})(g_t)_j$;
            \STATE $(v_{t})_j= \beta_{2}(v_{t-1})_j+ (1-\beta_{2})(g_t)_j^2$;
            \STATE $(\hat{m}_t)_j=\frac{(m_t)_j}{1-\beta_1^t}$;
            \STATE $(\hat{v}_t)_j=\frac{(v_t)_j}{1-\beta_2^t}$;
            \IF {$(\hat{v}_t)_j \geq \tau$}
            \STATE $(\theta_{t})_j = (\theta_{t-1})_j - \eta \big(\frac{(\hat{m}_t)_j}{{\color{red}{(\hat{v}_t})_j} + \varepsilon}+
            \lambda(\theta_{t-1})_j\big)$;
            \ELSE
            \STATE $(\theta_{t})_j = (\theta_{t-1})_j - \eta ((\hat{m}_t)_j+\lambda(\theta_{t-1})_j)$.
            \ENDIF
        \ENDFOR
        \ENDFOR
        \STATE \textbf{Output:} $\theta_{T}$.
    \end{algorithmic}
\end{algorithm}

\begin{lemma} \label{lem:F1}
Assume the sequences $\{\hat{m}_t,\hat{v}_t,m_t,v_t\}_{t=1}^T$ and $\{\hat{m}_t^{(i)},\hat{v}_t^{(i)},m_t^{(i)},v_t^{(i)}\}_{t=1}^T$ are generated from Algorithm~\ref{alg:3} based on the dataset $S$ and
$S^{(i)}$, respectively. Without loss of generality, let $\tau \geq 1$, we have
\begin{align}
\big\|\M(\hat{m}_t,\hat{v}_t) -  \M(\hat{m}_t^{(i)},\hat{v}_t^{(i)})\| \leq
\frac{\sqrt{d}}{(1-\beta_1^t)}\big\|m_t - m_t^{(i)}\big\| + \frac{G\sqrt{d}}{(1-\beta_1^t)(1-\beta_2^t)}\big\| v_t -v_t^{(i)}\big\|.
\end{align}
\end{lemma}

\begin{proof}
From Algorithm~\ref{alg:3}, since $S, S^{(i)}\sim \mathcal{D}$, we have $(m_{t})_j= \beta_{1}(m_{t-1})_j+ (1-\beta_{1})(g_t)_j$, $(v_{t})_j= \beta_{2}(v_{t-1})_j+ (1-\beta_{2})(g_t)_j^2$, $(m_{t}^{(i)})_j= \beta_{1}(m_{t-1}^{(i)})_j+ (1-\beta_{1})(g_t^{(i)})_j$ and $(v_{t}^{(i)})_j= \beta_{2}(v_{t-1}^{(i)})_j+ (1-\beta_{2})(g_t^{(i)})_j^2$. Meanwhile, we have $(\hat{m}_{t})_j=\frac{(m_{t})_j}{1-\beta_1^t}$, $(\hat{v}_{t})_j=\frac{(v_{t})_j}{1-\beta_2^t}$, $(\hat{m}_{t}^{(i)})_j=\frac{(m_{t}^{(i)})_j}{1-\beta_1^t}$ and $(\hat{v}_{t}^{(i)})_j=\frac{(v_{t}^{(i)})_j}{1-\beta_2^t}$.

Since $m_t^{(i)}$ is exponential moving average of $g_t^{(i)}$ and $\|g_t^{(i)}\|\leq G$, , we have
$(\hat{m}_t^{(i)})_j\leq \|\hat{m}_t^{(i)}\|\leq \frac{G}{1-\beta_1^t}$.
From Algorithm~\ref{alg:3}, when $(\hat{v}_t)_j \geq \tau$ and $(\hat{v}_t^{(i)})_j \geq \tau$, we have
\begin{align} \label{eq:H1}
 & \big| \frac{(\hat{m}_t)_j}{(\hat{v}_t)_j + \varepsilon} -  \frac{(\hat{m}_t^{(i)})_j}{(\hat{v}_t^{(i)})_j + \varepsilon}\big| \nonumber \\
 & = \frac{1}{((\hat{v}_t)_j + \varepsilon)((\hat{v}_t^{(i)})_j + \varepsilon)}\big|((\hat{v}_t^{(i)})_j + \varepsilon)(\hat{m}_t)_j -  (\hat{m}_t^{(i)})_j((\hat{v}_t)_j + \varepsilon)\big| \nonumber \\
 & \leq \frac{1}{((\hat{v}_t)_j + \varepsilon)((\hat{v}_t^{(i)})_j + \varepsilon)}\Big(\big|((\hat{v}_t^{(i)})_j + \varepsilon)(\hat{m}_t)_j - (\hat{m}_t^{(i)})_j((\hat{v}_t^{(i)})_j +\varepsilon) \big| \nonumber \\
 & \qquad+\big|(\hat{m}_t^{(i)})_j((\hat{v}_t^{(i)})_j + \varepsilon) -  (\hat{m}_t^{(i)})_j((\hat{v}_t)_j + \varepsilon)\big| \Big)\nonumber \\
 & = \frac{1}{(\hat{v}_t)_j + \varepsilon}\big|(\hat{m}_t)_j - (\hat{m}_t^{(i)})_j\big| + \frac{(\hat{m}_t^{(i)})_j}{((\hat{v}_t)_j + \varepsilon)((\hat{v}_t^{(i)})_j + \varepsilon)}\big|(\hat{v}_t^{(i)})_j - (\hat{v}_t)_j\big| \nonumber \\
 & \leq \frac{1}{\tau+\varepsilon}\big|(\hat{m}_t)_j - (\hat{m}_t^{(i)})_j\big| + \frac{G}{(1-\beta_1^t)(\tau+\varepsilon)^2}\big|(\hat{v}_t^{(i)})_j - (\hat{v}_t)_j\big| \nonumber \\
 & \leq \frac{1}{(1-\beta_1^t)}\big|(m_t)_j - (m_t^{(i)})_j\big| + \frac{G}{(1-\beta_1^t)(1-\beta_2^t)}\big|(v_t^{(i)})_j - (v_t)_j\big|,
\end{align}
where the last inequality is due to $\tau+\varepsilon\geq 1$.
From Algorithm~\ref{alg:3}, when $(\hat{v}_t)_j < \tau$ and $(\hat{v}_t^{(i)})_j < \tau$, we have
\begin{align} \label{eq:H2}
 \big| (\hat{m}_t)_j -  (\hat{m}_t^{(i)})_j\big|  \leq \frac{1}{(1-\beta_1^t)}\big|(m_t)_j - (m_t^{(i)})_j\big| + \frac{G}{(1-\beta_1^t)(1-\beta_2^t)}\big|(v_t^{(i)})_j - (v_t)_j\big|.
\end{align}

According to the above inequalities~(\ref{eq:H1}) and (\ref{eq:H2}),
then we have
\begin{align} \label{eq:H3}
 \big| \M((\hat{m}_t)_j,(\hat{v}_t)_j) -  \M((\hat{m}_t^{(i)})_j,(\hat{v}_t^{(i)})_j)\big|  \leq \frac{1}{(1-\beta_1^t)}\big|(m_t)_j - (m_t^{(i)})_j\big| + \frac{G}{(1-\beta_1^t)(1-\beta_2^t)}\big|(v_t^{(i)})_j - (v_t)_j\big|.
\end{align}

Thus we can obtain
\begin{align}
 \big\|\M(\hat{m}_t,\hat{v}_t) -  \M(\hat{m}_t^{(i)},\hat{v}_t^{(i)})\big\|_1 \leq
 \frac{1}{(1-\beta_1^t)}\big\|m_t - m_t^{(i)}\big\|_1 + \frac{G}{(1-\beta_1^t)(1-\beta_2^t)}\big\| v_t -v_t^{(i)}\big\|_1.
\end{align}
Since $\|\cdot\| \leq \|\cdot\|_1 \leq \sqrt{d}\|\cdot\|$, we have
\begin{align}
\big\|\M(\hat{m}_t,\hat{v}_t) -  \M(\hat{m}_t^{(i)},\hat{v}_t^{(i)})\| & \leq \big\|\M(\hat{m}_t,\hat{v}_t) -  \M(\hat{m}_t^{(i)},\hat{v}_t^{(i)})\big\|_1 \nonumber \\
 &\leq
 \frac{1}{(1-\beta_1^t)}\big\|m_t - m_t^{(i)}\big\|_1 + \frac{G}{(1-\beta_1^t)(1-\beta_2^t)}\big\| v_t -v_t^{(i)}\big\|_1 \nonumber \\
 & \leq
 \frac{\sqrt{d}}{(1-\beta_1^t)}\big\|m_t - m_t^{(i)}\big\| + \frac{G\sqrt{d}}{(1-\beta_1^t)(1-\beta_2^t)}\big\| v_t -v_t^{(i)}\big\|.
\end{align}

\end{proof}

\begin{theorem}
Assume the sequence $\{\theta_t\}_{t=1}^T$ is generated
from Algorithm~\ref{alg:3} on dataset $S=\{z_1,z_2,\cdots,z_N\}$. Under the Assumptions~\ref{ass:s1},~\ref{ass:g},~\ref{ass:v}, without loss of generality, let $\tau\geq 1$, $\lambda\in [0,\frac{1}{\eta})$, $\beta_1=O(1)$ with $\beta_1\in (0,1)$, $\beta_2=O(1)$ with $\beta_2\in (0,1)$, $\sigma=O(1)$, $G=O(1)$ and $L=O(1)$.
 If the iteration number is small (i.e., $T=O(1)$) set
 $\eta=\frac{1}{\sqrt{d}}$, otherwise set $\eta=\frac{1}{\sqrt{d}T}$, we have
\begin{align}
|\E [ F(\theta_T) - F_S(\theta_T)]| \leq O(\frac{1}{N}).
\end{align}
\end{theorem}

\begin{proof}
Implementing Algorithm~\ref{alg:3} on datasets $S$ and $S^{(i)}$ with
 the same random index sequence $\{j_t\}_{t=1}^T$, and let $\{\theta_t\}_{t=1}^T$ and $\{\theta_t^{(i)}\}_{t=1}^T$ be generated from Algorithm~\ref{alg:3} with $S$ and $S^{(i)}$. Without loss of generality, let $\tau \geq 1$ in Algorithm~\ref{alg:3}.

 According to the above gradient mapping $\M(\cdot,\cdot)$, we have $(\theta_{t})_j = (\theta_{t-1})_j - \eta \big(\M((\hat{m}_t)_j,(\hat{v}_t)_j) +\lambda(\theta_{t-1})_j \big)$ and $(\theta_{t}^{(i)})_j = (\theta_{t-1}^{(i)})_j - \eta \big(\M((\hat{m}_t^{(i)})_j,(\hat{v}_t^{(i)})_j) +\lambda(\theta_{t-1}^{(i)})_j \big)$ for all $j=1,2,\cdots,d$.
 Then we have
\begin{align}
 \theta_{t} - \theta_{t}^{(i)}  = (1-\eta\lambda)(\theta_{t-1} -\theta_{t-1}^{(i)}) - \eta \big(\M(\hat{m}_t,\hat{v}_t) -  \M(\hat{m}_t^{(i)},\hat{v}_t^{(i)})\big).
\end{align}
We can obtain
 \begin{align} \label{eq:H4}
 \|\theta_{t} - \theta_{t}^{(i)}\| & =\|(1-\eta\lambda)(\theta_{t-1} -\theta_{t-1}^{(i)}) - \eta \big(\M(\hat{m}_t,\hat{v}_t) -  \M(\hat{m}_t^{(i)},\hat{v}_t^{(i)})\big)\| \nonumber \\
 & \leq (1-\eta\lambda)\|\theta_{t-1} -\theta_{t-1}^{(i)}\| + \eta \|\M(\hat{m}_t,\hat{v}_t) -  \M(\hat{m}_t^{(i)},\hat{v}_t^{(i)})\| \nonumber \\
 & \leq (1-\eta\lambda)\|\theta_{t-1} -\theta_{t-1}^{(i)}\| +\frac{\eta\sqrt{d}}{(1-\beta_1^t)}\big\|m_t - m_t^{(i)}\big\|  + \frac{\eta G\sqrt{d}}{(1-\beta_1^t)(1-\beta_2^t)}\big\| v_t -v_t^{(i)}\big\|,
 \end{align}
 where the last inequality holds by Lemma~\ref{lem:F1}.

According to the above inequality~(\ref{eq:H4}), following the above proof of Theorem~\ref{th:1}, let $\eta = O(\frac{1}{\sqrt{d}})$, $\beta_1=O(1)$, $\beta_2=O(1)$,
$\sigma=O(1)$, $G=O(1)$ and $L=O(1)$, we have
\begin{align}
 & \phi_1=2(1-\beta_1)\sigma = O(1), \quad \psi_1 = 2(1-\beta_2)G^2=O(1),  \quad \varphi_1 = 2\eta\sqrt{d}\sigma+\frac{2\eta\sqrt{d}G^3}{(1-\beta_1)} = O(1), \nonumber \\
 &  \E \|m_1 - m_1^{(i)}\| \leq \frac{\phi_1}{N}, \quad   \E \|v_1 - v_1^{(i)}\| \leq \frac{\psi_1}{N}, \quad
  \E \|\theta_1 - \theta_1^{(i)}\|  \leq  \frac{\varphi_1}{N} =O(\frac{1}{N}).
  \end{align}

\textbf{If} the iteration number is small (i.e., $T=O(1)$),
let $\eta=\frac{1}{\sqrt{d}}$, $\lambda\in [0,\frac{1}{\eta})$, $\beta_1=O(1)$ with $\beta_1\in (0,1)$, $\beta_2=O(1)$ with $\beta_2\in (0,1)$, $\sigma=O(1)$, $G=O(1)$ and $L=O(1)$.
Following the above proof of Theorem~\ref{th:1}, assume $\E \|m_t - m_t^{(i)}\| \leq \frac{\phi_t}{N}$ with
$\phi_t=O(1)$, and $\E\|v_t-v_t^{(i)}\|\leq \frac{\psi_t}{N}$ with $\psi_t=O(1)$, we can obtain $\phi_{t+1} =O(1)$ and $\psi_{t+1}=O(1)$.
By using the above inequality~(\ref{eq:H4}), then we have
\begin{align}
 \E \|\theta_{t+1} - \theta_{t+1}^{(i)}\| \leq \frac{\varphi_{t+1}}{N} =O(\frac{1}{N}),
\end{align}
where $\varphi_{t+1}=(1-\eta\lambda)\varphi_t + \frac{\eta\sqrt{d}\phi_{t+1}}{(1-\beta_1^{t+1})}+  \frac{\eta G\sqrt{d}\psi_{t+1}}{(1-\beta_1^{t+1})(1-\beta_2^{t+1})}=O(1)$.
By using the mathematical induction, we have
\begin{align}
 \E \|\theta_{T} - \theta_{T}^{(i)}\| &  \leq  \frac{\varphi_T}{N} = O(\frac{1}{N}).
\end{align}

\textbf{If} the iteration number $T$ is large, we consider the iteration number $t\geq1$
in the generalization analysis. By using the mathematical induction, due to recursion of the above inequality~(\ref{eq:H4}), following the above proof of Theorem~\ref{th:1}, we assume $\E \|m_t - m_t^{(i)}\| \leq \frac{\phi_t}{N}$ with
$\phi_t=O(t)$, and $\E\|v_t-v_t^{(i)}\|\leq \frac{\psi_t}{N}$ with $\psi_t=O(t)$.

Let $\eta=O(\frac{1}{\sqrt{d}})$, $\lambda\in [0,\frac{1}{\eta})$, $\beta_1=O(1)$ with $\beta_1\in (0,1)$, $\beta_2=O(1)$ with $\beta_2\in (0,1)$, $\sigma=O(1)$, $G=O(1)$ and
$L=O(1)$, we have
\begin{align}
 & \phi_{t+1}= \beta_1\phi_t + 2(1-\beta_1)\sigma + (1-\beta_1)L\varphi_t=O(t+1), \nonumber \\
 & \psi_{t+1}= \beta_2 \psi_t + 4(1-\beta_2)G\sigma + 2(1-\beta_2)GL\varphi_t=O(t+1). \nonumber
\end{align}
By using the above inequality~(\ref{eq:H4}), then we can obtain
\begin{align}
 \E \|\theta_{t+1} - \theta_{t+1}^{(i)}\| \leq \frac{\varphi_{t+1}}{N} =O(\frac{t+1}{N}).
\end{align}
By using the mathematical induction, we have
\begin{align}
 \E \|\theta_{T} - \theta_{T}^{(i)}\| &  \leq  \frac{\varphi_T}{N} = O(\frac{T}{N}).
\end{align}

Further let $\eta=O( \textcolor{blue}{\frac{1}{T\sqrt{d}}})$, $\lambda\in [0,\frac{1}{\eta})$, $\beta_1=O(1)$ with $\beta_1\in (0,1)$, $\beta_2=O(1)$ with $\beta_2\in (0,1)$, $\sigma=O(1)$, $G=O(1)$ and
$L=O(1)$, we have $\phi_T=O(1)$, $\psi_T=O(1)$, $\phi_T=O(1)$,
\begin{align}
 \E \|\theta_{T} - \theta_{T}^{(i)}\| &  \leq  \frac{\varphi_T}{N} = O(\frac{1}{N}).
\end{align}

By using Assumption~\ref{ass:g}, i.e., the condition of $G$-Lipschitz $f(\theta;z)$
for any $z\in \mathcal{D}$, then we have
\begin{align} \label{eq:f3}
 \E |f(\theta_T;z)-f(\theta_T^{(i)};z)| \leq G \E \|\theta_{T} - \theta_{T}^{(i)}\| =O(\frac{1}{N}).
\end{align}
By taking expectations over $S$, $S^{(i)}$ and the algorithm's randomness on
the above inequality~(\ref{eq:f3}), and according to the above lemma~\ref{lem:gs}, we
can obtain
\begin{align}
 |\E [ F(\theta_T) - F_S(\theta_T)]| \leq O(\frac{1}{N}).
\end{align}

\end{proof}

\subsection{ Convrgence Analysis of our HomeAdam(W)-ew Algorithms }
In this subsection, we prove that our HomeAdam-ew and HomeAdamW-ew optimizers also has a fast convergence rate of $O(\frac{1}{T^{1/4}})$.

\begin{lemma} \label{lem:F2}
Assume the sequence $\{m_t\}_{t=0}^T$ is generated from Algorithm~\ref{alg:3}, let $\beta_1 = 1-c\eta\in (0,1)$, $\beta_2\in (0,1)$, we have
\begin{align}
 \E\|\nabla F(\theta_{t}) - m_{t+1}\|^2   \leq (1-c\eta)\mathbb{E} \|\nabla F(\theta_{t-1}) - m_{t} \|^2 + \frac{2}{c \eta}L^2\mathbb{E}\|\theta_t-\theta_{t-1}\|^2  + c^2\eta^2\sigma^2,
\end{align}
where $c>0$.
\end{lemma}

\begin{proof}
 This proof can follows the above proof of Lemma~\ref{lem:A2}.
\end{proof}

\begin{theorem}
Assume the sequence $\{\theta_t\}_{t=0}^T$ is generated
from Algorithm~\ref{alg:3}. Under the Assumptions~\ref{ass:s2},~\ref{ass:g},~\ref{ass:v},~\ref{ass:f}, and let $0\leq \lambda < \min(\frac{1}{\eta},\frac{1}{\eta T^\gamma \sqrt{d}r \hat{G}})$, $\|\theta_0\|_{\infty} \leq \eta \breve{G}$, $c\geq \frac{16L}{\nu}$, $\beta_1 = 1-c\eta\in (0,1)$, $\beta_2\in (0,1)$ and $0<\eta \leq \frac{\nu}{4L}$, we have
\begin{align}
\frac{1}{T+1}\sum_{t=0}^T\mathbb{E}\|\nabla F(\theta_{t})\|
   \leq  r\big(\frac{4\sqrt{2\Delta}}{\sqrt{T\eta\nu}} + \frac{4\sqrt{2}}{T^{\gamma-1}\nu} +
  \frac{4c\sigma\sqrt{\eta}}{\sqrt{L\nu}}\big)+ \frac{1}{T^{\gamma-1}},
\end{align}
where $\hat{G}=\max(\frac{G}{1-\beta_1},\frac{G}{(1-\beta_1)(\tau+\varepsilon)})$,
$\breve{G}=\min(\frac{G}{1-\beta_1},\frac{G}{(1-\beta_1)(\tau+\varepsilon)})$, $r = \max(1,G^2 +\varepsilon)$, $\nu= \min(1-\beta_1,(\tau + \varepsilon)(1-\beta_1))$ and $\Delta = F(\theta_0) + \frac{1}{L}(\sigma^2 + \beta_1^2G^2) - F^*$.
\end{theorem}

\begin{proof}
From our Algorithm~\ref{alg:3},
since $(m_t)_j$ is exponential moving average of $(g_t)_j$ for all $j=1,2,\cdots,d$, by using Assumption~\ref{ass:g}, i.e., $|(g_t)_j|\leq \|g_t\|\leq G$, we have $|(m_t)_j|\leq \|m_t\|\leq G$.

Let $\hat{G}=\max(\frac{G}{1-\beta_1},\frac{G}{(1-\beta_1)(\tau+\varepsilon)})$ and
$\breve{G}=\min(\frac{G}{1-\beta_1},\frac{G}{(1-\beta_1)(\tau+\varepsilon)})$.

When $(\hat{v}_t)_j \geq \tau$, from the line 12 of Algorithm~\ref{alg:3}, we have $(\theta_{t})_j = (\theta_{t-1})_j - \eta \big(\frac{(\hat{m}_t)_j}{(\hat{v}_t)_j + \varepsilon}+
            \lambda(\theta_{t-1})_j\big)$, then we can obtain
 \begin{align} \label{eq:m1}
  |(\theta_t)_j| & = \big| (\theta_{t-1})_j - \eta \big(\frac{(\hat{m}_t)_j}{(\hat{v}_t)_j + \varepsilon}+
            \lambda(\theta_{t-1})_j\big) \big| \nonumber \\
  & = \big| (1-\eta\lambda)(\theta_{t-1})_j - \eta \frac{(\hat{m}_t)_j}{(\hat{v}_t)_j+\varepsilon} \big| \nonumber \\
  & \leq (1-\eta\lambda) \big|(\theta_{t-1})_j\big| + \eta \big| \frac{(\hat{m}_t)_j}{(\hat{v}_t)_j+\varepsilon}\big| \nonumber \\
  & \leq (1-\eta\lambda) \big|(\theta_{t-1})_j\big| + \frac{\eta G}{(1-\beta_1^t)(\tau+\varepsilon)} \nonumber \\
  & \leq (1-\eta\lambda)^{t}\big|(\theta_{0})_j\big| + \frac{t\eta G}{(1-\beta_1)(\tau+\varepsilon)} \nonumber \\
  & \leq  (t+1)\frac{\eta G}{(1-\beta_1)(\tau+\varepsilon)} \leq (t+1)\eta \hat{G},
 \end{align}
 where the first inequality is due to $0\leq \lambda <\frac{1}{\eta}$, and the second last inequality holds by $\|\theta_0\|_{\infty} \leq \eta \breve{G} \leq \frac{\eta G}{(1-\beta_1)(\tau+\varepsilon)}$.

When $(\hat{v}_t)_j < \tau$, from the line 14 of Algorithm~\ref{alg:3}, we have $(\theta_{t})_j = (\theta_{t-1})_j - \eta \big((\hat{m}_t)_j+
            \lambda(\theta_{t-1})_j\big)$, then we can obtain
 \begin{align} \label{eq:m2}
  |(\theta_t)_j| & = \big| (\theta_{t-1})_j - \eta \big((\hat{m}_t)_j +
            \lambda(\theta_{t-1})_j\big) \big| \nonumber \\
  & = \big| (1-\eta\lambda)(\theta_{t-1})_j - \eta(\hat{m}_t)_j \big| \nonumber \\
  & \leq (1-\eta\lambda) \big|(\theta_{t-1})_j\big| + \eta \big| (\hat{m}_t)_j\big| \nonumber \\
  & \leq (1-\eta\lambda) \big|(\theta_{t-1})_j\big| + \frac{\eta G}{(1-\beta_1^t)} \nonumber \\
  & \leq (1-\eta\lambda)^{t}\big|(\theta_{0})_j\big| + \frac{t\eta G}{1-\beta_1} \nonumber \\
  & \leq  (t+1)\frac{\eta G}{1-\beta_1} \leq (t+1)\eta \hat{G},
 \end{align}
 where the first inequality is due to $0\leq \lambda <\frac{1}{\eta}$, and the second last inequality holds by $\|\theta_0\|_{\infty}\leq \eta \breve{G} \leq \frac{\eta G}{1-\beta_1}$.

By using the above inequalities~(\ref{eq:m1}) and~(\ref{eq:m2}), thus we have for all $j=1,2,\cdots,d$
\begin{align}
  \big|(\theta_t)_j\big| \leq  (t+1)\eta \hat{G}.
\end{align}
Then we can obtain
\begin{align} \label{eq:m3}
  \|\theta_t\|_{\infty} \leq  (t+1)\eta \hat{G}, \ \|\theta_t\| \leq \sqrt{d}\|\theta_t\|_{\infty} \leq  (t+1)\sqrt{d}\eta \hat{G}.
\end{align}

Next, we define a diagonal matrix $H_t$ to unify the lines 12 and 14 of Algorithm~\ref{alg:3} as follows
\begin{align} \label{eq:m4}
 (\theta_{t})_j = (\theta_{t-1})_j - \eta \big((H_t)_{jj}(m_t)_j +
            \lambda(\theta_{t-1})_j\big),
\end{align}
where $(H_t)_{jj}=\frac{1}{((\hat{v}_t)_j + \varepsilon)(1-\beta_1^t)}$ when $(\hat{v}_t)_j \geq \tau$,
otherwise $(H_t)_{jj}=\frac{1}{1-\beta_1^t}$.
Then we can rewrite the above equality~(\ref{eq:m4}) in vector form
\begin{align} \label{eq:m5}
 \theta_{t} = \theta_{t-1} - \eta \big(H_tm_t +
            \lambda\theta_{t-1}\big).
\end{align}
By using Assumption~\ref{ass:g}, we have $|(g_t)^2_j|\leq G^2$. Since $(v_t)_j$ is exponential moving average of $(g_t)^2_j$, we have $(v_t)_j\leq G^2$ for all $j=1,2,\cdots,d$.
Further let $\frac{1}{r} = \min(1,\frac{1}{G^2 +\varepsilon})$ and $\frac{1}{\nu}= \max(\frac{1}{1-\beta_1},\frac{1}{(\tau + \varepsilon)(1-\beta_1)})$, we have
$\frac{1}{r}I_d \preceq H_t\preceq \frac{1}{\nu}I_d$.

Then we could further rewrite the above equality~(\ref{eq:m5}) as follows:
\begin{align} \label{eq:m6}
\theta_{t} & = \theta_{t-1} - \eta (H_t m_t + \lambda\theta_{t-1} )\nonumber\\
& =(1-\eta\lambda)\theta_{t-1} - \eta H_t m_t \nonumber \\
&=\arg\min_{\theta\in \R^d} \Big\{ \langle m_t,\theta\rangle + \frac{1}{2\eta}\big(\theta-(1-\lambda\eta)\theta_{t-1}\big)^TH_t^{-1}
\big(\theta-(1-\lambda\eta)\theta_{t-1}\big)\Big\} .
\end{align}
By using the optimality condition of the subproblem~(\ref{eq:m6}),
we have
\begin{align} \label{eq:m7}
\langle m_t +\frac{1}{\eta}H_t^{-1}\big(\theta_{t}-(1-\lambda\eta)\theta_{t-1}\big),\theta-\theta_t\rangle \geq 0, \quad \forall \theta\in \mathbb{R}^d.
\end{align}
By putting $\theta=\theta_{t-1}$ into the above inequality~(\ref{eq:m7}), we have
\begin{align} \label{eq:m8}
\langle m_t +\frac{1}{\eta}H_t^{-1}\big(\theta_{t}-(1-\lambda\eta)\theta_{t-1}\big),\theta_{t-1}-\theta_{t}\rangle
\geq 0.
\end{align}
Thus we can obtain
\begin{align} \label{eq:m9}
\langle m_t,\theta_{t-1}-\theta_{t}\rangle &\geq \frac{1}{\eta}\langle H_t^{-1}(\theta_{t}-\theta_{t-1}),\theta_{t}-\theta_{t-1}\rangle + \lambda\langle H_t^{-1}\theta_{t-1},\theta_t-\theta_{t-1} \rangle \nonumber \\
& \geq \frac{\nu}{\eta}\|\theta_{t}-\theta_{t-1}\|^2+ \lambda\langle H_t^{-1}\theta_{t-1},\theta_t-\theta_{t-1} \rangle,
\end{align}
where the last inequality holds by $\nu I_d \preceq H_t^{-1} \preceq r I_d$.

Since $\nu I_d \preceq H_t^{-1} \preceq r I_d$ and $\|\theta_{t-1}\| \leq \sqrt{d}\|\theta_{t-1}\|_{\infty} \leq t\eta \sqrt{d}\hat{G}$, we have $\|H_t^{-1}\theta_{t-1}\|^2 \leq r^2\|\theta_{t-1}\|^2 \leq t^2r^2\eta^2d\hat{G}^2$ for all $t\geq1$.

According to Assumption~\ref{ass:s2}, i.e., $F(\theta)$ is $L$-smooth, we have
\begin{align} \label{eq:m10}
  \mathbb{E} [F(\theta_{t})]
 & \leq  \mathbb{E} [F(\theta_{t-1}) + \nabla F(\theta_{t-1})^T(\theta_{t}-\theta_{t-1}) + \frac{L}{2}\|\theta_{t}-\theta_{t-1}\|^2] \nonumber \\
 & = \mathbb{E} [F(\theta_{t-1}) + (\nabla F(\theta_{t-1})-m_{t})^T(\theta_{t}-\theta_{t-1}) + m_{t}^T(\theta_{t}-\theta_{t-1})+ \frac{L}{2}\|\theta_{t}-\theta_{t-1}\|^2] \nonumber \\
 & \mathop{\leq}^{(i)} \mathbb{E} [F(\theta_{t-1}) + \frac{\eta}{2\nu}\|\nabla F(\theta_{t-1})-m_{t}\|^2 + \frac{\nu}{2\eta}\|\theta_{t}-\theta_{t-1}\|^2 + m_t^T(\theta_{t}-\theta_{t-1})+ \frac{L}{2}\|\theta_{t}-\theta_{t-1}\|^2] \nonumber \\
 & \mathop{\leq}^{(ii)}  \mathbb{E} [F(\theta_{t-1}) + \frac{\eta}{2\nu}\|\nabla F(\theta_{t-1})-m_{t}\|^2 + \frac{\nu}{2\eta}\|\theta_{t}-\theta_{t-1}\|^2 - \frac{\nu}{\eta_{t}}\|\theta_{t}-\theta_{t-1}\|^2 \nonumber \\
 & \quad -  \lambda\langle H_t^{-1}\theta_{t-1},\theta_t-\theta_{t-1} \rangle + \frac{L}{2}\|\theta_{t}-\theta_{t-1}\|^2] \nonumber \\
 & \leq \mathbb{E} [F(\theta_{t-1}) + \frac{\eta}{2\nu}\|\nabla F(\theta_{t-1})-m_{t}\|^2 + \frac{\nu}{2\eta}\|\theta_{t}-\theta_{t-1}\|^2 - \frac{\nu}{\eta}\|\theta_{t}-\theta_{t-1}\|^2 \nonumber \\
 & \quad + \frac{\lambda^2\eta}{\nu}\| H_t^{-1}\theta_{t-1}\|^2+ \frac{\nu}{4\eta}\|\theta_t-\theta_{t-1}\|^2 + \frac{L}{2}\|\theta_{t}-\theta_{t-1}\|^2] \nonumber \\
 & \leq \mathbb{E} [F(\theta_{t-1}) + \frac{\eta}{\nu T^{2\gamma-2}}+ \frac{\eta}{2\nu}\|\nabla F(\theta_{t-1})-m_{t}\|^2  - \frac{\nu}{8\eta}\|\theta_{t}-\theta_{t-1}\|^2],
\end{align}
where the above inequality $(i)$ holds by Young's inequality, the above inequality $(ii)$ follows by
 the above inequality~(\ref{eq:m9}), and the last inequality holds by $0< \eta \leq \frac{\nu}{4L}$
 and $0\leq \lambda \leq \frac{1}{\eta T^\gamma \sqrt{d}r \hat{G}}$.

Here we define a useful Lyapunov function $\Psi_t = F(\theta_t) + \frac{1}{2L}\|\nabla F(\theta_{t}) - m_{t+1}\|^2$.
Then we have
 \begin{align}
  \E[\Psi_{t} -\Psi_{t-1} ]  & = \E [F(\theta_{t})] - \E [F(\theta_{t-1}) ] + \frac{1}{2L} (\E \|\nabla F(\theta_{t}) - m_{t+1}\|^2 - \E \|\nabla F(\theta_{t-1}) - m_{t}\|^2 ) \nonumber \\
 & \leq \frac{\eta}{\nu T^{2\gamma-2}}+ \frac{\eta}{2\nu}\E \|\nabla F(\theta_{t-1})-m_{t}\|^2  - \frac{\nu}{8\eta} \E \|\theta_{t}-\theta_{t-1}\|^2 \nonumber \\
 & \quad + \frac{1}{2L} \big( -c\eta\mathbb{E} \|\nabla F(\theta_{t-1}) - m_{t} \|^2 + \frac{2}{c \eta}L^2\mathbb{E}\|\theta_t-\theta_{t-1}\|^2  + c^2\eta^2\sigma^2\big) \nonumber \\
 & \mathop{\leq}^{(i)}  \frac{\eta}{\nu T^{2\gamma-2}}- \frac{c\eta}{4L}\E \|\nabla F(\theta_{t-1})-m_{t}\|^2  - \frac{\nu}{16\eta} \E \|\theta_{t}-\theta_{t-1}\|^2 + \frac{c^2\eta^2\sigma^2}{2L} \nonumber \\
 & \leq \frac{\eta}{\nu T^{2\gamma-2}}- \frac{4\eta}{\nu}\E \|\nabla F(\theta_{t-1})-m_{t}\|^2  - \frac{\nu}{16\eta} \E \|\theta_{t}-\theta_{t-1}\|^2 + \frac{c^2\eta^2\sigma^2}{2L},
 \end{align}
where the first inequality holds by Lemma~\ref{lem:F2}, and the above inequality $(i)$ holds by $c\geq \frac{16L}{\nu}$ such as
$\frac{\nu}{16\eta}\geq \frac{L}{c\eta}$ and $ \frac{c\eta}{4L}\geq \frac{c\eta}{32L} \geq \frac{\eta}{2\nu}$, and
the last inequality also is due to $c\geq \frac{16L}{\nu}$.
Then we have
\begin{align} \label{eq:m11}
    \frac{4\eta}{\nu}\E \|\nabla F(\theta_{t-1})-m_{t}\|^2 + \frac{\nu}{16\eta} \E \|\theta_{t}-\theta_{t-1}\|^2 \leq \E[\Psi_{t-1} - \Psi_{t}] + \frac{\eta}{\nu T^{2\gamma-2}} + \frac{c^2\eta^2\sigma^2}{2L}.
 \end{align}
By multiplying both sides of the above inequality~(\ref{eq:m11}) by $\frac{16}{\eta\nu}$, we can obtain
\begin{align}
     \frac{1}{\nu^2}\E \|\nabla F(\theta_{t-1})-m_{t}\|^2 + \frac{1}{\eta^2} \E \|\theta_{t}-\theta_{t-1}\|^2
    & \leq \frac{4}{\nu^2}\E \|\nabla F(\theta_{t-1})-m_{t}\|^2 + \frac{1}{\eta^2} \E \|\theta_{t}-\theta_{t-1}\|^2 \nonumber \\
    & \leq \E[\frac{16(\Psi_{t-1} - \Psi_{t})}{\eta\nu}] + \frac{16}{\nu^2 T^{2\gamma-2}} + \frac{8c^2\eta\sigma^2}{L\nu}.
 \end{align}

Since $m_1 = \beta_{1}m_{0}+ (1-\beta_{1})g_1=(1-\beta_{1})\nabla f(\theta_{0};z_{1})$, we have
\begin{align}
 \|\nabla F(\theta_0) - m_1\|^2 & = \|\nabla F(\theta_0) - (1-\beta_{1})\nabla f(\theta_{0};z_{1})\|^2 \nonumber \\
 & =\|\nabla F(\theta_0) - \nabla f(\theta_{0};z_{1}) + \beta_1\nabla f(\theta_{0};z_{1})\|^2  \leq 2\sigma^2 + 2\beta_1^2G^2.
\end{align}
Given $\Psi_0 = F(\theta_0) + \frac{1}{2L}\|\nabla F(\theta_{0}) - m_1\|^2$, we have
\begin{align} \label{eq:m12}
    & \frac{1}{T} \sum_{t=1}^T [\frac{1}{\nu^2}\E \|\nabla F(\theta_{t-1})-m_{t}\|^2 + \frac{1}{\eta^2} \E \|\theta_{t}-\theta_{t-1}\|^2] \nonumber \\
    & \leq \frac{1}{T} \sum_{t=1}^T\E[\frac{16(\Psi_{t-1} - \Psi_{t})}{\eta\nu}] + \frac{16}{T^{2\gamma-2}\nu^2} + \frac{8c^2\eta\sigma^2}{L\nu} \nonumber \\
    &\leq \frac{16(F(\theta_0) + \frac{1}{2L}(2\sigma^2 + 2\beta_1^2G^2) - F^*)}{T\eta\nu} + \frac{16}{T^{2\gamma-2}\nu^2} + \frac{8c^2\eta\sigma^2}{L\nu}.
 \end{align}
Let $\Delta = F(\theta_0) + \frac{1}{L}(\sigma^2 + \beta_1^2G^2) - F^*$,
we can rewrite the above inequality~(\ref{eq:m12}) as follows:
\begin{align}
    & \frac{1}{T} \sum_{t=1}^T [\frac{1}{\nu^2}\E \|\nabla F(\theta_{t-1})-m_{t}\|^2 + \frac{1}{\eta^2} \E \|\theta_{t}-\theta_{t-1}\|^2] \nonumber \\
    &\leq \frac{16\Delta}{T\eta\nu} + \frac{16}{T^{2\gamma-2}\nu^2} + \frac{8c^2\eta\sigma^2}{L\nu}.
 \end{align}
According to the Jensen’s inequality, then we can obtain
\begin{align}  \label{eq:m13}
 & \frac{1}{T}\sum_{t=1}^T\mathbb{E}[\frac{1}{\nu} \|\nabla F(\theta_{t-1})-m_{t}\| + \frac{1}{\eta}  \|\theta_{t}-\theta_{t-1}\| ] \nonumber \\
 & \leq  \Big(\frac{2}{T}\sum_{t=1}^T\mathbb{E}[\frac{1}{\nu^2} \|\nabla F(\theta_{t-1})-m_{t}\|^2 + \frac{1}{\eta^2} \|\theta_{t}-\theta_{t-1}\|^2]\Big)^{1/2} \nonumber \\
 & \leq  \sqrt{\frac{32\Delta}{T\eta\nu} + \frac{32}{T^{2\gamma-2}\nu^2} + \frac{16c^2\eta\sigma^2}{L\nu}}
  \leq \frac{4\sqrt{2\Delta}}{\sqrt{T\eta\nu}} + \frac{4\sqrt{2}}{T^{\gamma-1}\nu} +
  \frac{4c\sigma\sqrt{\eta}}{\sqrt{L\nu}}.
\end{align}

By using $\theta_{t} = \theta_{t-1} - \eta (H_t m_t + \lambda\theta_{t-1} )$, we have
\begin{align} \label{eq:m14}
 \frac{1}{\nu}\|\nabla F(\theta_{t-1})-m_{t}\|  + \frac{1}{\eta}\|\theta_{t}-\theta_{t-1}\|  & =  \frac{1}{\nu}\|\nabla F(\theta_{t-1})-m_{t}\|  + \frac{1}{\eta}\|\eta (H_t m_t + \lambda\theta_{t-1} )\| \nonumber \\
 & \geq \frac{1}{\nu}\|\nabla F(\theta_{t-1})-m_{t}\|  + \|H_tm_t\| - \|\lambda\theta_{t-1}\|  \nonumber \\
 & =\frac{1}{\nu}\|H_t^{-1}H_t(\nabla F(\theta_{t-1})-m_{t})\|  + \|H_t m_t\| - \|\lambda\theta_{t-1}\| \nonumber \\
 & \geq \|H_t(\nabla F(\theta_{t-1})-m_{t})\|  + \|H_t m_t\| - \|\lambda\theta_{t-1}\| \nonumber \\
 & \geq \|H_t\nabla F(\theta_{t-1})\| - \|\lambda\theta_{t-1}\| \nonumber \\
 & \geq \frac{\|\nabla F(\theta_{t-1})\|}{r}- \|\lambda\theta_{t-1}\|,
\end{align}
where the above inequality holds by $\nu I_d \preceq H_t^{-1} \preceq r I_d$
and $\frac{1}{r}I_d \preceq H_t\preceq \frac{1}{\nu}I_d$.

By putting the above inequalities~(\ref{eq:m14}) into~(\ref{eq:m13}), we can obtain
\begin{align}
  \frac{1}{T}\sum_{t=1}^T\mathbb{E}\|\nabla F(\theta_{t-1})\|
 & \leq  \frac{r}{T}\sum_{t=1}^T\mathbb{E}[\frac{1}{\nu}\|\nabla F(\theta_{t-1})-m_{t}\|  + \frac{1}{\eta}\|\theta_{t}-\theta_{t-1}\|+\|\lambda\theta_{t-1}\| ] \nonumber \\
 & \leq r\big(\frac{4\sqrt{2\Delta}}{\sqrt{T\eta\nu}} + \frac{4\sqrt{2}}{T^{\gamma-1}\nu} +
  \frac{4c\sigma\sqrt{\eta}}{\sqrt{L\nu}}\big) + \frac{r}{T}\sum_{t=1}^T\E\|\lambda\theta_{t-1}\| \nonumber \\
 & \mathop{\leq}^{(i)} r\big(\frac{4\sqrt{2\Delta}}{\sqrt{T\eta\nu}} + \frac{4\sqrt{2}}{T^{\gamma-1}\nu} +
  \frac{4c\sigma\sqrt{\eta}}{\sqrt{L\nu}}\big) + \frac{r}{T}\sum_{t=1}^T\lambda t\sqrt{d}\eta \hat{G} \nonumber \\
 & \mathop{\leq}^{(ii)} r\big(\frac{4\sqrt{2\Delta}}{\sqrt{T\eta\nu}} + \frac{4\sqrt{2}}{T^{\gamma-1}\nu} +
  \frac{4c\sigma\sqrt{\eta}}{\sqrt{L\nu}}\big) + \frac{r}{T}\sum_{t=1}^T \frac{1}{\eta T^\gamma \sqrt{d}r \hat{G}}t\sqrt{d}\eta \hat{G}  \nonumber \\
 & \leq r\big(\frac{4\sqrt{2\Delta}}{\sqrt{T\eta\nu}} + \frac{4\sqrt{2}}{T^{\gamma-1}\nu} +
  \frac{4c\sigma\sqrt{\eta}}{\sqrt{L\nu}}\big)+ \frac{1}{T^{\gamma-1}},
\end{align}
where the above inequality~$(i)$ is due to $\|\theta_{t-1}\| \leq t\sqrt{d}\eta \hat{G}$, and
the above inequality~$(ii)$ holds by $0\leq \lambda \leq \frac{1}{\eta T^\gamma \sqrt{d}r \hat{G}}$.
Then we have
\begin{align}
  \frac{1}{T+1}\sum_{t=0}^T\mathbb{E}\|\nabla F(\theta_{t})\|
  & \leq r\big(\frac{4\sqrt{2\Delta}}{\sqrt{(T+1)\eta\nu}} + \frac{4\sqrt{2}}{(T+1)^{\gamma-1}\nu} +
  \frac{4c\sigma\sqrt{\eta}}{\sqrt{L\nu}}\big)+ \frac{1}{(T+1)^{\gamma-1}} \nonumber \\
  & \leq r\big(\frac{4\sqrt{2\Delta}}{\sqrt{T\eta\nu}} + \frac{4\sqrt{2}}{T^{\gamma-1}\nu} +
  \frac{4c\sigma\sqrt{\eta}}{\sqrt{L\nu}}\big)+ \frac{1}{T^{\gamma-1}}.
\end{align}

\end{proof}

\begin{remark}
Form the above Theorem, let $\eta=\frac{1}{\sqrt{T}}$ and $\gamma=\frac{3}{4}$, we can obtain
\begin{align}
  \frac{1}{T+1}\sum_{t=0}^T\mathbb{E}\|\nabla F(\theta_{t})\|
   \leq  r\big(\frac{4\sqrt{2\Delta}}{\sqrt{\nu}T^{1/4}} + \frac{4\sqrt{2}}{\nu T^{1/4}} +
  \frac{4c\sigma}{\sqrt{L\nu}T^{1/4}}\big)+ \frac{1}{T^{1/4}}.
\end{align}

Further let $G=O(1)$, $L=O(1)$, $\sigma=O(1)$, $\tau=O(1)$, $\beta_1=O(1)$ with $\beta_1\in(0,1)$, $\beta_2=O(1)$ with $\beta_2\in(0,1)$ and $c=O(1)$, we have $r=\max(1,G^2 +\varepsilon)=O(1)$, $\frac{1}{\nu}= \max(\frac{1}{1-\beta_1},\frac{1}{(\tau + \varepsilon)(1-\beta_1)})=O(1)$ and $\Delta = F(\theta_0) + \frac{1}{L}(\sigma^2 + \beta_1^2G^2) - F^*=O(1)$. Then we have
\begin{align}
  \frac{1}{T+1}\sum_{t=0}^T\mathbb{E}\|\nabla F(\theta_{t})\|  \leq O(\frac{1}{T^{1/4}}).
\end{align}
\end{remark}

%%%%%%%%%%%%%%%%%%%%%%%%%%%%%%%%%%%%%%%%%%%%%%%%%%%%%%%%%%%%%%%%%%%%%%%%%%%%%%%
%%%%%%%%%%%%%%%%%%%%%%%%%%%%%%%%%%%%%%%%%%%%%%%%%%%%%%%%%%%%%%%%%%%%%%%%%%%%%%%

\end{document}